\documentclass{article}

    \PassOptionsToPackage{numbers, compress}{natbib}

\usepackage[final]{neurips_2025}

\usepackage[utf8]{inputenc} %
\usepackage[T1]{fontenc}    %
\usepackage{hyperref}       %

\usepackage{url}            %
\usepackage{booktabs}       %
\usepackage{amsmath,amssymb,amsthm,mathtools,amsfonts,mathrsfs}

\usepackage{nicefrac}       %
\usepackage{microtype}      %
\usepackage{xcolor}         %
\usepackage{dm-colors}
\hypersetup{
  colorlinks,
  citecolor=dmblue500,
  linkcolor=dmblue500,
  urlcolor=dmblue500
}

\usepackage{color, colortbl}
\definecolor{LightGray}{gray}{0.9}

\usepackage{cancel}
\usepackage{multirow}
\usepackage{times}
\usepackage{threeparttable}
\usepackage{xspace}

\usepackage{tikz}

\usepackage{graphicx}           %
\usepackage{subcaption}         %
\usepackage[space]{grffile}     %
\usepackage{url}                %
\usepackage{bbm}                %
\usepackage{wasysym}

\usepackage[vlined,ruled]{algorithm2e}
\usepackage{rl-theory}

\usepackage{nomencl}
\makenomenclature

\usepackage{colorbox-thms}

\usepackage{etoolbox}
\renewcommand\nomgroup[1]{%
  \item[\bfseries
  \ifstrequal{#1}{A}{Mathematical Notations}{%
  \ifstrequal{#1}{B}{Constrained Bandit}{%
  \ifstrequal{#1}{C}{Constrained MDP}{}}}%
]}

\newcommand{\rmvEventBandit}{\hyperref[lemma:good-event1]{\sE_1}\xspace}
\newcommand{\confeventBandit}{\hyperref[lemma:good-event2]{\sE_2}\xspace}

\newcommand{\rmvEevent}{\hyperref[lemma:good-event1-MDP]{\sE_1}\xspace}
\newcommand{\confevent}{\hyperref[lemma:good-event2-MDP]{\sE_2}\xspace}

\newcommand{\optpes}{\hyperref[eq:op-ps-bandit-propose]{Opt-Pes}\xspace}
\newcommand{\Banditalgo}{\hyperref[algo:zero-vio-bandit]{OPLB-SP}\xspace}
\newcommand{\MDPalgo}{\hyperref[algo:zero-vio-linear MDP]{OPSE-LCMDP}\xspace}

\newcommand{\cQr}{\hyperref[def:abs-Q-function]\cQ^r}
\newcommand{\cQu}{\hyperref[def:abs-Q-function]\cQ^u}
\newcommand{\cQc}{\hyperref[def:abs-Q-function]\cQ^\dagger}
\newcommand{\Qcomp}{\hyperref[def:composite-Q-function]\cQ^\circ}

\newcommand{\tiPi}{\hyperref[def:policy-class]{\widetilde{\Pi}}}

\newcommand{\cVr}{\hyperref[def:V-cover]\cV^r}
\newcommand{\cVu}{\hyperref[def:V-cover]\cV^u}
\newcommand{\cVc}{\hyperref[def:V-cover]\cV^\dagger}

\newcommand{\Deltar}{\hyperref[def:delta-Q]{\Delta^{(k)}_r}\xspace}
\newcommand{\Deltau}{\hyperref[def:delta-Q]{\Delta^{(k)}_u}\xspace}
\newcommand{\Deltac}{\hyperref[def:delta-Q]{\Delta^{(k)}_\dagger}\xspace}

\newcommand{\unconfBandit}{\hyperref[def:unconf-set]{\cU}}
\newcommand{\unconfMDP}{\hyperref[def:unconf-set MDP]{\cU}}

\title{Provably Efficient RL under Episode-Wise Safety in\\ Constrained MDPs with Linear Function Approximation}

\author{%
  Toshinori Kitamura \\
  The University of Tokyo\\
  Tokyo, Japan \\
  \texttt{toshinori-k@weblab.t.u-tokyo.ac.jp} \\
  \And
  Arnob Ghosh \\
  New Jersey Institute of Technology \\
  New Jersey, USA \\
  \And
  Tadashi Kozuno \\
  OMRON SINIC X, NexaScience \\
  Tokyo, Japan \\
  \And
  Wataru Kumagai \\
  OMRON SINIC X, NexaScience\\
  Tokyo, Japan \\
  \And
  Kazumi Kasaura \\
  OMRON SINIC X\\
  Tokyo, Japan \\
  \AND
  Kenta Hoshino\thanks{Current affiliation: Institute of Science Tokyo, Tokyo, Japan, and DENSO IT Laboratory, Tokyo, Japan.}\\
  Kyoto University \\
  Kyoto, Japan \\
  \And
  Yohei Hosoe\\
  Kyoto University \\
  Kyoto, Japan \\
  \And
  Yutaka Matsuo\\
  The University of Tokyo\\
  Tokyo, Japan \\
}

\begin{document}

\maketitle

\begin{abstract}
We study the reinforcement learning (RL) problem in a constrained Markov decision process (CMDP), where an agent explores the environment to maximize the expected cumulative reward while satisfying a single constraint on the expected total utility value in every episode. While this problem is well understood in the tabular setting, theoretical results for function approximation remain scarce. This paper closes the gap by proposing an RL algorithm for linear CMDPs that achieves $\widetilde{\mathcal{O}}(\sqrt{K})$ regret with an {\em episode-wise} zero-violation guarantee. Furthermore, our method is computationally efficient, scaling polynomially with problem-dependent parameters while remaining independent of the state space size. Our results significantly improve upon recent linear CMDP algorithms, which either violate the constraint or incur exponential computational costs.
\end{abstract}

\section{Introduction}

\looseness=-1
Safe decision-making is essential in real-world applications such as plant control and finance \citep{gu2022review}. 
Constrained Markov decision process (CMDP) is a mathematical framework for developing decision-making algorithms with formal safety guarantees \citep{altman1999constrained}. 
This paper studies the reinforcement learning (RL) problem in finite-horizon CMDPs, where an agent explores the environment to maximize the expected cumulative rewards while satisfying a single constraint on the expected total utility value.

\looseness=-1
Safe exploration has been established in the tabular CMDP settings.
Several works \citep{liu2021learning,bura2022dope,yu2024improved} achieve \textbf{episode-wise} zero-violation RL with $\tiO(\sqrt{K})$ regret for $K$ number of episodes, ensuring constraint satisfaction in every episode.
Their approach consists of two phases: 
deploying a known strictly safe policy $\pisafe$ and then updating policies via linear programs (LPs), which optimizes an optimistic objective while satisfying a pessimistic constraint.
Deploying $\pisafe$ is necessary to ensure feasible solutions for the LPs once enough environmental information is collected.

\looseness=-1
While safe exploration is well-established in tabular CMDPs, extending it to large-scale CMDPs remains a major challenge.
LP-based methods are impractical at scale due to their state-dependent computational cost.\footnote{While some works (e.g., \citealp{neu2023efficient}) proposed LP methods for unconstrained linear MDPs, they remain unsuitable for our exploration setting or still incur state-dependent computational costs (see \Cref{sec:related work}).} 
As a result, even in linear CMDPs, where value functions have linear structure, episode-wise safe RL has not been achieved. 
The state-of-the-art linear CMDP algorithm \citep{ghosh2024towards}, which achieves the $\tiO(\sqrt{K})$ violation regret,\footnote{Violation regret denotes the total amount of constraint violation during exploration.} incurs an exponential computational cost of $K^H$, where $H$ is the horizon.
Several studies achieve safe RL under instantaneous constraints \cite{amani2021safe,roknilamouki2025provably},\footnote{Inst. constraint requires $u_h(s^{(k)}_h, a^{(k)}_h) \geq b\; \forall h, k \in \bbrack{1,H}\times\bbrack{1,K}$ (see \cref{sec:MDP} for notations).} a special subclass of the episode-wise safety that can be overly conservative (e.g., in drone control, temporary high energy consumption is tolerable, but full battery depletion is not).
\Cref{table:algorithms} summarizes representative algorithms, with additional literature in \Cref{sec:related work}.
In short, a fundamental open question remains:
\begin{center}
\emph{Can we develop a computationally efficient\footnote{\looseness=-1 An algorithm is comp. efficient if its cost is polynomial with problem parameters, excluding the state space.} linear CMDP algorithm with \\sublinear regret and zero episode-wise constraint violation?}
\end{center}

\looseness=-1
\paragraph{Contributions.}
We propose \textbf{O}ptimistic-\textbf{P}essimistic \textbf{S}oftmax \textbf{E}xploration for \textbf{L}inear \textbf{CMDP} (\MDPalgo), the first algorithm for linear CMDPs that achieves \textbf{\(\boldsymbol{\tiO\paren{\sqrt{K}}}\)-regret and episode-wise safety}.
Our approach builds on the optimistic-pessimistic exploration framework with two key innovations for large-scale state-space problems:
($\mathrm{i}$) a new \textbf{deployment rule for $\boldsymbol{\pisafe}$},
and ($\mathrm{ii}$) a \textbf{computationally efficient} method to implement optimism for the objective and pessimism for the constraint within the softmax policy framework \citep{ghosh2022provably,ghosh2024towards}.

\looseness=-1
\Cref{sec:zero-vio bandit} first analyzes the linear constrained bandit problem as a ``warm-up'' for linear CMDPs ($H=1$ with an expected instantaneous constraint), highlighting the key role of the $\pisafe$ deployment rule in avoiding linear regret.
Previous instantaneous constraint literature limits $\pisafe$ deployments by representing the safe action as a vector $\basafe \in \R^d$ \citep{pacchiano2021stochastic,pacchiano2024contextual,hutchinson2024directional,amani2019linear,amani2021safe}. 
However, extending this approach to episode-wise safety is non-trivial, as the constraint is imposed on policies rather than actions, and policies may be nonlinear functions (e.g., softmax mapping from value functions) rather than single vectors.
We overcome this challenge by showing that \textbf{if $\boldsymbol{\pisafe}$ is deployed only when the agent is less confident in $\boldsymbol{\pisafe}$'s safety}, the number of deployments is logarithmically bounded (\Cref{lemma:Ck-bound-main}).

\looseness=-1  
\Cref{sec:MDP} then extends the bandit result to RL in CMDPs.
To enable optimistic-pessimistic exploration in linear CMDPs, \MDPalgo employs the \textbf{composite softmax policy} (\Cref{def:composite-softmax-policy}), which adjusts optimism and pessimism by controlling a variable \(\lambda\).  
\MDPalgo efficiently searches for the best \(\lambda\) through \textbf{bisection search}, achieving a \textbf{polynomial computational cost} in problem parameters, independent of state-space cardinality (\Cref{remark:computational cost}). 
Overall, our techniques---the novel $\pisafe$ deployment rule and softmax-based optimistic-pessimistic exploration---achieve the first episode-wise safe RL with sublinear regret and computational efficiency in linear CMDPs.

\begin{table*}[tb]
\small
\caption{
\looseness=-1
\small
Representative CMDP results with $K$-dependent regrets. See \Cref{sec:MDP} for CMDP notations.
}\label{table:algorithms}
\centering
\begin{tabular}{ccccc}
\hline
&Paper & Epi.-Wise Safe? & Comp. Efficient? & Regret \\\hline 
\multirow{1}{*}{Tabular} 
& \citet{yu2024improved} & Yes & $|\S|$ dependent & $\tiO\paren{\bslt^{-1} \sqrt{|\S|^2 |\A| H^5 K}}$  \\ \hline
\multirow{3}{*}{Linear} 
& \citet{ghosh2024towards} & No & No & $\tiO\paren{\sqrt{d^3 H^4 K}}$  \\ 
& \citet{roknilamouki2025provably} & Instantaneous & Yes & $\tiO\paren{\xi^{-1}\sqrt{d^3 H^4 K}}$ \\ 
\rowcolor{dmblue50!30!white} &\textbf{\MDPalgo (Ours)} & \textcolor{red}{Yes} & \textcolor{red}{Yes} & $\tiO\paren{\bslt^{-1} \sqrt{d^5 H^8 K}}$\\\hline
\multicolumn{2}{c}{Bandit lower bound \cite{pacchiano2021stochastic}} & Yes & N/A & $\widetilde{\Omega}(\max\brace{\sqrt{|\A| K}, \xi^{-2}})$\textsuperscript{*} \\\hline
\end{tabular}
\begin{tablenotes}
{\scriptsize
\item * The regrets of \citep{yu2024improved} and ours include an additional $\tiO(\xi^{-2})$ constant. We omit them from the table due to space limitations.
\citep{roknilamouki2025provably,pacchiano2021stochastic,amani2021safe} avoid the constant by assuming access to a safe action vector, but they are limited to instantaneous constraints (see \cref{subsec:bandit-challenge}).
Including ours, existing safe algorithms \citep{liu2021learning,bura2022dope,yu2024improved,roknilamouki2025provably,amani2019linear,pacchiano2021stochastic,hutchinson2024directional} incur $\tiO(\xi^{-1}\sqrt{K})$ regret. Whether this can be improved to $\tiO(\sqrt{K})$ remains open.
}
\end{tablenotes}
\vspace{-0.2in}
\end{table*}

\paragraph{Mathematical notations.}
\looseness=-1
The set of probability distributions over a set $\S$ is denoted by $\sP(\S)$.
For integers $a \leq b$, let $\bbrack{a,b}\df \brace*{a,\dots, b}$, and $\bbrack{a, b} \df \emptyset$ if $a > b$.
For $\bx \in \R^N$, its $n$-th element is $\bx_n$ or $\bx(n)$.
The clipping function $\clip\brace{\bx, a, b}$ returns $\bx'$ with $\bx'_i = \min\brace{\max\brace{\bx_i, a}, b}$ for each $i$.
We define $\bzero \df \paren{0, \ldots, 0}^\top$ and $\bone \df \paren{1, \ldots, 1}^\top$, with dimensions inferred from the context.
For a positive definite matrix $\bA \in \R^{d\times d}$ and $\bx \in \R^d$, we denote $\norm*{\bx}_\bA = \sqrt{\bx^\top \bA \bx}$.
For positive sequences $\brace{a_n}$ and $\brace{b_n}$ with $n=1,2, \ldots$, we write $a_n=O\left(b_n\right)$ if there exist $C>0$ and $N \in \N$ such that $a_n \leq C b_n$ for all $n \geq N$, and $a_n = \Omega(b_n)$ for the reverse inequality.
We use $\widetilde{O}(\cdot)$ and $\widetilde{\Omega}(\cdot)$ to further hide the polylogarithmic factors.
Finally, for $\bx \in \R^d$, we denote its softmax distribution as $\softmax(\bx) \in \sP(\bbrack{1, d})$ with its $i$-th component $\softmax(\bx)_i = \exp(\bx_i) / \paren*{\sum_{i} \exp(\bx_i)}$.

\section{Warm Up: Safe Exploration in Linear Constrained Bandit}\label{sec:zero-vio bandit}

\looseness=-1
To better illustrate the core ideas of our CMDP algorithm, this section introduces a contextual linear bandit variant based on \citet{pacchiano2021stochastic}.
All the proofs in this section are provided in \Cref{sec:regret-analysis-bandit}.
Let $\A \subset \R^d$ be the action space, a compact set of bounded $d$-dimensional vectors. 
Without loss of generality, we assume $\norm{\ba}_2 \leq 1$ for any $\ba \in \A$.
At each round $k$, the agent selects a policy $\pi^{(k)} \in \sP\paren*{\A}$, samples an action $\ba^{(k)} \sim \pi^{(k)}$, and observes the reward $r^{(k)} = \btheta_r^\top \ba^{(k)}+ \varepsilon^{(k)}_r$ and utility $u^{(k)} = \btheta_u^\top \ba^{(k)}+ \varepsilon^{(k)}_u$.
Here, $\btheta_r, \btheta_u\in \R^d$ are vectors unknown to the agent such that \(\norm{\btheta_r}_2, \norm{\btheta_u}_2 \leq B\), and $\varepsilon^{(k)}_r, \varepsilon^{(k)}_u$ are $R$-sub-Gaussian random noises.
For any policy $\pi$ and $g \in \brace{r, u}$, let $g_{\pi} \df \E_{\ba \sim \pi}\brack*{\left\langle\btheta_g, \ba\right\rangle}$.
We consider a constraint such that the expected utility must be above the threshold $b \in \R$. 
Formally, let $\Pisafe \df \brace*{\pi \given u_{\pi} \geq b}$ denote the set of safe policies. 
The agent's goal is to achieve sublinear regret while satisfying the expected \textbf{instantaneous} constraints defined as follows:
\begin{align}\label{eq:const-bandit}
{\textstyle
\regret (K) \df 
\sum^K_{k=1} 
r_{\pi^\star} - r_{\pi^{(k)}}
= o\paren*{K}
\;\text{ such that }\;
{\pi^{(k)}} \in \Pisafe \; \forall k \in \bbrack{1, K}\;,
}
\end{align}
where $\pi^\star \in \argmax_{\pi \in \Pisafe} r_{\pi}$.
A sublinear regret exploration is efficient, as its averaged reward approaches the optimal value, i.e., \(\lim_{K\to \infty }\frac{1}{K}r_{\pi^{(K)}} \to  r_{\pi^\star}\).
Finally, we assume access to a strictly safe policy in $\Pisafe$, as deploying arbitrary policies without this assumption risks violating constraints\footnote{The knowledge of $\bslt$ is for simplicity. If unknown, we can estimate it by deploying $\pisafe$ with a little overhead.}.
\begin{assumption}[Safe policy]\label{assumption:slater-bandit}
We have access to $\pisafe \in \Pisafe$ and $\bslt > 0$ such that $u_{\pisafe} - b \geq \bslt\;$.
\end{assumption}

\subsection{Technical Challenge: Zero-Violation with a Safe Policy}\label{subsec:bandit-challenge}

\looseness=-1
The key to efficient and safe exploration is the \textbf{optimistic-pessimistic} exploration, which constructs an optimistic reward $\widebar{r}^{(k)}_\pi \geq r_\pi$ and a pessimistic utility $\underline{u}^{(k)}_\pi \leq u_\pi$, and then computes a policy by:
\begin{align}\label{eq:optimistic-pessimistic-bandit}
\textstyle
\max_{\pi \in \sP(\A)} 
\widebar{r}^{(k)}_\pi
\quad \text{ such that }\quad \underline{u}^{(k)}_\pi \geq b\;.
\end{align}
Here, $\widebar{r}^{(k)}_\pi$ and $\underline{u}^{(k)}_\pi$ are designed to quickly approach $r_\pi$ and $u_\pi$ as data accumulates for efficient exploration \citep{abbasi2011improved}.
However, although \Cref{eq:optimistic-pessimistic-bandit} can have feasible solutions when $\underline{u}_\pi \approx u_\pi$, the pessimistic constraint may not have any feasible solution in the early stages of exploration.

\looseness=-1
To ensure that \eqref{eq:optimistic-pessimistic-bandit} always has a solution, a common bandit approach assumes access to a safe action $\basafe \in \A$ such that $\btheta_u^\top \basafe \geq b + \bslt$, and then ensures the feasibility of \eqref{eq:optimistic-pessimistic-bandit} by leveraging the \textbf{vector representation} of $\basafe \in \R^d$. 
For example, \cite{pacchiano2021stochastic,pacchiano2024contextual,amani2019linear} designed $\underline{u}^{(k)}_\pi$ using the orthogonal direction $\paren*{\basafe}^\bot \df \basafe - \basafe / \norm*{\basafe}_2$, while \cite{hutchinson2024directional} assume $\basafe = \bzero \in \A$ with a negative constraint threshold $b < 0$. 
Both approaches ensure that a policy playing $\basafe$ with probability $1$ is always feasible in \eqref{eq:optimistic-pessimistic-bandit}.

\looseness=-1
However, extending this safe action technique to episode-wise safe RL is non-trivial, as the episode-wise constraint is imposed on policies rather than actions, and policies in linear CMDPs may be nonlinear functions (e.g., softmax mappings from value functions) rather than single vectors.
To address this challenge, we first develop a safe bandit algorithm without relying on safe action techniques.

\subsection{Algorithm and Analysis}\label{sec:bandit-algo-analysis}

\looseness=-1
We summarize the proposed \textbf{O}ptimistic-\textbf{P}essimistic \textbf{L}inear \textbf{B}andit with \textbf{S}afe \textbf{P}olicy (\Banditalgo) in \Cref{algo:zero-vio-bandit}, which follows the standard linear bandit framework (see \citet{abbasi2011improved}).
Throughout this section, we analyze \Cref{algo:zero-vio-bandit} under the parameters listed in its \hyperref[algo:zero-vio-bandit]{\textbf{Input}} line.
Let $\hbtheta^{(k)}_r \df \paren{\bLambda^{(k)}}^{-1}\sum_{i=1}^{k-1} \ba^{(i)} r^{(i)}$ and $\hbtheta^{(k)}_u \df \paren{\bLambda^{(k)}}^{-1}\sum_{i=1}^{k-1} \ba^{(i)} u^{(i)} $ denote the regularized least-squares estimates of $\btheta_r$ and $\btheta_u$, respectively, where \(\bLambda^{(k)} \df \rho \bI + \sum_{i=1}^{k-1} \ba^{(i)} \paren{\ba^{(i)}}^\top \).
Let \(\widehat{r}^{(k)}_\pi\df \E_{\ba \sim \pi}\brack{\ba^\top \hbtheta_r^{(k)}}\) and \(\widehat{u}^{(k)}_\pi \df \E_{\ba \sim \pi}\brack{\ba^\top \hbtheta_u^{(k)}}\) be the estimated reward and utility functions.
Using the bonus function 
\(\beta^{(k)}_\pi \df \E_{\ba \sim \pi}\norm*{\ba}_{\paren*{\bLambda^{(k)}}^{-1}}\), with the well-established elliptical confidence bound argument for linear bandits \citep{abbasi2011improved}, the following confidence bounds hold:
\begin{lemma}[Confidence bounds]\label{lemma:bandit-opt-pes-main}
For any $\pi$ and $k$, with probability (w.p.) at least $1-\delta$,
$$
r_{\pi} + 2\cp \beta^{(k)}_{\pi} \geq 
\widehat{r}_{\pi}^{(k)} + \cp \beta^{(k)}_{\pi} \geq r_{\pi} \quad \text{ and }\quad
u_{\pi} \geq \widehat{u}_{\pi}^{(k)} - \cp \beta^{(k)}_{\pi}
\geq u_{\pi} - 2\cp \beta^{(k)}_{\pi}\;.
$$
\end{lemma}
Based on \cref{lemma:bandit-opt-pes-main}, \Cref{algo:zero-vio-bandit} solves the following optimistic-pessimistic (\optpes) problem.
The optimistic objective promotes efficient exploration, while the pessimism enforces the constraint satisfaction:
\begin{align}\label{eq:op-ps-bandit-propose}
\textstyle
\textbf{Opt-Pes (\Cref{line:op-ps-policy-bandit})}
\quad
\pi^{(k)} \in 
\argmax_{\pi \in \sP(\A)}\; \widehat{r}^{(k)}_\pi + \co \beta^{(k)}_\pi \; \text { such that }\; \widehat{u}^{(k)}_\pi - \cp \beta^{(k)}_\pi  \geq b \;.
\end{align}
This is a convex optimization problem when the set $\A$ satisfies certain structural assumptions, such as being discrete or ellipsoidal (see, e.g., Section 19.3 of \citet{lattimore2020bandit}).
To emphasize our approach to the technical challenge in \cref{subsec:bandit-challenge}, this section omits the computational details of \eqref{eq:op-ps-bandit-propose} and focuses instead on the core technique for efficient exploration under episode-wise safety.

{\small
\begin{algorithm}[t!]
\caption{Optimistic-Pessimistic Linear Bandit with Safe Policy}
\label{algo:zero-vio-bandit}
\DontPrintSemicolon
\LinesNumbered
{\small
    \KwIn{Regression coefficient $\rho = 1$, bonus scalers $\cp = B + R\sqrt{d \ln 4K\delta^{-1}}$ and $\co = \cp \paren{1 + 2B\bslt^{-1}}$, safe policy $\pisafe$, and iteration length $K \in \N$}
    \For{$k = 1, \dots, K$}{
    Let \(\beta^{(k)}_\pi\), \(\widehat{r}^{(k)}_\pi\), and \(\widehat{u}^{(k)}_\pi\) be bonus, estimated reward and utility, respectively (see \Cref{sec:bandit-algo-analysis})\;
    \textbf{if} \(\cp \beta^{(k)}_{\pisafe} > \frac{\bslt}{2}\) \textbf{then} {\(\pi^{(k)}\df \pisafe\){\color{blue}\tcc*[r]{Deploy $\pisafe$ if $\pisafe$ is unconfident}}\label{line:safe-policy-bandit}}
    \lElse{
    \(\pi^{(k)} \in 
    \argmax_{\pi \in \sP(\A)}\; \widehat{r}^{(k)}_\pi + \co \beta^{(k)}_\pi \; \text { such that }\; \widehat{u}^{(k)}_\pi - \cp \beta^{(k)}_\pi \geq b \)
    }\label{line:op-ps-policy-bandit}
    Sample an action $\ba^{(k)} \sim \pi^{(k)}$ and observe reward $r^{(k)}$ and utility $u^{(k)}$.
    }
}
\end{algorithm}
}

\subsubsection{Zero-Violation and Logarithmic Number of $\pisafe$ Deployments}\label{subsec:zero-vio-bandit}

\looseness=-1
Since $\pi^{(k)}$ is either $\pisafe$ or the solution to \optpes (if feasible), all deployed policies in \Cref{algo:zero-vio-bandit} satisfy the constraint with high probability due to the pessimistic constraint.
However, as noted in \Cref{subsec:bandit-challenge}, the pessimistic constraint may render \optpes infeasible, requiring \Cref{line:op-ps-policy-bandit} to wait until the bonus \(\beta^{(k)}_\pi\) shrinks sufficiently.  
Yet, waiting too long overuses the suboptimal \(\pisafe\), leading to poor regret. Thus, exploration must keep the number of iterations where \Cref{eq:optimistic-pessimistic-bandit} is infeasible bounded.

\looseness=-1
The core technique of \Cref{algo:zero-vio-bandit} lies in the \textbf{$\boldsymbol{\pisafe}$ deployment trigger} based on the confidence of $\pisafe$.
Specifically, we solve the optimistic-pessimistic optimization whenever 
$\beta^{(k)}_{\pisafe} \leq \frac{\bslt}{2\cp}$; otherwise, we correct the data by deploying $\pisafe$ (see \Cref{line:safe-policy-bandit}).
Under this trigger, the following \Cref{lemma:Ck-bound-main} ensures that the number of $\pisafe$ deployments grows \textbf{logarithmically} with the iteration length $K$.
\begin{definition}[$\pisafe$ unconfident iterations]\label{def:unconf-set}
Let $\unconfBandit$ be the set of iterations when \Cref{algo:zero-vio-bandit} is unconfident in $\pisafe$, i.e., 
\(
\unconfBandit \df 
\brace{k \in \bbrack{1, K} \given 
\beta^{(k)}_{\pisafe} > \bslt / (2\cp)
}
\).
Let $\unconfBandit^\complement \df \bbrack{1, K} \setminus \unconfBandit$ be its complement.
\end{definition}
\begin{theorem}[Logarithmic $\abs{\unconfBandit}$ bound]\label{lemma:Ck-bound-main}
It holds w.p. at least $1-\delta$ that
\(|\unconfBandit| \leq \cO\paren*{d\cp^2\bslt^{-2}\ln\paren*{K\delta^{-1}}}\).
\end{theorem}
The proof utilizes the well-known elliptical potential lemma \citep{abbasi2011improved}.
Intuitively, it ensures that the confidence bounds shrink on average, thereby limiting the number of iterations where the algorithm remains unconfident in $\pisafe$.
\citet{he2021uniform,zhang2023interplay} employed a similar technique in linear bandits to ensure the suboptimality of policies after sufficient iterations.

\looseness=-1
Moreover, combined with \Cref{lemma:bandit-opt-pes-main}, the following \Cref{lemma:alpha-feasibility-main} ensures that, after logarithmic iterations, \textbf{policies around $\boldsymbol{\pisafe}$ will become feasible solutions to \optpes and \Cref{line:op-ps-policy-bandit}}.
\begin{lemma}[Mixture policy feasibility]\label{lemma:alpha-feasibility-main}
Consider $k \in \unconfBandit^\complement$.
Let $\alpha^{(k)} \df \frac{\bslt - 2\cp \beta^{(k)}_{\pisafe}}{\bslt - 2\cp \beta^{(k)}_{\pisafe} + 2\cp \beta^{(k)}_{\pi^\star}}$.
For any $\alpha \in \brack*{0, \alpha^{(k)}}$, the mixture policy \(\pi_\alpha \df (1-\alpha) \pisafe + \alpha \pi^\star\) satisfies
$u_{\pi_\alpha} - 2\cp \beta^{(k)}_{\pi_\alpha} \geq b$.
\end{lemma}
Note that the mixture policy $\pi_\alpha$ is introduced only to ensure the feasibility of \optpes; it does not need to be computed in the algorithm.

\looseness=-1
Finally, \Cref{lemma:bandit-opt-pes-main} and \Cref{lemma:alpha-feasibility-main} directly imply the following zero-violation guarantee:
\begin{corollary}[Zero-violation]
W.p. at least $1-\delta$, \Cref{algo:zero-vio-bandit} satisfies $\pi^{(k)} \in \Pisafe$ for any $k$.
\end{corollary}

\subsubsection{Regret Analysis}
\looseness=-1
The remaining task is to ensure sublinear regret. By \Cref{lemma:Ck-bound-main,,lemma:bandit-opt-pes-main}, the regret is decomposed as:
\begin{align*}
{\textstyle
\regret(K) 
\leq \tiO\paren*{d B \cp^2\bslt^{-2}}
+ \underline{3\co \sum_{k \in \unconfBandit^\complement}\beta^{(k)}_\pi}_{\;\circled{1}}
+ \underline{\sum_{k \in \unconfBandit^\complement} \paren*{r_{\pi^\star} - \widehat{r}^{(k)}_{\pi^{(k)}} - \co \beta^{(k)}_\pi}}_{\;\circled{2}}\;.
}
\end{align*}
Using the elliptical potential lemma \citep{abbasi2011improved}, we can bound $\circled{1} \leq \tiO\paren{\co \sqrt{dK}}$.

\looseness=-1
For the term $\circled{2}$, when there is no constraint in \optpes, the common strategy is bounding $\circled{2}$ using $r_{\pi^\star} - \widehat{r}^{(k)}_{\pi^{(k)}} - \co \beta^{(k)}_\pi \leq 0$, leveraging the optimism due to \Cref{lemma:bandit-opt-pes-main} with the maximality of $\pi^{(k)}$ in \optpes (see, e.g., \citet{abbasi2011improved}).
However, due to the pessimistic constraint in \optpes, $\pi^\star$ may not be a solution to \optpes, necessitating a modification to this approach.

\looseness=-1
Recall from \Cref{lemma:alpha-feasibility-main} that, for $k \in \unconfBandit^\complement$, the mixture policy $\pi_{\alpha^{(k)}} \df (1-\alpha^{(k)})\pisafe + \alpha^{(k)}\pi^\star$ satisfies $u_{\pi_{\alpha^{(k)}}} - 2\cp \beta^{(k)}_{\pi_{\alpha^{(k)}}} \geq b$.
For this $\pi_{\alpha^{(k)}}$, the following optimism with respect to $\pi^\star$ holds:
\begin{lemma}[$\pi_{\alpha^{(k)}}$ optimism]\label{lemma:optimism-main}
For any $k \in \unconfBandit^\complement$, it holds 
\(r_{\pi_{\alpha^{(k)}}} + (2B\cp\bslt^{-1}) \beta^{(k)}_{\pi_{\alpha^{(k)}}} \geq r_{\pi^{\star}}\).
\end{lemma}
Using \Cref{lemma:optimism-main,lemma:bandit-opt-pes-main} with $\co= \cp\paren{1+2B \bslt^{-1}}$, we have
\begin{align}\label{eq:mixture-policy-optimism}
{\textstyle
\circled{2}
\leq
\sum_{k\in \unconfBandit^\complement}
\widehat{r}_{\pi_{\alpha^{(k)}}} + \co \beta^{(k)}_{\pi_{\alpha^{(k)}}}
- \widehat{r}^{(k)}_{\pi^{(k)}} - \co\beta^{(k)}_{\pi^{(k)}}
\leq 0\;,
}
\end{align}
where the second inequality holds since $\pi_{\alpha^{(k)}}$ is a feasible solution to \optpes and $\pi^{(k)}$ is its maximizer.
This optimism via a mixture policy technique is adapted from tabular CMDPs \citep{liu2021learning, bura2022dope} to the linear bandit setup.
By combining all the results, \Cref{algo:zero-vio-bandit} archives the following guarantees:
\begin{theorem}\label{theorem:bandit-regret-main}
\looseness=-1
If \Banditalgo is run with the parameters listed in its \hyperref[algo:zero-vio-bandit]{\textbf{Input}} line, w.p. at least $1-\delta$,
\begin{align*}
\pi^{(k)} \in \Pisafe \;\text{ for any }\; k \in \bbrack{1, K} 
\quad \text{ and }\quad 
\regret (K) \leq 
\tiO\paren{dB \cp^2\bslt^{-2}
+ \co \sqrt{dK}
}\;.
\end{align*}
When $B = R = 1$, the regret bound simplifies to $\tiO\paren{d^2\bslt^{-2} + \xi^{-1}\sqrt{d^3K}}$.
\end{theorem}

\looseness=-1
In summary, \Banditalgo relies on three components:
($\mathrm{i}$) optimistic-pessimistic updates (\optpes),
($\mathrm{ii}$) a logarithmic number of $\pisafe$ deployments (\Cref{lemma:Ck-bound-main}), and
($\mathrm{iii}$) compensation for the pessimism (\Cref{lemma:optimism-main}).
Building on these components, the next section develops a linear CMDP algorithm.

\section{Safe Reinforcement Learning in Linear Constrained MDP} \label{sec:MDP}

\looseness=-1
A finite-horizon CMDP is defined as a tuple $(\S, \A, H, P, r, u, b, s_1)$, where
$\S$ is the finite but potentially exponentially large state space,
$\A$ is the finite action space ($\abs{\A}=\aA$),\footnote{While \cref{sec:zero-vio bandit} permits infinite actions, we here restrict to the finite case. Even then, episode-wise safe exploration in linear CMDP is non-trivial, and infinite actions would further complicate the regret analysis.}
$H \in \N$ is the episode horizon,
$b \in [0, H]$ is the constrained threshold, 
and $s_1$ is the fixed initial state.
The reward and utility functions $r, u: \HSA \to [0, 1]$ specify the reward $r_h(s, a)$ and constraint utility $u_h(s, a)$ when taking action $a$ at state $s$ in step $h$.
Finally, $P_\cdot\paren*{\cdot\given \cdot, \cdot}: \HSA\times \S \to [0, 1]$ denotes the transition kernel, where $P_{h}\paren{s' \given s, a}$ denotes the state transition probability to a new state $s'$ from a state $s$ when taking an action $a$ in step $h$.
With a slight abuse of notation, for functions $V : \S \to \R$ and $P_h$, we write $(P_h V)(x, a) = \sum_{y\in \S} V(y) P_h \paren{y \given x, a}$.

\looseness=-1
\paragraph{Policy and (regularized) value functions.}
A policy is defined as $\pi_{\cdot}\paren*{\cdot\given\cdot} : \HSA \to [0,1]$, where $\pi_h \paren{a \given s}$ gives the probability of taking an action $a$ at state $s$ in step $h$.
The set of all the policies is denoted as $\Pi$.
With an abuse of notation, for any policy $\pi $ and $Q: \SA \to \R$, let $\pi_h$ be an operator such that $(\pi_h Q) (s) = \sum_{a \in \A} \pi_h\paren{a\given s} Q (s, a)$.
For a policy $\pi$, transition kernel $P$, reward function $g: \HSA \to \R$, and entropy coefficient $\kappa \geq 0$, let $\qf{\pi, g}_{P, h}[\kappa]: \SA \to \R$ and $\vf{\pi, g}_{P, h}[\kappa]: \S \to \R$ denote the entropy-regularized value functions at step $h$ satisfying:
\begin{align*}
\qf{\pi, g}_{P,h}[\kappa] = g_h + \paren{P_h\vf{\pi, g}_{h+1, P}[\kappa]},\; 
\vf{\pi, g}_{P, h}[\kappa] = \pi_h
    \paren{\qf{\pi, g}_{P,h}[\kappa] - \kappa \ln \pi_h },
\;\text{ and }\; \vf{\pi, g}_{H+1, P}[\kappa] = \bzero\;.
\end{align*}
For \(\kappa = 0\), we omit \(\kappa\), e.g., \(\qf{\pi, g}_{P, h} \df \qf{\pi, g}_{P, h}[0]\). We denote \(h_\kappa \df h(1 + \kappa \ln A)\) for $h \in \bbrack{1, H}$.

\looseness=-1
For $h \in \bbrack{1, H}$, let $\occ{\pi}_{P, h}\in \Delta(\SA)$ denote the occupancy measure of $\pi$ in $P$ at step $h$ such that
\begin{equation}\label{eq:occupancy measure}
\occ{\pi}_{P, h}(s, a) = \P\paren*{s_h=s, a_h=a \given \pi, P}\quad \forall\; (h, s, a) \in \HSA\;,
\end{equation}
where the expectation is taken over all possible trajectories, in which $a_h \sim \pi_h \paren{\cdot \given s_h}$ and $s_{h+1} \sim P_h \paren{\cdot \given s_h, a_h}$.
With a slight abuse of notation, we write $\occ{\pi}_{P, h}(s) = \sum_{a \in \A} \occ{\pi}_{P, h}(s, a)$.

\looseness=-1
\paragraph{Learning Setup.}
An agent interacts with the CMDP for $K$ episodes using policies $\pi^{(1)}, \dots, \pi^{(K)} \in \Pi$. 
Each episode $k$ starts from $s_1$.
At step $h$ in episode $k$, the agent observes a state $s_h^{(k)}$, selects an action $a_h^{(k)} \sim \pi^{(k)}_h(\cdot\mid s_h^{(k)})$, and transitions to $s_{h+1}^{(k)} \sim P_h\paren{\cdot\given s_h^{(k)}, a_h^{(k)}}$.
The algorithm lacks knowledge of the transition kernel $P$, while $r$ and $u$ are known for simplicity. 
Extending our setting to unknown stochastic reward and utility is straightforward (see, e.g., \citet{efroni2020exploration}).

\looseness=-1
To handle a potentially large state space, we consider the following linear MDP assumption:
\begin{assumption}[Linear MDP]
\label{assumption:linear mdp}
We have a known feature map $\bphi: \SA \to \R^d$ satisfying:
there exist unknown $d$ (signed) measures
$\bmu_h \df \paren{\bmu^1_h, \ldots, \bmu_h^d} \in \R^{S\times d}$ such that
\(P_h \paren*{s' \given s, a} = \bmu_h(s')^\top \bphi (s, a)\), and known vectors $\btheta_h^r, \btheta_h^u \in \R^d$ such that 
$r_h(s, a) = \paren*{\btheta_h^r}^\top \bphi(s, a)$ and $u_h(s, a) = \paren*{\btheta_h^u}^\top \bphi(s, a)$.
We assume $\sup_{s, a}\norm*{\bphi(s, a)}_2 \leq 1$ 
and $\norm*{V^\top \bmu_h}_2 \leq \sqrt{d}$ for any $V \in \R^\S$ such that $\norm*{V}_\infty \leq 1$.
\end{assumption}

\looseness=-1
Let $\pi^\star \in \argmax_{\pi \in \Pisafe} \vf{\pi, r}_{P, 1}(s_1)$ be the optimal policy, where $\Pisafe \df \brace{\pi \given \vf{\pi, u}_{P, 1}(s_1) \geq b}$ is the set of safe policies. The goal is to achieve sublinear regret under \textbf{episode-wise} constraints:
\begin{align}\label{eq:CMDP-goal}
{\textstyle
\regret (K) 
\df 
\sum^K_{k=1} 
\vf{\pi^\star, r}_{P, 1}(s_1) - 
\vf{\pi^{(k)}, r}_{P, 1}(s_1)
= o\paren*{K}
\;\text{ such that }\;
\pi^{(k)} \in \Pisafe \quad \forall k \in [K]\;.
}
\end{align}
\looseness=-1
Finally, we assume the strictly safe policy similar to \Cref{sec:zero-vio bandit}.
\begin{assumption}[Safe policy]\label{assumption:slater}
We have access to $\pisafe \in \Pisafe$ and $\bslt > 0$ such that $\vf{\pisafe, u}_{P, 1}(s_1) - b \geq \bslt\;$.
\end{assumption}

\subsection{Technical Challenge: Optimistic-Pessimistic Optimization in Linear CMDP}\label{subsec:technical-challenge-MDP}

\looseness=-1
Our linear CMDP algorithm builds on \Banditalgo in \Cref{sec:zero-vio bandit}: deploying an optimistic-pessimistic policy when confident in \(\pisafe\); otherwise, it uses \(\pisafe\). 
We will logarithmically bound the number of \(\pisafe\) deployments, similar to \Cref{lemma:Ck-bound-main}, and ensure optimism through a linear mixture of policies, as in \Cref{lemma:alpha-feasibility-main}.
However, computing an optimistic-pessimistic policy in the linear CMDP setting, similar to \optpes, presents a non-trivial challenge.
This section outlines the difficulties.

\looseness=-1
Following standard linear MDP algorithm frameworks (e.g., \citet{jin2020provably,lykouris2021corruption}), for each $h, k$, let 
\(\beta^{(k)}_h: (s, a) \mapsto \norm{\bphi(s, a)}_{\paren{\bLambda^{(k)}_h}^{-1}}\) be the bonus, where \(\bLambda^{(k)}_h \df \rho \bI + \sum_{i=1}^{k-1} \bphi\paren{s_h^{(i)}, a_h^{(i)}} \bphi\paren{s_h^{(i)}, a_h^{(i)}}^\top\) and $\rho > 0$.
For any $V: \S \to \R$, let $\hP^{(k)}_h V$ be the next-step value estimation defined as:
$
\paren{\hP^{(k)}_h V}(s, a) \df \bphi(s, a)^{\top}\paren{\bLambda^{(k)}_h}^{-1} \sum_{i=1}^{k-1}\bphi\paren{s_h^{(i)}, a_h^{(i)}} V\paren{s_{h+1}^{(i)}}
$.
We construct the following optimistic and pessimistic value functions for reward and utility, respectively:
\begin{definition}[Clipped value functions]\label{def:clipped value functions}
\looseness=-1
Let $\co, \cp, C_\dagger, B_\dagger > 0$.
For each $k, h$, $\pi$, and $\kappa \geq 0$, 
define \(\oqf{\pi, r}_{(k), h}[\kappa], \oqf{\pi, \dagger}_{(k), h}, \pqf{\pi, u}_{(k), h}: \SA\to \R\) 
and \(\ovf{\pi, r}_{(k), h}[\kappa], \ovf{\pi, \dagger}_{(k), h}, \pvf{\pi, u}_{(k), h}: \S \to \R\) such that:
{
\begin{align*}
&\oqf{\pi, r}_{(k), h}[\kappa] \df r_h + \clip\brace{\co \beta^{(k)}_h + \hP^{(k)}_h \ovf{\pi, r}_{(k), h+1}[\kappa], \; 0, \; H_\kappa-h_\kappa},\;
&&\ovf{\pi, r}_{(k), h}[\kappa] \df \pi_h\paren{\oqf{\pi, r}_{(k), h}[\kappa] - \kappa \ln \pi_h}\;,\\
&\oqf{\pi, \dagger}_{(k), h} \df B_\dagger \beta^{(k)}_h + \clip\brace{C_\dagger \beta^{(k)}_h + \hP^{(k)}_h \ovf{\pi, \dagger}_{(k), h+1}, \; 0, \; B_\dagger (H-h)},\;
&&\ovf{\pi, \dagger}_{(k), h} \df \pi_h\oqf{\pi, \dagger}_{(k), h}\;,\\
&\pqf{\pi, u}_{(k), h} \df u_h + \clip\brace{-\cp \beta^{(k)}_h + \hP^{(k)}_h \pvf{\pi, u}_{(k), h+1}, \; 0, \; H-h},\quad \text{ and}
&&\pvf{\pi, u}_{(k), h} \df \pi_h\pqf{\pi, u}_{(k), h}\;.
\end{align*}
}
We set \(\ovf{\pi, r}_{(k), H+1}[\kappa] = \ovf{\pi, \dagger}_{(k), H+1} = \pvf{\pi, u}_{(k), H+1} = \bzero\).
For \(\kappa = 0\), omit \(\kappa\), e.g., \(\oqf{\pi, r}_{(k), h} \df \oqf{\pi, r}_{(k), h}[0]\).
\end{definition}
We will utilize \(\oqf{\pi, \dagger}_{(k), h}\) and \(\ovf{\pi, \dagger}_{(k), h}\) to compensate for the pessimism, similar to the bandit proof in \eqref{eq:mixture-policy-optimism}.\footnote{\looseness=-1 Increasing $C_r$ and clip-threshold could offer similar compensation, but separated values simplify analysis.}
Entropy regularization in $\oqf{\pi, r}_{(k), h}[\kappa]$ is for the later analysis.
The clipping operators are essential to avoid the propagation of unreasonable value estimates \citep{zanette2020frequentist}.

\looseness=-1
Using these value functions, one might consider extending \optpes to linear CMDPs by solving:
\begin{align}\label{eq:opt-pes-CMDP-informal}
\textstyle
\max_{\pi \in \Pi} \ovf{\pi, r}_{(k), 1}(s_1) + \ovf{\pi, \dagger}_{(k), 1}(s_1) 
\; \text{ such that }\; \pvf{\pi, u}_{(k), 1}(s_1) \geq b\;.
\end{align}
However, solving this \eqref{eq:opt-pes-CMDP-informal} is challenging due to (\(\mathrm{i}\)) \textbf{the large state space} in the linear CMDP setting (\(|\S| \gg 1\)) and (\(\mathrm{ii}\)) \textbf{the clipping operators} in \(\oqf{\pi, r}_{(k), h}\), \(\oqf{\pi, \dagger}_{(k), h}\), and \(\pqf{\pi, u}_{(k), h}\).

\looseness=-1
In tabular CMDPs with small $|\S|$, \citet{liu2021learning} and \citet{bura2022dope} 
used linear programming (LP) to solve similar optimistic-pessimistic optimization problems, achieving zero violation. However, the computational cost of LP scales with $|\S|$, making it impractical for linear CMDPs.

\looseness=-1
Another option is the Lagrangian method, which reformulates the constrained optimization as a min-max optimization:
$
\min_{\lambda \geq 0} 
\max_{\pi \in \Pi} \ovf{\pi, r}_{(k), 1}(s_1) + \ovf{\pi, \dagger}_{(k), 1}(s_1)  
+ 
\lambda\paren{\pvf{\pi, u}_{(k), 1}(s_1) - b}
$.
When the value functions are exact, i.e., $\ovf{\pi, \dagger}_{(k), h}+ \ovf{\pi, r}_{(k), h} + \pvf{\pi, u}_{(k), h} = \vf{\pi, r+B_\dagger \beta^{(k)}+\lambda u}_{P, h}$, 
this min-max is equivalent to \eqref{eq:opt-pes-CMDP-informal}, and the inner maximization reduces to a standard policy optimization \citep{altman1999constrained}.
Both favorable properties arise due to the linearity of the value function in the occupancy measure (see, e.g., \citet{paternain2019constrained}).
However, due to clipping, the value functions in \Cref{def:clipped value functions} may not be representable via occupancy measures, making the Lagrangian approach inapplicable.

\looseness=-1
To address this large-scale optimization challenge, instead of directly solving \eqref{eq:opt-pes-CMDP-informal}, we realize optimism and pessimism through a novel adaptation of the recent \textbf{softmax policy} technique for linear CMDPs \citep{ghosh2024towards,ghosh2022provably}, combined with \textbf{the $\boldsymbol{\pisafe}$ deployment technique} from \Cref{sec:zero-vio bandit}.

\begin{algorithm}[t!]
\caption{
Optimistic-Pessimistic Softmax Exploration for Linear CMDP}
\label{algo:zero-vio-linear MDP}
\DontPrintSemicolon
\LinesNumbered
{\small
\KwIn{Regr. coeff. $\rho = 1$, bonus scalers 
$\co=\tiO(dH)$, 
$\cp=\tiO(dH)$,
$C_\dagger=\tiO(d^2H^3\bslt^{-1})$,
$B_\dagger = \tiO\paren*{dH^2\bslt^{-1}}$,
entropy coeff. $\kappa = \widetilde{\Omega}\paren*{\bslt^3 H^{-4}d^{-1}K^{-0.5}}$, 
search length $T=\tiO(H)$, 
$\lambda$-threshold $C_\lambda = \tiO\paren*{dH^4\bslt^{-2}}$,
safe policy $\pisafe$, and iter. length $K \in \N$}
\For{$k = 1, \dots, K$}{
    Let \(\pvf{\pi, u}_{(k), h}\) be value function (\Cref{def:clipped value functions}) and $\pi^{(k), \lambda}$ be softmax policy (\Cref{def:composite-softmax-policy})\;
    {\color{blue}\tcc{$\pisafe$ trigger is implicitly tied to $\pisafe$ confidence (\Cref{lemma:trigger-condition-main})}}
    \lIf{\(\pvf{\pi^{(k), C_\lambda}, u}_{(k), 1}(s_1) < b\)}{Set \(\pi^{(k)}\df \pisafe\)
    } \label{line:pisafe-deploy}
    \lElseIf{\(\pvf{\pi^{(k), 0}, u}_{(k), 1}(s_1) \geq b\)}{Set \(\pi^{(k)}\df \pi^{(k), 0}\)} \label{line:pizero-deploy}
    \Else({\tcc*[h]{\color{blue} Do bisection-search to find safe $\pi^{(k), \lambda}$ with small $\lambda$}})
    {
    Set $\underline{\lambda}^{(k, 1)} \df 0$ and $\widebar{\lambda}^{(k, 1)} \df C_\lambda$.
    Let $\lambda^{(k, t)} \df \paren{\underline{\lambda}^{(k, t)} + \widebar{\lambda}^{(k, t)}} / 2$\;
    \For{$t = 1, \dots, T$ \label{line:binary-search}}{
        \lIf{\(\pvf{\pi^{(k), \lambda^{(k, t)}},u}_{(k), 1}(s_1) \geq b\;\)}{
        \(\underline{\lambda}^{(k, t+1)} \df \underline{\lambda}^{(k, t)}\; \) 
        and \(\;\widebar{\lambda}^{(k, t+1)} \df \lambda^{(k, t)}\)}
        \lElse{
        \(\underline{\lambda}^{(k, t+1)} \df \lambda^{(k, t)}\; \) 
        and \(\;\widebar{\lambda}^{(k, t+1)} \df \widebar{\lambda}^{(k, t)}\)
        }
    }
    Set $\pi^{(k)} \df \pi^{(k), \widebar{\lambda}^{(k, T)}}$\label{line:pik-deploy}
    }
    Sample a trajectory $(s^{(k)}_1, a^{(k)}_1, \dots, s^{(k)}_H, a^{(k)}_H)$ by deploying $\pi^{(k)}$\;
}
}
\end{algorithm}

\subsection{Algorithm and Analysis}

\looseness=-1
We summarize the proposed \MDPalgo in \Cref{algo:zero-vio-linear MDP} and analyze it under the parameters in its \textbf{Input} line.
All formal theorems and proofs in this section are in \Cref{appendix:MDP-regret-analysis}.
A key component of our algorithm is the \textbf{composite softmax policy}, which balances optimism and pessimism via $\lambda \geq 0$:
\begin{definition}[Composite softmax policy]\label{def:composite-softmax-policy}
For $\lambda \geq 0$, $\kappa > 0$, let $\pi^{(k), \lambda}\in \Pi$ be a policy such that 
{\small
\begin{align*}
\pi^{(k), \lambda}_{h}\paren*{\cdot \given s} = \softmax\paren*{\frac{1}{\kappa}\paren*{
\oqf{\pi^{(k), \lambda}, \dagger}_{(k), h}(s, \cdot) + 
\oqf{\pi^{(k), \lambda}, r}_{(k), h}[\kappa](s, \cdot) + 
\lambda \pqf{\pi^{(k), \lambda}, u}_{(k), h}(s, \cdot)}}\;.
\end{align*}
}
\end{definition}
$\pi^{(k), \lambda}$ can be computed iteratively in a backward manner for $h=H, \dots, 1$.
For this $\pi^{(k), \lambda}$, using the Lipschitz continuity of $\softmax(\cdot)$ (see \citet{ghosh2022provably}), the following confidence bounds hold:
\begin{lemma}[Confidence bounds]\label{lemma:opt-pes-MDP-main}
For any $(k, h)$, $\lambda \in [0, C_\lambda]$, $\pi \in \brace{\pi^{(k), \lambda}, \pisafe}$, w.p. at least $1-\delta$,
{\small
\begin{align*}
\vf{\pi,r}_{P, h} \leq \ovf{\pi,r}_{(k), h} \leq  \vf{\pi,r+ 2\co \beta^{(k)}}_{P, h},\;
\vf{\pi,B_\dagger\beta^{(k)}}_{P, h} \leq \ovf{\pi,\dagger}_{(k), h} \leq  \vf{\pi,(B_\dagger + 2C_\dagger) \beta^{(k)}}_{P, h},\;
\vf{\pi,u-2\cp \beta^{(k)}}_{P, h} \leq \pvf{\pi,u}_{(k), h} \leq \vf{\pi,u}_{P, h}.
\end{align*}
}
\end{lemma}
Using \Cref{lemma:opt-pes-MDP-main}, analogous to \Cref{subsec:zero-vio-bandit}, we next establish the zero-violation guarantee.

\subsubsection{Zero-Violation and Logarithmic Number of $\pisafe$ Deployments}\label{subsec:MDP-zero-vio}

\looseness=-1
In the softmax policy (\Cref{def:composite-softmax-policy}), $\lambda$ balances optimism and pessimism: a small $\lambda$ promotes exploration, while a large $\lambda$ prioritizes constraint satisfaction.
Building on this, \Cref{algo:zero-vio-linear MDP} conducts a \textbf{bisection search} to find the smallest feasible $\lambda$ while ensuring the pessimistic constraint holds (\Cref{line:pizero-deploy} to \Cref{line:pik-deploy}).
If a large $\lambda = C_\lambda$ fails to satisfy the constraint, the algorithm assumes no feasible pessimistic policy exists and deploys $\pisafe$ (\Cref{line:pisafe-deploy}).
Since the softmax policy is only deployed for $\lambda$ satisfying $\pvf{\pi^{(k), \lambda}, u}_{(k), 1}(s_1) \geq b$, \Cref{lemma:opt-pes-MDP-main} implies the following zero-violation guarantees:
\begin{corollary}[Zero-violation]\label{corollary:zero-violation-MDP}
W.p. at least $1-\delta$, \Cref{algo:zero-vio-linear MDP} satisfies $\pi^{(k)} \in \Pisafe$ for any $k$.
\end{corollary}

\looseness=-1  
Next, we bound the number of $\pisafe$ deployments to achieve sublinear regret.
To this end, similar to the bandit warm-up (\Cref{sec:zero-vio bandit}), \textbf{we relate $\boldsymbol{\pisafe}$ deployment to $\boldsymbol{\pisafe}$ uncertainty level} and logarithmically bound the number of uncertain iterations.
The following \Cref{lemma:trigger-condition-main} ensures that, if \Cref{algo:zero-vio-linear MDP} is confident in $\pisafe$ and runs with appropriate $C_\lambda$ and $\kappa$, then $\pisafe$ is not deployed.
\begin{definition}[$\pisafe$ unconfident iterations]\label{def:unconf-set MDP}
Let $\unconfMDP$ be the iterations when \Cref{algo:zero-vio-linear MDP} is unconfident in $\pisafe$, i.e., 
$
\unconfMDP \df 
\brace{k \in \bbrack{1, K} \given 
\vf{\pisafe, \beta^{(k)}}_{P, 1}(s_1) > \frac{\bslt}{4\cp}
}
$.
Let $\unconfMDP^\complement \df \bbrack{1, K} \setminus \unconfMDP$ be its complement.
\end{definition}
\begin{lemma}[Implicit $\pisafe$ deployment trigger]\label{lemma:trigger-condition-main}
When $C_\lambda \geq \frac{8H_\kappa^2 (B_\dagger + 1)}{\bslt}$ and $\kappa \leq \frac{\bslt^2}{32H_\kappa^2 (B_\dagger + 1)}$, then w.p. at least $1-\delta$, it holds that $\pvf{\pi^{(k), C_\lambda}, u}_{(k), 1}(s_1) \geq b$ for all $k \in \unconfMDP^\complement$.
\end{lemma}
\looseness=-1
Essentially, the proof of \Cref{lemma:trigger-condition-main} relies on the following monotonic property of the value function for the softmax policy: if the value estimation is exact, increasing $\lambda$ monotonically improves safety.
\begin{lemma}[Softmax value monotonicity]\label{lemma:softmax-value-monotonicity-main}
\looseness=-1
For $\lambda \geq 0$, let $\pi^\lambda$ be a softmax policy such that 
\(
\pi^\lambda_h \paren{\cdot \given s} = \softmax\paren{\frac{1}{\kappa}\paren{
    \qf{\pi, r}_{P, h}[\kappa](s, \cdot)
    + \lambda \qf{\pi, u}_{P, h}(s, \cdot)}}
\).
Then, \(\vf{\pi^\lambda, u}_{P, 1}(s_1)\) is monotonically increasing in $\lambda$.
\end{lemma}
While the true value function enjoys this monotonicity, the estimated value $\pvf{\pi^{(k), \lambda}, u}_{(k), 1}(s_1)$ may not, as $\hP^{(k)}_h V$ can take negative values even when $V$ is positive. 
This complicates the proof of \Cref{lemma:trigger-condition-main}.
To address this, we leverage \Cref{lemma:opt-pes-MDP-main}, which sandwiches the estimated values by some true values.
We prove \Cref{lemma:trigger-condition-main} by showing that, for sufficiently large $C_\lambda$, any sandwiched value satisfies the constraint under pessimism, implying that the estimated value also satisfies it.
This novel result enables bisection search to adjust $\lambda$, making \MDPalgo more computationally efficient than \citet{ghosh2024towards}.
The detailed proofs of \Cref{lemma:trigger-condition-main,,lemma:softmax-value-monotonicity-main} are provided in \Cref{subsec:safe-softmax-policy-exists-proof}.

\looseness=-1
Finally, the following theorem ensures that the number of $\pisafe$ deployment scales logarithmic to $K$, as in \Cref{lemma:Ck-bound-main}.
The proof follows from extending the bandit’s proof of \Cref{lemma:Ck-bound-main} to CMDPs.
\begin{theorem}[Logarithmic $|\unconfMDP|$ bound]\label{lemma:Ck-bound-MDP-main}
It holds w.p. at least $1-\delta$ that
\(
\abs{\unconfMDP}
\leq
\cO\paren*{d^3 H^4\bslt^{-2} \ln KH \delta^{-1}}
\).
\end{theorem}

\subsubsection{Regret Analysis}\label{subsec:MDP-regret-analysis}
\looseness=-1
The remaining task is to ensure sublinear regret. By \Cref{lemma:Ck-bound-MDP-main,,lemma:opt-pes-MDP-main}, the regret is decomposed as:
{\small
\begin{align*}
\regret(K) 
\leq
\tiO\paren*{\frac{d^3 H^4}{\xi^{2}}}
+ 
\underline{
 \sum_{k \in \unconfMDP^\complement} \vf{\pi^{(k)}, 2\co\beta^{(k)}}_{P, 1}(s_1)
}_{\;\circled{1}}
+
\underline{\sum_{k\in \unconfMDP^\complement} \paren*{\vf{\pi^\star, r}_{P, 1}(s_1) - \ovf{\pi^{(k)}, r}_{(k), 1}[\kappa](s_1)}}_{\;\circled{2}}
+ \kappa KH \ln A 
\;,
\end{align*}
}
where the last term arises from the entropy regularization ($\vf{\pi, r}_{P, 1}(s_1)[\kappa] - \vf{\pi, r}_{P, 1}(s_1) \leq \kappa H \ln \aA$).
Using the elliptical potential lemma for linear MDPs \citep{jin2020provably}, we obtain $\circled{1} \leq \tiO\paren{\co H\sqrt{d K}}$.

\looseness=-1
We now bound $\circled{2}$. Note that for any $k \in \unconfMDP^\complement$, due to \Cref{lemma:trigger-condition-main}, $\pi^{(k)}$ is the softmax policy by \Cref{line:pik-deploy}.
To bound $\circled{2}$, following a similar approach to \Cref{lemma:optimism-main} in the bandit, we replace $\pi^\star$ with a mixture policy that satisfies the pessimistic constraint.
To this end, we utilize the following lemmas.
\begin{definition}[Mixture policy]\label{def:mixture-policy}
For $\alpha \in [0, 1]$, let $\pi^\alpha$ be a mixture policy such that, for any $h$, 
\(\occ{\pi^\alpha}_{P, h} = (1 - \alpha) \occ{\pisafe}_{P, h} + \alpha \occ{\pi^\star}_{P, h}\).
Such a $\pi^\alpha$ is ensured to exists for any $\alpha \in [0, 1]$ \citep{borkar1988convex}.
\end{definition}

\begin{lemma}[Safe and optimistic mixture policy]\label{lemma:mixture-safe-optimism-MDP-main}
Let $\alpha^{(k)} \df \frac{\bslt}{\bslt + 2 \vf{\pi^\star, 2\cp \beta^{(k)}}_{P, 1}(s_1)}$.
If $B_\dagger \geq \frac{4\cp H}{\bslt}$, then for any $k \in \unconfMDP^\complement$, it holds
($\mathrm{i}$)
\(
\vf{\pi^{\alpha^{(k)}}, u-2\cp \beta^{(k)}}_{P, 1}(s_1) \geq b
\) and 
($\mathrm{ii}$)
\(
\vf{\pi^{\alpha^{(k)}}, r+B_\dagger\beta^{(k)}}_{P, 1}(s_1) \geq
\vf{\pi^\star, r}_{P, 1}(s_1) 
\).
\end{lemma}
We note that the mixture policy $\pi^{\alpha^{(k)}}$ is introduced only for the regret analysis; it is not required in the actual algorithm.
Since $\widebar{\lambda}^{(k, T)}$ is chosen to satisfy $\pvf{\pi^{(k)}, u}_{(k), 1}(s_1) < b$ and $b \leq \vf{\pi^{\alpha^{(k)}}, u - 2\cp \beta^{(k)}}_{P, 1}(s_1)$ holds by \Cref{lemma:mixture-safe-optimism-MDP-main}, 
{\footnotesize
\begin{align*}
&
\begin{rcases}
\circled{2} \leq
&\sum_{k\in \unconfMDP^\complement} 
\Big(\vf{\pi^{\alpha^{(k)}}, B_\dagger \beta^{(k)}}_{P, 1}(s_1) 
+ \vf{\pi^{\alpha^{(k)}}, r}_{P, 1}[\kappa](s_1)
+ \widebar{\lambda}^{(k, T)} \vf{\pi^{\alpha^{(k)}}, u - 2\cp \beta^{(k)}}_{P, 1}(s_1)\\
&
\quad\quad\quad\quad
- \ovf{\pi^{(k)}, \dagger}_{(k), 1}(s_1)
- \ovf{\pi^{(k)}, r}_{(k), 1}[\kappa](s_1)
- \widebar{\lambda}^{(k, T)} \pvf{\pi^{(k)}, u}_{(k), 1}(s_1)
\Big)
\end{rcases}\; \circled{3}\\
&
\quad\quad\quad\quad
{\textstyle
+
\underline{\sum_{k\in \unconfMDP^\complement}  \ovf{\pi^{(k)}, \dagger}_{(k), 1}(s_1) }_{\;\circled{4}}
+ 
\underline{
C_\lambda \sum_{k \in \unconfMDP^\complement} 
\paren*{
\pvf{\pi^{(k), \widebar{\lambda}^{(k, T)}}, u}_{(k), 1}(s_1) - 
\pvf{\pi^{(k), \underline{\lambda}^{(k, T)}}, u}_{(k), 1}(s_1)
}
}_{\;\circled{5}} \;.
}
\end{align*}
}
\looseness=-1
Using \Cref{lemma:opt-pes-MDP-main}, similar to $\circled{1}$, we have
\(
\circled{4} 
\leq \tiO\paren*{(B_\dagger + C_\dagger)H\sqrt{d K}}
\).
The term $\circled{5}$ is controlled by the bisection search width ($\widebar{\lambda}^{(k, T)} - \underline{\lambda}^{(k, T)}$) and the following sensitivity of $\pvf{\pi^{(k), \lambda}, u}_{(k), 1}(s_1)$ to $\lambda$.
\begin{lemma}\label{lemma:lambda-sensitive-main}
For any $k$ and $\lambda \in [0, C_\lambda]$, we have
\(
\abs*{\pvf{\pi^{(k), \lambda}, u}_{(k), 1}(s_1)
- \pvf{\pi^{(k), \lambda + \varepsilon}, u}_{(k), 1}(s_1)}
\leq \cO\paren*{(KH)^H}\varepsilon\)
\end{lemma}
\citet{ghosh2024towards} also derived a similar exponential bound (see their Appendix C).
Due to the update rule of the bisection search, setting the search iteration to $T = \tiO(H)$ ensures that \(\circled{5} \leq \tiO\paren{1}\).

\looseness=-1
For $\circled{3}$, using a modification of the so-called value-difference lemma \citep{shani2020optimistic}, we have
\begin{align}\label{eq:third-term-decompose}
\textstyle
\circled{3} =
\sum_{k \in \unconfMDP^\complement}
\vf{\pi^{\alpha^{(k)}}, f^1}_{P, 1}(s_1) 
- \vf{\pi^{\alpha^{(k)}}, f^2}_{P, 1}(s_1) 
- \widebar{\lambda}^{(k, T)}\vf{\pi^{\alpha^{(k)}}, 2\cp \beta^{(k)}}_{P, 1}(s_1)\;, 
\end{align}
where $f^1, f^2: \HSA\to \R$ are functions such that, for any $h$,
\begin{align*}
&
{\scriptstyle
f^1_h = 
\paren*{\pi^{\alpha^{(k)}}_h - \pi^{(k)}_h}
\paren*{
\oqf{\pi^{(k)}, \dagger}_{(k), h} +
\oqf{\pi^{(k)}, r}_{(k), h}[\kappa] + \widebar{\lambda}^{(k)} 
\pqf{\pi^{(k)}, u}_{(k), h}
}
- \kappa\pi^{\alpha^{(k)}}_h\ln \pi^{\alpha^{(k)}}_h 
+ \kappa \pi^{(k)}_h \ln \pi^{(k)}_h
}
\\
&
\text{ and }\;
{\scriptstyle
f^2_h = 
\paren*{\oqf{\pi^{(k)}, r}_{(k), h}[\kappa]
- r_h - P_h \ovf{\pi^{(k)}, r}_{(k), h+1}[\kappa]}
+ \widebar{\lambda}^{(k, T)}\paren*{u_h + P_h \pvf{\pi^{(k)}, u}_{(k), h+1} - \pqf{\pi^{(k)}, u}_{(k), h}}
+ \paren*{\oqf{\pi^{(k)}, \dagger}_{(k), h} - B_\dagger \beta^{(k)} - P_h \ovf{\pi^{(k)}, \dagger}_{(k), h+1}}\;.
}
\end{align*}
\looseness=-1
\textbf{Our use of the softmax policy with entropy regularization is crucial for bounding $\circled{3}$.}
Since the analytical maximizer of the regularized optimization 
\(
\max_{\pi \in \sP(\A)}
\sum_{a \in \A}
\pi\paren{a} \paren*{\bx(a) - \kappa \ln \pi\paren{a}}
\) is given by  $\softmax\paren*{\frac{1}{\kappa} \bx(\cdot)}$, it follows that $f^1$ is non-positive, implying $\vf{\pi^{\alpha^{(k)}}, f^1}_{P, 1}(s_1) \leq 0$.
Additionally, applying \Cref{lemma:opt-pes-MDP-main}, we derive 
\(
f^2_h \geq -\widebar{\lambda}^{(k, T)} 2\cp \beta^{(k)}_h
\), which leads to
\(-\vf{\pi^{\alpha^{(k)}}, f^2}_{P, 1}(s_1) 
-\widebar{\lambda}^{(k, T)}\vf{\pi^{\alpha^{(k)}}, 2\cp \beta^{(k)}}_{P, 1}(s_1) 
\leq 0\).
By substituting these bounds into \Cref{eq:third-term-decompose}, we obtain $\circled{3} \leq 0$.

\looseness=-1
By combining all the results, \Cref{algo:zero-vio-linear MDP} achieves the following guarantees:
\begin{theorem}\label{theorem:MDP-regret-main}
\looseness=-1
If \MDPalgo is run with the parameters listed in its \hyperref[algo:zero-vio-linear MDP]{\textbf{Input}} line, w.p. at least $1-\delta$,
\begin{align*}
\pi^{(k)} \in \Pisafe \; \forall k \in \bbrack{1, K} 
\; \text{ and }\; 
\regret (K) \leq 
\underline{
\tiO\paren{H^{2} \sqrt{d^3 K}}
}_{\;(\mathrm{i})}
+ 
\underline{
\tiO\paren{d^3H^4 \bslt^{-2}}
}_{\;(\mathrm{ii})}
+ 
\underline{
\tiO\paren{H^{4} \bslt^{-1}\sqrt{d^5 K}}
}_{\;(\mathrm{iii})}\;.
\end{align*}
\end{theorem}
Notably, \textbf{\cref{theorem:MDP-regret-main} is the first linear CMDP result achieving zero episode-wise constraint violations and sublinear regret.}
We conclude this section by discussing the regret bound quality and computational cost of \MDPalgo. 

\begin{remark}[Can we do better?]
\looseness=-1
Without the constraint---i.e., removing ({\rm{ii}}) and ({\rm{iii}})---our bound matches the \(\tiO\paren{H^{2} \sqrt{d^3 K}}\) regret of the fundamental LSVI-UCB algorithm \citep{jin2020provably}.
The $\xi^{-2}$ term in ({\rm{ii}}) is unavoidable \citep{pacchiano2021stochastic}. 
The $\xi^{-1}$ dependence in ({\rm iii}) remains unresolved, yet appears in all the existing safe RL literature \citep{liu2021learning,bura2022dope,yu2024improved,roknilamouki2025provably,amani2019linear,pacchiano2021stochastic,hutchinson2024directional}. 
These observations suggest that the bound is tight in $\xi^{-1}$ and $K$.

\looseness=-1
As for $d$ and $H$, the term ({\rm{iii}}) introduces an extra $dH^2$ factor over ({\rm{i}}).
Similar deterioration has been observed in tabular CMDPs \citep{liu2021learning,vaswani2022near} and partially mitigated via Bernstein-type bonus analysis \citep{yu2024improved}.

\looseness=-1
While improvement may be possible, it is unclear whether our $\sqrt{d^5 H^8}$ dependence is overly loose when compared with existing CMDP regret bounds (e.g., $\sqrt{d^3 H^4}$ by \citet{ghosh2024towards}), since none of the existing results achieve episode-wise safe exploration. In general, regret or sample complexity bounds under different safety requirements are not directly comparable, even if the problem settings appear similar. For example, \citet{vaswani2022near} shows that the sample complexity lower bound for tabular CMDPs exhibits a worse dependence on the horizon under a strict safe policy requirement than when small violations are allowed. Establishing a formal regret lower bound for our setting would be necessary to assess the tightness of our result, but this is beyond the scope of the current paper.

\end{remark}

\begin{remark}[Computational cost]\label{remark:computational cost}  
\looseness=-1
\Cref{algo:zero-vio-linear MDP} requires up to \( T \) value evaluations (\Cref{def:clipped value functions}) and policy computation (\Cref{def:composite-softmax-policy}).  
Using the bisection search, we bound \( T = \tiO(H) \), reducing the computational cost per-iteration to \( \tiO(H \times \text{[value \& policy comp.]}) \).  
As this cost scales polynomially with \( \aA, H, \) and \( d \) \citep{lykouris2021corruption}, \textbf{\MDPalgo runs in polynomial time}---an improvement over recent \citet{ghosh2024towards}, which achieves $\tiO(\sqrt{K})$ violation regret but incurs an exponential \( K^H \) cost.
\end{remark}

\section{Conclusion}\label{sec:conclusion}

\looseness=-1
This paper proposed \MDPalgo, the first RL algorithm achieving both sublinear regret and episode-wise constraint satisfaction in linear CMDPs (\Cref{theorem:MDP-regret-main}). 
Our approach builds on optimistic-pessimistic exploration with two key innovations: ($\mathrm{i}$) a novel deployment rule for $\pisafe$ and ($\mathrm{ii}$) a softmax-based approach for efficiently implementing optimistic-pessimistic policies in linear CMDPs.

\looseness=-1
\paragraph{Experiments.}
We numerically evaluate \MDPalgo on several linear CMDP environments to support our theoretical results.
We compare \MDPalgo with the prior state-of-the-art linear CMDP algorithm of \citet{ghosh2024towards} and the tabular algorithm called DOPE \citep{bura2022dope}.
Across all environments, \MDPalgo achieves sublinear regret with zero constraint violation, while \citet{ghosh2024towards} shows positive violation regret. 
These results empirically validate \cref{theorem:MDP-regret-main}.
While DOPE also achieves zero violation, its use is limited to the tabular settings where $\aS$ is small.
This highlights the computational tractability of our \MDPalgo in large $\aS$, which supports \cref{remark:computational cost}.
All the results and details are deferred to \cref{sec: experiment}.

\paragraph{Limitation and future work.}
\looseness=-1
\MDPalgo achieves computational efficiency by the bisection search over $\lambda \in [0, C_\lambda]$, which works in the single-constraint setting thanks to the monotonicity in \Cref{lemma:softmax-value-monotonicity-main}.
However, extending our method to the \textbf{multi-constraint setting} is non-trivial, as $\lambda$ becomes a vector, requiring a vectorized version of the monotonicity lemma.
Nonetheless, all theoretical results in \Cref{table:algorithms} are also limited to single-constraint settings, meaning our work still advances the state of the art in safety.
An efficient and safe algorithm for multi-constraint settings remains open for future work.

\looseness=-1
Another future direction is to extend the analysis to adversarial initial states $s_1$. This extension is non-trivial, as our core techniques rely on the fixed-state assumption. For example, we control the number of safe policy deployments by evaluating the bonus-return function $\vf{\pisafe, \beta^{(k)}}_{P, 1}(s_1)$ (\cref{def:unconf-set MDP}), which explicitly depends on $s_1$. 
The bound on the number of $\pi^{\mathrm{sf}}$ deployments (\cref{lemma:Ck-bound-MDP-main}) and the existence of optimistic–pessimistic policies (\cref{lemma:mixture-safe-optimism-MDP-main}) also depend on this initial state. 
Extending these analyses to handle adversarial $s_1$ would entail non-trivial technical challenges and is left for future work.

\begin{ack}
This work is supported by JST Moonshot R\&D Program Grant Number JPMJMS2236.
AG acknowledges NJIT Start-up fund indexed number 172884.
\end{ack}

\bibliography{mybib}
\bibliographystyle{plainnat}

\newpage
\appendix

\tableofcontents
\newpage
\section{List of Symbols}

\renewcommand{\nomname}{}
\renewcommand{\nompreamble}{The next list describes several symbols used within this paper.}

\nomenclature[A, 0]{$\sP(\S)$}{Set of probability distributions over a set $\S$}
\nomenclature[A, 1]{$\bbrack{a,b}$}{Set of integers defined by $\brace*{a,\dots, b}$}
\nomenclature[A, 2]{$\clip\brace{\bx, a, b}$}{Clipping function which returns $\bx'$ with $\bx'_i = \min\brace{\max\brace{\bx_i, a}, b}$ for each $i$}
\nomenclature[A, 3]{$\bzero$, $\bone$}{Vectors such that $\bzero \df \paren{0, \ldots, 0}^\top$ and $\bone \df \paren{1, \ldots, 1}^\top$}
\nomenclature[A, 4]{$\norm*{\bx}_\bA$}{For a positive definite matrix $\bA \in \R^{d\times d}$ and $\bx \in \R^d$, we denote $\norm*{\bx}_\bA = \sqrt{\bx^\top \bA \bx}$}
\nomenclature[A, 5]{$a_n=O\left(b_n\right)$}{There exist constants $C>0$ and $N \in \N$ such that $a_n \leq C b_n$ for all $n \geq N$}
\nomenclature[A, 6]{$\widetilde{O}(\cdot)$}{Hide the polylogarithmic factors from $O(\cdot)$}
\nomenclature[A, 7]{$\softmax(\bx)$}{Softmax distribution satisfying $\softmax(\bx)_i = \exp(\bx_i) / \paren*{\sum_{i} \exp(\bx_i)}$}
\nomenclature[A, 8]{$\dist_\infty$}{Distance metric for two functions $Q, Q': \SA\to \R$ or $V, V': \S\to \R$ (see \cref{def:distance-metrics})}
\nomenclature[A, 9]{$\dist_1$}{Distance metric for two functions $\pi, \pi': \S\to \sP(\A)$ (see \cref{def:distance-metrics})}

\nomenclature[B, 1]{$\A \subset \R^d$}{Set of actions}
\nomenclature[B, 11]{$K$}{Round length of the bandit problem}
\nomenclature[B, 12]{$r^{(k)}, u^{(k)}$}{Reward and utility at $k$-th round}
\nomenclature[B, 13]{$\btheta_r, \btheta_u \in \R^d$}{Unknown vectors for reward and utility}
\nomenclature[B, 14]{$\varepsilon_r, \varepsilon_u$}{$R$-sub-Gaussian random noises}
\nomenclature[B, 15]{$r_\pi, u_\pi$}{$g_\pi = \E_{\ba \sim \pi} [\angle{\btheta_g, \ba}]$ for both $g \in \brace{r, u}$}
\nomenclature[B, 16]{$\Pisafe$}{Set of safe policies $\brace{\pi \given u_\pi \geq b}$}
\nomenclature[B, 17]{$\pisafe$}{Safe policy}
\nomenclature[B, 18]{$\xi > 0$}{Safety of $\pisafe$ such that $u_{\pisafe} - b \geq \xi$}
\nomenclature[B, 19]{$\bLambda^{(k)}$}{Gram matrix defined by \(\bLambda^{(k)} \df \rho \bI + \sum_{i=1}^{k-1} \ba^{(i)} \paren{\ba^{(i)}}^\top \)}
\nomenclature[B, 20]{$\hbtheta^{(k)}_r$, $\hbtheta^{(k)}_u$}{Estimates of $\btheta_r$ and $\btheta_u$}
\nomenclature[B, 21]{$\beta_\pi^{(k)}$}{Bonus function}
\nomenclature[B, 22]{$C_r, C_u$}{Bonus scalers for reward and utility, respectively}
\nomenclature[B, 23]{$\unconfBandit$}{Set of iterations when \Cref{algo:zero-vio-bandit} is unconfident in $\pisafe$ (see \cref{def:unconf-set})}

\nomenclature[C, 0]{$K$}{Number of episodes of the CMDP problem}
\nomenclature[C, 1]{$H$}{Horizon}
\nomenclature[C, 111]{$\S, \A$}{State space and action spaces}
\nomenclature[C, 112]{$P$}{Transition kernel}
\nomenclature[C, 113]{$r, u$}{Reward and utility functions}
\nomenclature[C, 114]{$s_1$}{Initial state}
\nomenclature[C, 13]{$\vf{\pi, g}_{P, h}[\kappa]$}{Regularized state value function for a reward function $g$ with an entropy coefficient $\kappa$}
\nomenclature[C, 14]{$\qf{\pi, g}_{P, h}[\kappa]$}{Regularized action value function for a reward function $g$ with an entropy coefficient $\kappa$}
\nomenclature[C, 15]{$h_\kappa$}{Shorthand of $h(1+\kappa \ln A)$}
\nomenclature[C, 16]{$\occ{\pi}_{P, h}$}{Occupancy measure of $\pi$ in $P$ at step $h$}
\nomenclature[C, 17]{$\bphi: \SA \to \R^d$}{Feature map of the linear CMDP (see \cref{assumption:linear mdp})}
\nomenclature[C, 18]{$\bmu_h$}{$d$-signed measures specifying the transition probabilities (see \cref{assumption:linear mdp})}
\nomenclature[C, 19]{$\btheta^r_h, \btheta^u_r \in \R^d$}{Known vectors specifying the reward and utilitiy functions (see \cref{assumption:linear mdp})}
\nomenclature[C, 20]{$\Pisafe$}{Set of safe policies $\brace{\pi \given \vf{\pi, u}_{P, 1}(s_1) \geq b}$}
\nomenclature[C, 21]{$\pisafe$}{Safe policy}
\nomenclature[C, 22]{$\xi > 0$}{Safety of $\pisafe$ such that $\vf{\pi, u}_{P, 1}(\pisafe) - b \geq \xi$}
\nomenclature[C, 23]{$\bLambda^{(k)}_h$}{Gram matrix defined by \(\bLambda^{(k)}_h \df \rho \bI + \sum_{i=1}^{k-1} \bphi\paren{s_h^{(i)}, a_h^{(i)}} \bphi\paren{s_h^{(i)}, a_h^{(i)}}^\top\)
}
\nomenclature[C, 24]{$\hP^{(k)}_h V$}{Next value estimation: $
\paren{\hP^{(k)}_h V}(s, a) \df \bphi(s, a)^{\top}\paren{\bLambda^{(k)}_h}^{-1} \sum_{i=1}^{k-1}\bphi\paren{s_h^{(i)}, a_h^{(i)}} V\paren{s_{h+1}^{(i)}}
$}
\nomenclature[C, 25]{$\oqf{\pi, r}_{(k), h}[\kappa], \pqf{\pi, u}_{(k), h}, \oqf{\pi, \dagger}_{(k), h}$}{Clipped value functions defined in \cref{def:clipped value functions}}
\nomenclature[C, 26]{$C_r, C_u, C_\dagger, B_\dagger$}{Bonus scalers used in \cref{def:clipped value functions}}
\nomenclature[C, 261]{$\cN_\varepsilon$}{$\varepsilon$-cover of a certain set}
\nomenclature[C, 27]{$\unconfMDP$}{Set of iterations when \Cref{algo:zero-vio-linear MDP} is unconfident in $\pisafe$ (see \cref{def:unconf-set MDP})}
\nomenclature[C, 271]{$\pi^{(k), \lambda}_{h}$}{Softmax policy with a parameter $\lambda$ (see \cref{def:composite-softmax-policy})}
\nomenclature[C, 28]{$\lambda$}{Parameter to balance optimism and pessimism of $\pi^{(k), \lambda}_h$ (see \cref{def:composite-softmax-policy})}
\nomenclature[C, 281]{$C_\lambda$}{Maximum value of $\lambda$}
\nomenclature[C, 282]{$[\underline{\lambda}^{(k, t)}, \widebar{\lambda}^{(k, t)}]$}{Search space of the bisection search at iteration $t$ in episode $k$ (see \cref{algo:zero-vio-linear MDP})}
\nomenclature[C, 29]{$T$}{Iteration length of the bisection search (see \cref{algo:zero-vio-linear MDP})}
\nomenclature[C, 30]{$\cQr, \cQu, \cQc$}{Function classes for $Q$-functions defined in \cref{def:abs-Q-function}}
\nomenclature[C, 31]{$\Qcomp$}{Function class for the composite of $Q$-functions defined in \cref{def:composite-Q-function}}
\nomenclature[C, 32]{$\tiPi$}{Class for softmax policies defined in \cref{def:policy-class}}
\nomenclature[C, 33]{$\cVr, \cVu, \cVc$}{Function classes for $V$-functions defined in \cref{def:V-cover}}
\nomenclature[C, 34]{$\delta^{\pi, r}_{(k)}, \delta^{\pi, u}_{(k)}, \delta^{\pi, \dagger}_{(k)}$}{Discrepancies between the estimated and true $Q$-functions (see \cref{def:delta-Q})}
\nomenclature[C, 35]{$\Deltar, \Deltau, \Deltac$}{Function classes for $\delta^{\pi, r}_{(k)}, \delta^{\pi, u}_{(k)}, \delta^{\pi, \dagger}_{(k)}$ (see \cref{def:delta-Q})}

\printnomenclature

\newpage

\section{Related Work}\label{sec:related work}

\subsection{Related Algorithms}
\looseness=-1
Building on the seminal work of \citet{efroni2020exploration}, numerous safe RL algorithms for CMDPs have been developed, broadly categorized into linear programming (LP) approaches and Lagrangian-based approaches.

\looseness=-1
\paragraph{Linear programming.}
LP approaches formulate CMDPs as linear optimization problems \citep{altman1999constrained}, solving them using an estimated transition kernel.
\citet{efroni2020exploration} introduced a basic sublinear regret algorithm, while \citet{hasanzadezonuzy2021learning} provided $(\varepsilon,\delta)$-PAC guarantees, ensuring the algorithm outputs a near-optimal policy.
However, these methods permit constraint violations during exploration, making them unsuitable for safety-critical applications.
\citet{liu2021learning} and \citet{bura2022dope} developed LP-based algorithms that achieve sublinear regret while maintaining episode-wise zero-violation guarantees by incorporating optimistic-pessimistic estimation into the LP formulation.

\looseness=-1
LP-based approaches in tabular settings, however, suffer from computational costs that scale with the size of the state space, making them impractical for linear CMDPs. 
While several studies propose LP algorithms for linear MDPs \citep{neu2020unifying,bas2021logistic,neu2023efficient,lakshminarayanan2017linearly,gabbianelli2024offline}, these methods either use occupancy measures as decision variables---which can be exponentially large for large state spaces---or require a set of feature vectors that sufficiently cover the state space, which may not be feasible in our exploration settings.
Moreover, as described in \Cref{subsec:technical-challenge-MDP}, the estimated value functions in linear CMDPs with exploration require clipping operators, further complicating the use of occupancy-measure-based approaches like LP methods in our setting.

\looseness=-1
\paragraph{Lagrangian approach.}
Lagrangian approaches reformulate the constrained optimization\\ $\max_\pi \brace*{f(\pi) \given h(\pi) \geq 0}$ as a min-max optimization $\min_{\lambda \geq 0} \max_{\pi}\brace*{f(\pi) + \lambda h(\pi)}$, and simultaneously optimize both $\pi$ and $\lambda$.
When an algorithm gradually updates \(\pi\) and then adjusts \(\lambda\) incrementally, it is referred to as a \textbf{primal-dual (PD)} algorithm \citep{ding2023last}.  
In contrast, if \(\lambda\) is updated only after fully optimizing \(\pi\) in the inner maximization, it is known as a \textbf{dual} approach \citep{ying2022dual}.
Since the inner maximization reduces to standard policy optimization, Lagrangian methods integrate naturally with scalable methods such as policy gradient and value iteration.

\looseness=-1
For the tabular settings, \citet{wei2021provably,muller2024truly} develop model-free primal-dual algorithms with sublinear regret, while \citet{wei2022provably} extends this approach to the average-reward setting. 
\citet{zeng2022finite,kitamura2024policy} propose \((\varepsilon,\delta)\)-PAC primal-dual algorithms, and \citet{vaswani2022near} achieved the PAC guarantee via dual approach.

\looseness=-1
Beyond tabular settings, \citet{ding2021provably} propose PD algorithms with linear function approximation, achieving sublinear regret guarantees.
\citet{gosh2023achiving} extend this to the average-reward linear CMDPs.
\citet{ghosh2022provably} take a dual approach, also attaining sublinear regret in the finite-horizon settings.

\looseness=-1
These PD and dual algorithms, however, do not ensure episode-wise zero violation.
Intuitively, the key issue lies in their $\lambda$-adjustment strategy, which updates $\lambda$ only incrementally.
For example, the basic PD and dual algorithms by \citet{efroni2020exploration} updates $\lambda$ using $\lambda^{(k+1)} \leftarrow \lambda^{(k)} + \alpha \cdot [\text{violation}]$, where $\alpha$ is a small learning rate.
Since $\lambda$ controls constraint satisfaction, if the current policy fails to satisfy constraints adequately, $\lambda$ should be increased sufficiently before the next policy deployment.

\looseness=-1
Following this principle, \citet{ghosh2024towards} propose a dual approach that searches for an appropriate $\lambda$ within each episode, leading to a tighter violation regret guarantee than \citet{ghosh2022provably}.
However, due to the lack of pessimistic constraint estimation, their method does not ensure episode-wise safety and allows constraint violations.
Like \citet{ghosh2024towards}, our \MDPalgo searches for the best $\lambda$ in each episode.
However, unlike their approach, \MDPalgo controls $\lambda$ with pessimism, ensuring zero violation, and guarantees the existence of a feasible $\lambda$ by deploying a sufficient number of $\pisafe$.

\subsection{Related Safety Types}

\paragraph{Instantaneous safety.}
\looseness=-1  
Unlike our episode-wise safety, instantaneous safety defines exploration as safe if it satisfies \( u_h(s^{(k)}_h, a^{(k)}_h) \geq b \) for all \( h \) and \( k \) \citep{pacchiano2021stochastic, pacchiano2024contextual, hutchinson2024directional, shi2023near, amani2021safe}.  
In other words, states and actions must belong to predefined safe sets, \( \mathcal{S}_\safe \times \mathcal{A}_\safe \).  
Instantaneous safety is a special case of the episode-wise constraint. Indeed, by defining \( u_h(s, a) = -\mathbb{I}\{(s, a) \notin \mathcal{S}_\safe \times \mathcal{A}_\safe\} \) and setting \( b = 0 \), an episode-wise safe algorithm safeties the instantaneous constraint for all \( h \) and \( k \).

\paragraph{Cancel Safety.}
Cancel safety is another common safety measure in CMDP literature \citet{wei2021provably,ghosh2022provably}.
It allows a strict constraint satisfaction in one episode to compensate for a violation in another. 
Formally, cancel safety ensures that the following cumulative \textbf{cancel violation regret} remains non-positive:
$$\violation_{\mathrm{cancel}}(K) \df \sum^K_{k=1} b - \vf{\pi^{(k)}, u}_{P, 1}(s_1)\;.$$
Note that the ``hard'' violation regret $\violation_{\mathrm{hard}}(K) \df \sum^K_{k=1} \max\brace*{b - \vf{\pi^{(k)}, u}_{P, 1}(s_1), 0}$ which considers violations in each individual episode \citep{ghosh2024towards, efroni2020exploration, muller2024truly}, always upper-bounds the cancel regret. This means cancel regret is a weaker measure.
Since episode-wise safety ensures $\violation_{\mathrm{hard}}=0$, our \MDPalgo always satisfies cancel safety, but cancel safety does not necessarily guarantee episode-wise safety.

\section{Useful Lemmas}\label{appendix:useful-lemma}

\begin{definition}
For a set of positive values $\brace*{a_n}_{n=1}^N$, we write $x=\polylog\left(a_1, \dots, a_N\right)$ if there exists an absolute constants $\brace*{b_n}_{n=0}^N >0$ and $\brace*{c_n}_{n=1}^N >0$ such that 
$
x \leq b_0 + b_1 \paren*{\ln a_1}^{c_1} + \dots + b_N \paren*{\ln a_N}^{c_N}
$.
\end{definition}

\begin{definition}[Distance metrics]\label{def:distance-metrics}
Let $\dist_{\infty}$ be the distance metric such that, for two functions $Q, Q': \SA \to \R$, $\dist_\infty(Q, Q') = \sup_{(s, a) \in \SA} \abs*{Q(s, a) - Q'(s, a)}$. 
Similarly, for two functions $V, V': \S \to \R$, $\dist_\infty(V, V') = \sup_{s \in \S} \abs*{V(s) - V'(s)}$. 
Finally, $\dist_1$ denotes the distance metric such that, for two functions $\pi, \pi': \S \to \sP(\A)$, $\dist_1(\pi, \pi') = \sup_{s \in \S} \norm*{\pi\paren{\cdot \given s} - \pi'\paren{\cdot \given s}}_1$. 
\end{definition}

\begin{definition}[$\varepsilon$-cover]
Let $\Theta=\left\{\btheta \in \R^d:\|\btheta\|_2 \leq R \right\}$ be a ball with radius $R$.
Fix an $\varepsilon$.
An $\varepsilon$-net $\cM_\varepsilon \subset \Theta$ is a finite set such that for any $\btheta \in \Theta$, there exists a $\btheta' \in \cM_\varepsilon$ such that $\dist\paren*{\btheta, \btheta'} \leq \varepsilon$ for some distance metric $\dist(\cdot, \cdot)$.
The smallest $\varepsilon$-net is called $\varepsilon$-cover and denoted as $\cN_\varepsilon$. The size of the $\varepsilon$-net is called the $\varepsilon$-covering number.
\end{definition}

\begin{lemma}[\textbf{Lemma 5.2} in \citet{vershynin2010introduction}]\label{lemma:ball-covering-bound}
The $\varepsilon$-covering number of the ball \\    
$\Theta=\left\{\theta \in \R^d:\|\theta\|_2 \leq R \right\}$ with the distance metric $\norm*{\cdot}_2$ is upper bounded by $(1+2 R / \varepsilon)^d$.
\end{lemma}

\begin{lemma}[Danskin's Theorem \citep{bertsekas1997nonlinear}]\label{lemma:danskin's theorem}
    \looseness=-1
    Let $f: \R^n \times \cZ \to \R$ be a continuous function where $\cZ \in \R^m$ is a compact set and $g(x) \df \max_{z\in \cZ} f(x, z)$.\\
    Let $\cZ_0(x) \df \brace*{\bar{z} \given f(x, \bar{z})=\max_{z\in \cZ}f(x, z)}$ be the maximizing points of $f(x, z)$.
    Assume that $f(x, z)$ is convex in $x$ for every $z\in \cZ$. 
    Then, $g(x)$ is convex.
    Furthermore, if $Z_0(x)$ consists of a single element $\bar{z}$, i.e., $\cZ_0(x)=\brace{\bar{z}}$, it holds that
    $\frac{\partial g(x)}{\partial x}=\frac{\partial f(x, \bar{z})}{\partial x}$.
\end{lemma}

\begin{lemma}[\textbf{Lemma D.4} in \citet{rosenberg2020near}]\label{lemma:exp-to-concrete}
Let $\left(X^{(k)}\right)_{k=1}^{\infty}$ be a sequence of random variables with expectation adapted to the filtration $\left(\mathcal{F}^{(k)}\right)_{k=0}^{\infty}$. 
Suppose that $0 \leq X^{(k)} \leq B$ almost surely. 
Then, with probability at least $1-\delta$, the following holds for all $k \geq 1$ simultaneously:    
$$
\sum_{i=1}^k \mathbb{E}\left[X^{(i)} \mid \mathcal{F}^{(i-1)}\right] \leq 2 \sum_{i=1}^k X^{(i)}+4 B \ln \frac{2 k}{\delta}
$$
\end{lemma}

\begin{lemma}[\textbf{Lemma 11} in \citet{abbasi2011improved}]\label{lemma:elliptical potential}
Let \(\left\{\bx^{(k)} \right\}_{k=1}^K\) be a sequence in $\R^d$.
Let \(\bLambda^{(k)}=\rho \bI+\sum_{i=1}^{k-1} \bx^{(i)} \paren*{\bx^{(i)}}^{\top}.
\)
If \(\norm*{\bx^{(k)}}_2 \leq B\) for all $k$, 
$$
\sum_{k=1}^K \min \left\{1,\left\|\bx^{(k)}\right\|_{\paren*{\bLambda^{(k)}}^{-1}}^2\right\} \leq 2 d \ln \left(\frac{\rho d+K B^2}{\rho d}\right)\;.
$$
Additionally, if \(\norm*{\bx^{(k)}}_2 \leq 1\) for all $k$ and $\rho \geq 1$\footnote{The second argument follows since $\left\|\bx\right\|_{\bLambda^{-1}}^2 \leq \sigma_{\max }\left(\bLambda^{-1}\right)\left\|\bx\right\|^2 \leq \rho^{-1} \leq 1$, where $\sigma_{\max}(\bLambda^{-1})$ denotes the maximum eigen value of $\bLambda^{-1}$.}, we have
$$
\sum_{k=1}^K \left\|\bx^{(k)}\right\|_{\paren*{\bLambda^{(k)}}^{-1}}^2 \leq 2 d \ln \left(\frac{\rho d+K}{\rho d}\right)\;.
$$
\end{lemma}

\begin{lemma}[\textbf{Theorem 2} in \citet{abbasi2011improved}]\label{lemma:confidence-ellipsoid}
Let $\left\{\cF^{(k)}\right\}_{k=0}^{\infty}$ be a filtration. 
Let $\left\{\varepsilon^{(k)}\right\}_{k=1}^{\infty}$ be a real-valued stochastic process such that $\varepsilon^{(k)}$ is $\cF^{(k)}$-measurable and $\varepsilon^{(k)}$ is conditionally $R$-sub-Gaussian for some $R \geq 0$. 
Let $\left\{\bphi^{(k)}\right\}_{k=1}^{\infty}$ be an $\R^d$-valued stochastic process such that $\bphi^{(k)}$ is $\cF^{(k-1)}$ measurable and $\|\bphi^{(k)}\|_2 \leq L$ for all $k$. 
For any $k \geq 0$, define $Y_k \df \btheta^\top \bphi^{(k)}+ \varepsilon_t$ for some $\btheta \in \R^d$ such that $\norm*{\btheta}_2 \leq B$,
$\bLambda^{(k)}\df\rho \bI+\sum_{i=1}^k \bphi^{(i)} \paren*{\bphi^{(i)}}^\top$,
and  \(\widehat{\btheta}^{(k)} \df \paren*{\bLambda^{(k)}}^{-1}\sum_{i =1}^k \bphi^{(i)} Y^{(i)} \).
Then for any $\delta>0$, with probability at least $1-\delta$, for all $k \geq 0$, we have
\begin{align*}
\norm*{\widehat{\btheta}^{(k)}-\btheta}_{\bLambda^{(k)}} \leq 
\rho^{1 / 2} B + 
R \sqrt{d \ln \left(\frac{1+k L^2 / \rho}{\delta}\right)}\;.
\end{align*}
\end{lemma}

\begin{lemma}[\textbf{Lemma D.4} in \citet{jin2020provably}]\label{lemma:fixed-V-bound}
Let $\left\{s^{(k)}\right\}_{k=1}^{\infty}$ be a stochastic process on state space $\mathcal{S}$ with corresponding filtration $\left\{\cF^{(k)}\right\}_{k=0}^{\infty}$.
Let $\left\{\bphi^{(k)}\right\}_{k=0}^{\infty}$ be an $\mathbb{R}^d$-valued stochastic process where $\bphi^{(k)}$ is $\cF^{(k-1)}$-measurable and $\left\|\bphi^{(k)}\right\| \leq 1$.
Let $\bLambda^{(k)}=\rho \bI+\sum_{k=1}^k \bphi^{(k)} \paren*{\bphi^{(k)}}^{\top}$ and let $\cV$ be a class of real-valued function over the state space $\S$ such that $\sup_s|V(s)| \leq B$ for a $B > 0$.
Let $\mathcal{N}^\cV_{\varepsilon}$ be the $\varepsilon$-cover of $\cV$ with respect to the distance $\dist_\infty$. 
Then for any $\delta>0$, with probability at least $1-\delta$, for all $K \geq 0$, and any $V \in \cV$, we have:    
$$
\left\|\sum_{k=1}^K \bphi^{(k)}\paren*{V\left(s^{(k)}\right)-\mathbb{E}\left[V\left(s^{(k)}\right) \mid \mathcal{F}^{(k-1)}\right]}\right\|_{\paren*{\bLambda^{(k)}}^{-1}}^2 \leq 4 B^2\paren*{\frac{d}{2} \ln \left(\frac{K+\rho}{\rho}\right)+\ln \frac{|\cN^\cV_{\varepsilon}|}{\delta}}+\frac{8 K^2 \varepsilon^2}{\rho}\;.
$$
\end{lemma}

\begin{lemma}[\textbf{Lemma A.1} in \citet{shalev2014understanding}]\label{lemma:alogx-inequality}
Let $a>0$. Then, $x \geq 2 a \ln (a)$ yields $x \geq a \ln (x)$. 
It follows that a necessary condition for the inequality $x\leq a \ln (x)$ to hold is that $x\leq 2 a \ln (a)$.    
\end{lemma}

\begin{lemma}\label{lemma:sum of sqrt}
For any positive real numbers $x_1, x_2, \dots, x_n$, 
$\sum_{i=1}^n \sqrt{x_i} \leq \sqrt{n}\sqrt{\sum^n_{i=1}x_i}$.
\end{lemma}
\begin{proof}
Due to the Cauchy-Schwarz inequality, we have
$
\left(\frac{\sum_{i=1}^n \sqrt{x_i}}{n}\right)^2 \leq \frac{\sum_{i=1}^n x_i}{n}
$.
Taking the square root of the inequality proves the claim.
\end{proof}

\begin{lemma}[\textbf{Lemma 1} in \citet{shani2020optimistic}]\label{lemma:extended value difference}
Let $\widetilde{\pi}, \pi$ be two policies, $P$ be a transition kernel, and $g$ be a reward function.
Let $\tvf{\pi}_h: \S \to \R$ be a function such that 
$$
\tvf{\pi}_h(s) = \sum_{a\in \A}\widetilde{\pi}_h\paren{a\given s} \widetilde{Q}_h(s, a)\;,$$ 
for all $h \in \bbrack{1, H}$ with some function $\widetilde{Q}_h: \HSA \to \R$.
Then, for any $(h, s) \in \bbrack{1, H}\times \S$
\begin{align*}
\tvf{\widetilde{\pi}}_h(s)-\vf{\pi, g}_{P, h}(s) 
= 
\vf{\pi, g^1}_{P, h}(s) 
+ \vf{\pi, g^2}_{P, h}(s)\;,
\end{align*}
where $g^1$ and $g^2$ are reward functions such that
\begin{align*}
g^1_h(s, a) = \sum_{a \in \A} \paren*{\tpi_h\paren*{a\mid s}-\pi_h\paren*{a \mid s}}\widetilde{Q}_h(s, a)\;
\text{ and }\;
g^2_h(s, a) = 
\widetilde{Q}_h\left(s, a \right)-g_h\left(s, a\right)-\paren*{P_h\tvf{\widetilde{\pi}}_{h+1}}(s, a)\;.
\end{align*}
\end{lemma}

\begin{lemma}[Regularized value difference lemma]\label{lemma:regularized value difference}
Let $\kappa \geq 0$ be a non-negative value, $\pi, \pi'$ be two policies, $P$ be a transition kernel, and $g$ be a reward function.
Let $\tvf{\tpi}_h[\kappa]: \S \to \R$ be a function such that 
$$
\tvf{\tpi}_h[\kappa](s) = \sum_{a \in \A}\tpi_h\paren{a\given s} \paren*{\widetilde{Q}_h(s, a) - \kappa \ln \tpi_h\paren{a\given s}}\;,$$ 
for all $h \in \bbrack{1, H}$ with some function $\widetilde{Q}_h: \HSA \to \R$.
Then, for any $(h, s) \in \bbrack{1, H}\times \S$
\begin{align*}
\tvf{\tpi}_h[\kappa](s)-\vf{\pi, g}_{P, h}[\kappa](s) 
= 
\vf{\pi, f^1}_{P, h}(s) + \vf{\pi, f^2}_{P, h}(s)\;,
\end{align*}
where $f^1$ and $f^2$ are reward functions such that
\begin{align*}
f^1_h(s, a) 
&= \sum_{a \in \A}\tpi_h\paren*{a\mid s}\paren*{\widetilde{Q}_h(s, a) - \kappa \ln \tpi_h\paren{a\given s}}
- \pi_h\paren*{a \mid s} \paren*{\widetilde{Q}_h(s, a) - \kappa \ln \pi_h\paren{a\given s}}\\
\text{ and }\;
f^2_h(s, a) &= 
\widetilde{Q}_h(s, a) 
-g_h\left(s, a\right)
-\paren*{P_h\tvf{\widetilde{\pi}}_{h+1}[\kappa]}(s, a)\;.
\end{align*}
\end{lemma}
\begin{proof}
Since 
$$
\tvf{\tpi}_h[\kappa](s) = \sum_{a \in \A}\tpi_h\paren{a\given s} 
\paren*{\widetilde{Q}_h(s, a) - \kappa \ln \tpi_h\paren{a\given s}}
\; \text{ and }\;
\vf{\pi, g}_{P, h}[\kappa](s) = 
\vf{\pi, g - \kappa \ln \pi}_{P, h}(s) \;,
$$ 
using \Cref{lemma:extended value difference}, we have
\begin{align*}
\tvf{\tpi}_1[\kappa](s_1)  -\vf{\pi, g}_{P, 1}[\kappa](s_1) 
= 
\vf{\pi, g^1}_{P, 1}(s_1) + \vf{\pi, g^2}_{P, 1}(s_1)\;,
\end{align*}
where $g^1$ and $g^2$ are reward functions such that
\begin{align*}
g^1_h(s, a) 
&= \sum_{a \in \A}\paren*{\tpi_h\paren*{a\mid s}-\pi_h\paren*{a \mid s}} \paren*{\widetilde{Q}_h(s, a) - \kappa \ln \tpi_h\paren{a\given s}}\\
&= \sum_{a \in \A}\tpi_h\paren*{a\mid s}\paren*{\widetilde{Q}_h(s, a) - \kappa \ln \tpi_h\paren{a\given s}}
- 
\pi_h\paren*{a \mid s} 
\paren*{\widetilde{Q}_h(s, a) - \kappa \ln \pi_h\paren{a\given s}}
\\
&\underbrace{+ \sum_{a \in \A}\pi_h\paren{a \given s}
\paren*{\kappa \ln \tpi_h\paren{a\given s} - \kappa \ln \pi_h\paren{a\given s}}}_{\text{(a)}}\\
\text{ and }\;&g^2_h(s, a) = 
\widetilde{Q}_h(s, a) 
-g_h\left(s, a\right)
-\paren*{P_h\tvf{\widetilde{\pi}}_{h+1}[\kappa]}(s, a)
\underbrace{- \kappa \ln \tpi_h\paren{a\given s}
+ \kappa \ln \pi_h\paren{a \given s}}_{(b)}
\;.
\end{align*}
The claim holds since the terms (a) and (b) are canceled out in
\(
\vf{\pi, g^1}_{P, h}(s) + \vf{\pi, g^2}_{P, h}(s)
\).
\end{proof}

\begin{lemma}\label{lemma:softmax-policy-bound}
\looseness=-1    
Let $Q, \widetilde{Q} : \A \to \R$ be two functions.
Let $\kappa > 0$ be a positive constant.
Define two softmax distributions $\pi, \widetilde{\pi} \in \sP(\A)$ such that
\(
\pi = \softmax\paren*{\frac{Q}{\kappa}} 
\) and 
\(
\widetilde{\pi} = \softmax\paren*{\frac{\widetilde{Q}}{\kappa}}
\).
Then, 
\(
\norm*{\pi - \widetilde{\pi}}_1 \leq 
\frac{8}{\kappa}
\norm*{Q -\widetilde{Q}}_\infty
\).
\end{lemma}
\begin{proof}
It holds that
\begin{align*}
\frac{1}{2}\norm*{\pi - \widetilde{\pi}}_{1}
&\numeq{\leq}{a} 2\sum_{a \in \A} \pi\paren*{a} \abs*{\ln \pi\paren*{a} - \ln \widetilde{\pi}\paren*{a}}
\leq 2\max_a \abs*{\ln \pi\paren*{a} - \ln \widetilde{\pi}\paren*{a}}\\
&= 2\max_a \abs*{\frac{1}{\kappa} Q\paren*{a} - \frac{1}{\kappa}\widetilde{Q}\paren*{a} - \ln \sum_{a}\exp\paren*{\frac{1}{\kappa}Q(a)} + \ln \sum_{a}\exp\paren*{\frac{1}{\kappa}\widetilde{Q}(a)}}\\
&\leq 
2\max_a \abs*{\frac{1}{\kappa} Q\paren*{a} - \frac{1}{\kappa}\widetilde{Q}\paren*{a}}+ 2\abs*{\ln \sum_{a}\exp\paren*{\frac{1}{\kappa}Q(a)} - \ln \sum_{a}\exp\paren*{\frac{1}{\kappa}\widetilde{Q}(a)}}\\
&\numeq{\leq}{b} 4\max_a \abs*{\frac{1}{\kappa} Q\paren*{a} - \frac{1}{\kappa}\widetilde{Q}\paren*{a}}\;,
\end{align*}
where (a) uses \textbf{Theorem 17} in \citet{sason2016f} and (b) uses the fact that $\ln \sum_i \exp(\bx_i) - \ln \sum_i \exp(\by_i) \leq \max_i \paren*{\bx_i - \by_i}$ (see, e.g., \textbf{Theorem 1} in \citet{dutta2024log}).
This concludes the proof.
\end{proof}

\section{Regret Analysis (Linear Constrained Bandit)}\label{sec:regret-analysis-bandit}

\begin{lemma}[Good event 1]\label{lemma:good-event1}
Suppose \Cref{algo:zero-vio-bandit} is run with $\rho = 1$.
Let $\delta \in (0, 1]$.
Define $\rmvEventBandit$ as the event where the following inequality holds: 
\begin{align*}
&\sum_{k =1}^K
\E_{\ba \sim \pi^{(k)}} \norm*{\ba}^2_{\paren*{\bLambda^{(k)}}^{-1}} 
\leq 2\sum_{k =1 }^K
\norm*{\ba^{(k)}}^2_{\paren*{\bLambda^{(k)}}^{-1}}
+ 4 \ln \frac{2K}{\delta} \;.
\end{align*}
Then, $\P(\rmvEventBandit) \geq 1 - \delta$.
\end{lemma}
\begin{proof}
\looseness=-1
The claim immediately follows from \Cref{lemma:exp-to-concrete} with $\norm*{\ba}_2\leq 1$ and $\rho=1$.
\end{proof}

\begin{lemma}[Good event 2]\label{lemma:good-event2}
Define $\confeventBandit$ as the event where the following two hold:
For any $\pi \in \Pi$, $k \in \bbrack{1, K}$,
$$
\abs*{\widehat{r}^{(k)}_\pi - r_\pi} \leq \cp \beta^{(k)}_\pi
\quad\text{ and }\quad
\abs*{\widehat{u}^{(k)}_\pi - u_\pi} \leq \cp \beta^{(k)}_\pi\;.
$$
Then, if \Cref{algo:zero-vio-bandit} is run with $\rho = 1$ and the value of $\cp \geq B + R\sqrt{d \ln \frac{4K}{\delta}}$, it holds that $\P(\confeventBandit) \geq 1 - \delta$.
\end{lemma}
\begin{proof}
\looseness=-1
Using \Cref{lemma:confidence-ellipsoid} with $\rho = 1$, with probability at least $1-\delta$, for any $k \in \bbrack{1, K}$ and for both $g \in \brace{r, u}$, we have
\begin{align*}
\abs*{\ba^\top \paren*{\hbtheta^{(k)}_g - \btheta_g}} 
&\leq
\norm*{\hbtheta^{(k)}_g - \btheta_g}_{\bLambda^{(k)}}
\norm*{\ba}_{\paren*{\bLambda^{(k)}}^{-1}}\\
&\numeq{\leq}{a}
\paren*{B + R\sqrt{d \ln \frac{2\paren*{1 + K}}{\delta}}} \norm*{\ba}_{\paren*{\bLambda^{(k)}}^{-1}}\\
&\leq
\paren*{B + R\sqrt{d \ln \frac{4K}{\delta}}} \norm*{\ba}_{\paren*{\bLambda^{(k)}}^{-1}}\;,
\end{align*}
where (a) uses \Cref{lemma:confidence-ellipsoid}.
The claim holds by \(
\abs*{\widehat{g}^{(k)}_{\pi} - g_{\pi}}
\leq \E_{\ba \sim \pi}\abs*{\ba^\top \paren*{\hbtheta^{(k)}_g - \btheta_g}}
\) for $g \in \brace{r, u}$. 
\end{proof}

\begin{lemma}[Cumulative bonus bound]\label{lemma:cumulative-bonus-bandit}
Suppose $\rmvEventBandit$ holds.
Then, $\sum_{k=1}^K\beta^{(k)}_{\pi^{(k)}} \leq \sqrt{K} \sqrt{2d\ln \paren*{1 + \frac{K}{d}} + 4 \ln \frac{2K}{\delta}}$.
\end{lemma}
\begin{proof}
It holds that
\begin{align*}
\sum_{k=1}^K\beta^{(k)}_{\pi^{(k)}}
&\numeq{\leq}{a} \sqrt{K 
\sum_{k=1}^K
\paren*{\E_{\ba \sim \pi^{(k)}} \norm*{\ba}_{\paren*{\bLambda^{(k)}}^{-1}}}^2
}
\numeq{\leq}{b} \sqrt{K 
\sum_{k=1}^K
\E_{\ba \sim \pi^{(k)}} \norm*{\ba}^2_{\paren*{\bLambda^{(k)}}^{-1}}
}\\
&\numeq{\leq}{c}  \sqrt{K}
\sqrt{2
\sum_{k=1}^K
\norm*{\ba^{(k)}}^2_{\paren*{\bLambda^{(k)}}^{-1}}
+ 4 \ln \frac{2K}{\delta}} 
\numeq{\leq}{d} \sqrt{K} \sqrt{2d\ln \paren*{1 + \frac{K}{d}} + 4 \ln \frac{2K}{\delta}}\;,
\end{align*}    
where (a) and (b) use Cauchy–Schwarz inequality, (c) is due to \(\rmvEventBandit\), and (d) uses \Cref{lemma:elliptical potential}.
\end{proof}

\begin{lemma}[Restatement of \Cref{lemma:bandit-opt-pes-main}]\label{lemma:bandit-opt-pes}
Suppose $\confeventBandit$ holds. Then, for any $\pi \in \Pi$ and $k \in \bbrack{1, K}$, 
$$
r_{\pi} + 2\cp \beta^{(k)}_{\pi} \geq 
\widehat{r}_{\pi}^{(k)} + \cp \beta^{(k)}_{\pi} \geq r_{\pi} \quad \text{ and }\quad
u_{\pi} \geq \widehat{u}_{\pi}^{(k)} - \cp \beta^{(k)}_{\pi}
\geq u_{\pi} - 2\cp \beta^{(k)}_{\pi}\;.
$$
\end{lemma}
\begin{proof}
We have
\begin{align*}
u_\pi
\geq 
\widehat{u}^{(k)}_\pi - \abs*{\widehat{u}^{(k)}_\pi - u_\pi} 
\geq \widehat{u}^{(k)}_\pi - \cp \beta^{(k)}_\pi
\geq \widehat{u}^{(k)}_\pi - \abs*{\widehat{u}^{(k)}_\pi - u_\pi} - \cp \beta^{(k)}_\pi
\geq u_\pi - 2\cp \beta^{(k)}_\pi\;.
\end{align*}
Similarly, 
\begin{align*}
r_\pi + 2\cp \beta^{(k)}_\pi
\geq 
\widehat{r}^{(k)}_\pi + \abs*{\widehat{r}^{(k)}_\pi - r_\pi} + \cp \beta^{(k)}_\pi
\geq \widehat{r}^{(k)}_\pi + \cp \beta^{(k)}_\pi
\geq \widehat{r}^{(k)}_\pi + \abs*{\widehat{r}^{(k)}_\pi - r_\pi}
\geq r_\pi\;.
\end{align*}
\end{proof}

\begin{lemma}[Restatement of \Cref{lemma:alpha-feasibility-main}]\label{lemma:alpha-feasibility}
Consider $k \in \unconfBandit^\complement$.
For any $\alpha \in \brack*{0,  \frac{\bslt - 2\cp \beta^{(k)}_{\pisafe}}{\bslt - 2\cp \beta^{(k)}_{\pisafe} + 2\cp \beta^{(k)}_{\pi^\star}}}$, a mixture policy \(\pi_\alpha \df (1-\alpha) \pisafe + \alpha \pi^\star\) satisfies
$u_{\pi_\alpha} - 2\cp \beta^{(k)}_{\pi_\alpha} \geq b$.
\end{lemma}
\begin{proof}
For any $k$ and $\alpha \in [0, 1]$, we have
\begin{align*}
u_{\pi_\alpha} - b - 2\cp \beta^{(k)}_{\pi_\alpha}
&= (1-\alpha) \underbrace{\paren*{u_{\pisafe} - b}}_{\geq \bslt} + \alpha \underbrace{\paren*{u_{\pi^\star} - b}}_{\geq 0} - 2\cp (1-\alpha) \beta^{(k)}_{\pisafe} - 2\cp \alpha \beta^{(k)}_{\pi^\star}\\
&\geq (1-\alpha) \paren*{\bslt - 2\cp \beta^{(k)}_{\pisafe}} - 2\alpha \cp \beta^{(k)}_{\pi^\star}.
\end{align*}
To make \((1-\alpha) \paren*{\bslt- 2\cp \beta^{(k)}_{\pisafe}} - 2\alpha \cp \beta^{(k)}_{\pi^\star} \geq 0\), a sufficient condition is
\begin{align}\label{eq:alpha-cond-temp}
   \alpha \leq  \frac{\bslt - 2\cp \beta^{(k)}_{\pisafe}}{\bslt - 2\cp \beta^{(k)}_{\pisafe} + 2\cp \beta^{(k)}_{\pi^\star}}\;,
\end{align}
where the right hand side is non-negative since $k \in \unconfBandit^\complement$.
This concludes the proof.
\end{proof}

\begin{lemma}[Restatement of \Cref{lemma:Ck-bound-main}]\label{lemma:Ck-bound}
Suppose \Cref{algo:zero-vio-bandit} is run with $\rho = 1$.
Assume the event $\rmvEventBandit$ holds.
Then, 
\(|\unconfBandit| \leq  32d\cp^2\bslt^{-2}\ln\paren*{2K\delta^{-1}}\).
\end{lemma}
\begin{proof}
\looseness=-1
We have 
\begin{align*}
\sum_{k=1}^K
\E_{\ba \sim \pi^{(k)}} \norm*{\ba}^2_{\paren*{\bLambda^{(k)}}^{-1}}
&\geq
\sum_{k \in \unconfBandit}
\E_{\ba \sim \pi^{(k)}} \norm*{\ba}^2_{\paren*{\bLambda^{(k)}}^{-1}}\\
&\numeq{\geq}{a}
\sum_{k \in \unconfBandit}
\underbrace{
\paren*{\E_{\ba \sim \pi^{(k)}} \norm*{\ba}_{\paren*{\bLambda^{(k)}}^{-1}}}^2
}_{=\paren*{\beta^{(k)}_{\pisafe}}^2 \text{ since $\pi^{(k)}=\pisafe$}}
\numeq{\geq}{b} 
|\unconfBandit| \frac{\bslt^2}{4\cp^2}\;,
\end{align*}
where (a) is due to Jensen's inequality,
and (b) is due to \Cref{def:unconf-set}.
Due to \(\rmvEventBandit\), we have
\begin{align*}
\sum_{k=1}^K
\E_{\ba \sim \pi^{(k)}} \norm*{\ba}^2_{\paren*{\bLambda^{(k)}}^{-1}}
\leq 
2\sum_{k =1}^K
\norm*{\ba^{(k)}}^2_{\paren*{\bLambda^{(k)}}^{-1}}
+ 4 \ln \frac{2K}{\delta}\;.
\end{align*}
Using \Cref{lemma:elliptical potential} and since $\norm*{\ba}_2\leq 1$ and $\rho = 1$, the first term is bounded by:
\(\leq 2d\ln\paren*{1 + \frac{K}{d}}\)\;.
Thus, 
\begin{align*}
\frac{\bslt^2}{4\cp^2} |\unconfBandit| 
\leq 2d\ln\underbrace{\paren*{1 + \frac{K}{d}}}_{\leq 2K} + 4 \ln \frac{2K}{\delta}
\leq 8d\ln\paren*{\frac{2K}{\delta}}
\;.
\end{align*}
The claim holds by rearranging the above inequality.
\end{proof}

\begin{lemma}[Restatement of \Cref{lemma:optimism-main}]\label{lemma:optimism}
For any $k \in \unconfBandit^\complement$, $\pi_{\alpha^{(k)}}$ satisfies
\(r_{\pi_{\alpha^{(k)}}} + \frac{2B\cp}{\bslt} \beta^{(k)}_{\pi_{\alpha^{(k)}}} \geq r_{\pi^{\star}}\).
\end{lemma}
\begin{proof}
Let $\alpha^{(k)} \df \frac{\bslt - 2\cp \beta^{(k)}_{\pisafe}}{\bslt - 2\cp \beta^{(k)}_{\pisafe} + 2\cp \beta^{(k)}_{\pi^\star}}$ and $C \df \frac{2B\cp}{\bslt}$.
Note that 
$\frac{\alpha^{(k)}}{1 - \alpha^{(k)}} = \frac{\bslt - 2\cp \beta^{(k)}_{\pisafe}}{2\cp \beta^{(k)}_{\pi^\star}}$.
We have,
\begin{align*}
r_{\pi_{\alpha^{(k)}}} + C \beta^{(k)}_{\pi_{\alpha^{(k)}}}
&= (1-\alpha^{(k)}) r_{\pisafe} + \alpha^{(k)} r_{\pi^{\star}} + C (1-\alpha^{(k)}) \beta^{(k)}_{\pisafe} + C \alpha^{(k)} \beta^{(k)}_{\pi^\star}\\
&\geq \alpha^{(k)} r_{\pi^{\star}} +  C\paren*{\paren*{1-\alpha^{(k)}}\beta^{(k)}_{\pisafe} + \alpha^{(k)} \beta^{(k)}_{\pi^\star}} \;.
\end{align*}
A sufficient condition to have
$
\alpha^{(k)} r_{\pi^{\star}} + 
C \paren*{\paren*{1-\alpha^{(k)}}\beta^{(k)}_{\pisafe} + \alpha^{(k)} \beta^{(k)}_{\pi^\star}}
\geq r_{\pi^\star}$ is, since $r_{\pi^\star} = \E_{\ba \sim \pi^\star}\brack*{\left\langle\btheta, \ba\right\rangle} \leq \norm*{\btheta}_2 \E_{\ba \sim \pi^\star}\norm*{\ba}_2 \leq B$, 
\begin{align*}
B 
&\leq C\paren*{\beta^{(k)}_{\pisafe} + \frac{\alpha^{(k)}}{1-\alpha^{(k)}} \beta^{(k)}_{\pi^\star}}\\
&= C\paren*{\beta^{(k)}_{\pisafe} + \frac{1}{2\cp}\bslt - \beta^{(k)}_{\pisafe}}
\leq \frac{C}{2\cp}\bslt\;.
\end{align*}
\looseness=-1
Therefore, when $C \geq \frac{2B \cp}{\bslt}$, we have
\(
r_{\pi_{\alpha^{(k)}}} + C \beta^{(k)}_{\pi_{\alpha^{(k)}}} \geq r_{\pi^{\star}}
\).
\end{proof}

\begin{theorem}[Restatement of \Cref{theorem:bandit-regret-main}]\label{theorem:formal-regret}
\looseness=-1
Suppose that \Cref{algo:zero-vio-bandit} is run with $\rho=1$,
\begin{align*}
\cp =B + R\sqrt{d \ln \frac{4K}{\delta}},
\; \text{ and }\;
\co = \cp \paren*{1 + \frac{2B}{\bslt}}.
\end{align*} 
Then, with probability at least $1-2\delta$, the following two hold simultaneously:
\begin{itemize}
    \item $\pi^{(k)} \in \Pisafe$ for any $k \in [K]$ 
    \item \( \regret (K) \leq 
64dB \cp^2\bslt^{-2}\ln\paren*{2K\delta^{-1}} +
4\co \sqrt{K} \sqrt{2d\ln \paren*{1 + \frac{K}{d}} + 4 \ln \frac{2K}{\delta}}
\)
\end{itemize}
\end{theorem}
\begin{proof}
Suppose the good events \(\rmvEventBandit \cap \confeventBandit\) hold. 
Recall that $\pi^{(k)}$ is either $\pisafe$ in $k \in \unconfBandit$ or the solution to \optpes in $k \in \unconfBandit^\complement$. 
Since \optpes is ensured to have feasible solutions by \Cref{lemma:alpha-feasibility} for $k \in \unconfBandit^\complement$,  the first claim follows immediately.

\looseness=-1
We will prove the second claim.
It holds that
\begin{align*}
\regret(K) 
&= \sum^K_{k=1} r_{\pi^\star} - r_{\pi^{(k)}}
= \underbrace{\sum_{k\in \unconfBandit} r_{\pi^\star} - r_{\pi^{(k)}}}_{\text{$\pi^{(k)}=\pisafe$}}
+ \underbrace{\sum_{k\notin \unconfBandit} r_{\pi^\star} - r_{\pi^{(k)}}}_{\text{$\pi^{(k)}$ is computed by \optpes}}\\
&\leq 2B |\unconfBandit| + \sum_{k\notin \unconfBandit} r_{\pi^\star} - r_{\pi^{(k)}}\\
&\numeq{\leq}{a} 64dB \cp^2\bslt^{-2}\ln\paren*{2K\delta^{-1}} 
+ \underbrace{\sum_{k\notin \unconfBandit} \paren*{r_{\pi^\star} - \widehat{r}^{(k)}_{\pi^{(k)}} - \co\beta^{(k)}_{\pi^{(k)}}}}_{\circled{1}}
+ \underbrace{\sum_{k\notin \unconfBandit} \paren*{\widehat{r}^{(k)}_{\pi^{(k)}} + \co\beta^{(k)}_{\pi^{(k)}} - r_{\pi^{(k)}}}}_{\circled{2}}\;,
\end{align*}
where (a) uses the bound of $|\unconfBandit|$ (\Cref{lemma:Ck-bound}). 
Using \Cref{lemma:bandit-opt-pes}, the term $\circled{2}$ is bounded by
\(
\circled{2} \leq 
\sum_{k\notin \unconfBandit} 3\co\beta^{(k)}_{\pi^{(k)}}
\).
On the other hand, $\circled{1}$ is bounded by
\begin{align*}
\circled{1} 
&\numeq{\leq}{a} \sum_{k\notin \unconfBandit}
r_{\pi_{\alpha^{(k)}}} + \frac{2B\cp}{\xi} \beta^{(k)}_{\pi_{\alpha^{(k)}}}
- \widehat{r}^{(k)}_{\pi^{(k)}} - \co\beta^{(k)}_{\pi^{(k)}} \\
&\numeq{\leq}{b} \sum_{k\notin \unconfBandit}
\widehat{r}^{(k)}_{\pi_{\alpha^{(k)}}} + \cp \beta^{(k)}_{\pi_{\alpha^{(k)}}} 
+ \frac{2B\cp}{\xi} \beta^{(k)}_{\pi_{\alpha^{(k)}}}
- \widehat{r}^{(k)}_{\pi^{(k)}} - \co\beta^{(k)}_{\pi^{(k)}} \numeq{\leq}{c} 0 \;,
\end{align*}
where (a) uses the optimism of mixture policy (\Cref{lemma:optimism}), (b) uses \Cref{lemma:bandit-opt-pes}, and (c) uses $\co = (1 + 2B\cp\xi^{-1})\cp$ and since $\pi_{\alpha^{(k)}}$ is a feasible solution to \optpes due to \Cref{lemma:alpha-feasibility}.

\looseness=-1
Finally, by combining all the results, we have
\begin{align*}
\regret(K) &\leq 
64 dB \cp^2\bslt^{-2}\ln\paren*{2K\delta^{-1}} 
+ 3\co \sum_{k\notin \unconfBandit} \beta^{(k)}_{\pi^{(k)}}\\
&\leq 
64 dB \cp^2\bslt^{-2}\ln\paren*{2K\delta^{-1}} +
3\co \sqrt{K} \sqrt{2d\ln \paren*{1 + \frac{K}{d}} + 4 \ln \frac{2K}{\delta}}
\end{align*}
where the second inequality uses \Cref{lemma:cumulative-bonus-bandit}.
Since the good event $\rmvEventBandit \cap \confeventBandit$ occurs with probability at least $1 - 2\delta$ due to \Cref{lemma:good-event1,lemma:good-event2}, the claim holds.
\end{proof}

\section{Regret Analysis (Linear CMDP)}\label{appendix:MDP-regret-analysis}
\subsection{Definitions and Useful Lemmas}

\begin{definition}[$\bmu$-estimator]\label{def:mu-estimation}
Let \(\be(s) \in \R^\S\) denote a one-hot vector such that only the element at \(s \in \S\) is \(1\) and otherwise \(0\).
In \Cref{algo:zero-vio-linear MDP}, for all $h$ and $k$, define $\bmu^{(k)}_h \in \R^{S\times d}$ and $\bepsilon_h^{(k)} \in \R^\S$ such that
\begin{align}
\bmu^{(k)}_h \df 
\sum^{k-1}_{i=1} \be\paren*{s^{(i)}_{h+1}} \bphi\paren*{s^{(i)}_h, a_h^{(i)}}^\top \paren*{\bLambda^{(k)}_h}^{-1}
\;\text{ and }\; \bepsilon_h^{(k)} \df \be\paren*{s^{(k)}_{h+1}} - P\paren*{\cdot \given s^{(k)}_h, a^{(k)}_h} \;.
\end{align}
We remark that 
$\paren*{\hP^{(k)}_h V}(s, a) = \bphi(s, a)^\top \paren*{\bmu_h^{(k)}}^\top V$ for any $V \in \R^\S$.
\end{definition}

\begin{lemma}\label{lemma:mu-diff}
For all $k$ and $h$, it holds that:
\begin{align*}
\bmu_h^{(k)}-\bmu_h=-\rho \bmu_h\left(\bLambda_h^{(k)}\right)^{-1}+\sum_{i=1}^{k-1} \bepsilon_h^{(i)} \bphi\left(s_h^{(i)}, a_h^{(i)}\right)^{\top}\left(\bLambda_h^{(k)}\right)^{-1}
\end{align*}
\end{lemma}
\begin{proof}
Due to the definition of $\bmu^{(k)}_h$, we have
\begin{align*}
\bmu_h^{(k)} & 
=\sum^{k-1}_{i=1} \be\paren*{s^{(i)}_{h+1}} \bphi\paren*{s^{(i)}_h, a_h^{(i)}}^\top \paren*{\bLambda^{(k)}_h}^{-1}
=\sum^{k-1}_{i=1} \paren*{P\paren*{\cdot \given s^{(k)}_h, a^{(k)}_h} +  \bepsilon_h^{(k)}} 
\bphi\paren*{s^{(i)}_h, a_h^{(i)}}^\top \paren*{\bLambda^{(k)}_h}^{-1}\\
&=\sum^{k-1}_{i=1} \paren*{\bmu_h \bphi\paren*{s^{(k)}_h, a^{(k)}_h} +  \bepsilon_h^{(k)}} \bphi\paren*{s^{(i)}_h, a_h^{(i)}}^\top \paren*{\bLambda^{(k)}_h}^{-1}\\
&=\sum^{k-1}_{i=1} \bmu_h \bphi\paren*{s^{(k)}_h, a^{(k)}_h} \bphi\paren*{s^{(i)}_h, a_h^{(i)}}^\top \paren*{\bLambda^{(k)}_h}^{-1} + \sum^{k-1}_{i=1} \bepsilon_h^{(k)} \bphi\paren*{s^{(i)}_h, a_h^{(i)}}^\top \paren*{\bLambda^{(k)}_h}^{-1} \\
&=\bmu_h \paren*{\bLambda^{(k)}_h - \rho \bI} \paren*{\bLambda^{(k)}_h}^{-1} + \sum^{k-1}_{i=1} \bepsilon_h^{(k)} \bphi\paren*{s^{(i)}_h, a_h^{(i)}}^\top \paren*{\bLambda^{(k)}_h}^{-1} \\
&=\bmu_h - \rho \bmu_h \paren*{\bLambda^{(k)}_h}^{-1} + \sum^{k-1}_{i=1} \bepsilon_h^{(k)} \bphi\paren*{s^{(i)}_h, a_h^{(i)}}^\top \paren*{\bLambda^{(k)}_h}^{-1} \;.
\end{align*}
\end{proof}

\begin{lemma}\label{lemma:hP-P-V bound}
Let $\cV$ be a class of real-valued function over the state space $\S$ such that $\sup_s|V(s)| \leq B$ for a $B > 0$.
Let $\mathcal{N}_{\varepsilon}$ be the $\varepsilon$-cover of $\cV$ with respect to the distance $\dist_\infty$. 
In \Cref{algo:zero-vio-linear MDP}, for all $k, h, s, a$, for any $V \in \cV$, with probability at least $1-\delta$, we have
\begin{align*}
&\abs*{\paren*{\paren*{\hP_h^{(k)}-P_h}V}(s, a)} \\
\leq 
&\norm*{\bphi(s, a)}_{\paren*{\bLambda_h^{(k)}}^{-1} }
\paren*{
\sqrt{d\rho}B
+ 
2 B\sqrt{\frac{d}{2} \ln \paren*{\frac{k + \rho}{\rho}}}
+2B \sqrt{\ln \frac{|\cN_{\varepsilon}|}{\delta}}
+\frac{4 k \varepsilon}{\sqrt{\rho}}
}\;.
\end{align*}
\end{lemma}
\begin{proof}
Using \Cref{lemma:fixed-V-bound} and due to the definition of $\bLambda^{(k)}$ in \Cref{algo:zero-vio-linear MDP}, with probability at least \(1-\delta\), for all \(k, h\), we have
\begin{align*}
\left\|\sum_{i=1}^{k-1} \bphi\left(s_h^{(k)}, a_h^{(k)}\right)\left(V^{\top} \bepsilon_h^{(i)}\right)\right\|_{\left(\bLambda^{(k)}\right)^{-1}} 
&\leq 
\sqrt{4 B^2\paren*{\frac{d}{2} \ln \paren*{\frac{k + \rho}{\rho}}+\ln \frac{|\cN_{\varepsilon}|}{\delta}}+\frac{8 k^2 \varepsilon^2}{\rho}}\\
&\leq 
2 B\sqrt{\frac{d}{2} \ln \paren*{\frac{k + \rho}{\rho}}}
+2B \sqrt{\ln \frac{|\cN_{\varepsilon}|}{\delta}}
+\frac{4 k \varepsilon}{\sqrt{\rho}}\;,
\end{align*}
where the second inequality uses $\sqrt{a + b} \leq \sqrt{a} + \sqrt{b}$.
By inserting this to \Cref{def:mu-estimation}, we have
\begin{align*}
&\abs*{\paren*{\paren*{\hP_h^{(k)}-P_h}V}(s, a)}\\
=&\abs*{\bphi(s, a)^\top \paren*{\bmu_h^{(k)} - \bmu_h}^\top V}\\
=&\abs*{\bphi(s, a)^\top 
\paren*{
-\rho \bmu_h\left(\bLambda_h^{(k)}\right)^{-1}+\sum_{i=1}^{k-1} \bepsilon_h^{(i)} \bphi\left(s_h^{(i)}, a_h^{(i)}\right)^{\top}\left(\bLambda_h^{(k)}\right)^{-1} }^\top  V}\\
\leq &
\rho \abs*{
\bphi(s, a)^\top \paren*{\bLambda_h^{(k)}}^{-1} \paren*{\bmu_h}^\top V}
+\abs*{
\bphi(s, a)^\top \paren*{\bLambda_h^{(k)}}^{-1} 
\sum_{i=1}^{k-1} 
\bphi\left(s_h^{(i)}, a_h^{(i)}\right)
\paren*{V^\top \bepsilon_h^{(i)}}}\\
\leq &
\rho \norm*{\bphi(s, a)}_{\paren*{\bLambda_h^{(k)}}^{-1}}
\underbrace{\norm*{\paren*{\bmu_h}^\top V}_{\paren*{\bLambda_h^{(k)}}^{-1}}}_{\leq B \sqrt{d / \rho}\; \text{ by \Cref{assumption:linear mdp}}}
+
\norm*{\bphi(s, a)}_{\paren*{\bLambda_h^{(k)}}^{-1} }
\norm*{
\sum_{i=1}^{k-1} 
\bphi\left(s_h^{(i)}, a_h^{(i)}\right)
\paren*{\bepsilon_h^{(i)}}^{\top}V}_{\paren*{\bLambda_h^{(k)}}^{-1} }\\
\leq &
\norm*{\bphi(s, a)}_{\paren*{\bLambda_h^{(k)}}^{-1} }
\paren*{
\sqrt{d\rho}B
+ 
2 B\sqrt{\frac{d}{2} \ln \paren*{\frac{k + \rho}{\rho}}}
+2B \sqrt{\ln \frac{|\cN_{\varepsilon}|}{\delta}}
+\frac{4 k \varepsilon}{\sqrt{\rho}}
}\;.
\end{align*}
\end{proof}

\subsection{Function Classes and Covering Argument }

\begin{definition}[$Q$ function class]\label{def:abs-Q-function}
For any $h$ and for a pair of $(\bw, \bLambda)$, where $\bw \in \mathbb{R}^d$ and $\bLambda \in \R^{d\times d}$, define 
$\qf{(\bw, \bLambda), r}_{h} : \SA \to \R$, 
$\qf{(\bw, \bLambda), u}_{h}: \SA \to \R$,
and $\qf{(\bw, \bLambda), \dagger}_{h} : \SA \to \R$ such that
\begin{align*}
\qf{(\bw, \bLambda), r}_{h}(s, a) &=
r_h(s, a) + \clip\brace*{\co \norm*{\bphi(s, a)}_{\bLambda^{-1}} + \bw^{\top} \bphi(s, a),\; 0,\; H_\kappa-h_\kappa}\\
\qf{(\bw, \bLambda), u}_{h} (s, a) &= u_h(s, a) + \clip\brace*{-\cp \norm*{\bphi(s, a)}_{\bLambda^{-1}} + \bw^{\top} \bphi(s, a),\; 0,\; H-h}\\
\qf{(\bw, \bLambda), \dagger}_{h} (s, a) &= 
B_\dagger \norm*{\bphi(s, a)}_{\bLambda^{-1}} + \clip\brace*{C_\dagger \norm*{\bphi(s, a)}_{\bLambda^{-1}} + \bw^{\top} \bphi(s, a),\; 0,\; B_\dagger(H-h)}\;,
\end{align*}
where $\kappa, \co, \cp, B_\dagger, C_\dagger \geq 0$.
We denoted \(h_\kappa \df h(1 + \kappa \ln A)\) for $h \in \bbrack{1, H}$.
Let $\cQr_h$, $\cQu_h$, $\cQc_h$ denote function classes such that
\begin{align*}
\cQr_h&\df \brace*{\qf{(\bw, \bLambda), r}_{h}\given \|\bw\|_2 \leq KH_\kappa,\; \sigma_{\min }(\bLambda) \geq 1}\;,\\
\cQu_h&\df \brace*{\qf{(\bw, \bLambda), u}_h\given \|\bw\|_2 \leq KH,\; \sigma_{\min }(\bLambda) \geq 1}\;,\\
\text{ and }\; \cQc_h&\df \brace*{\qf{(\bw, \bLambda), \dagger}_h\given \|\bw\|_2 \leq KHB_\dagger ,\; \sigma_{\min }(\bLambda) \geq 1}\;.
\end{align*}
We let $\cN^{\cQr_h}_{\varepsilon}$, 
$\cN^{\cQu_h}_{\varepsilon}$, 
and $\cN^{\cQc_h}_{\varepsilon}$,
be the $\varepsilon$-covers of $\cQr_h$, $\cQu_h$, and $\cQc_h$ with the distance metric $\dist_\infty$.
\end{definition}

\begin{lemma}[$Q$ covers]\label{lemma:Q eps covers}
\looseness=-1
When \Cref{algo:zero-vio-linear MDP} is run with $\rho =1$, it hold that:
\begin{itemize}
\item[$(\mathrm{i})$] 
For all $k, h$ and for any $\pi \in \Pi$, 
\(\oqf{\pi, r}_{(k), h}[\kappa] \in \cQr_h\),
\(\pqf{\pi, u}_{(k), h} \in \cQu_h \),
and \(\oqf{\pi, \dagger}_{(k), h} \in \cQc_h \) 
\item[$(\mathrm{ii})$]
\(\ln |\cN^{\cQr_h}_{\varepsilon}| \leq d \ln \paren*{1+\frac{4 KH_\kappa }{\varepsilon}}+d^2 \ln \paren*{1+\frac{8 \sqrt{d} \co^2}{\varepsilon^2}} 
= \cO\paren*{d^2}\polylog\paren*{d,K,H_\kappa,\co,\varepsilon^{-1}}
\),\\
\(\ln |\cN^{\cQu_h}_{\varepsilon}| 
\leq d \ln \paren*{1+\frac{4 KH}{\varepsilon}}+d^2 \ln \paren*{1+\frac{8 \sqrt{d} \cp^2}{\varepsilon^2}}= \cO\paren*{d^2}\polylog\paren*{d,K,H,\cp,\varepsilon^{-1}}
\), \\
and \(\ln |\cN^{\cQc_h}_{\varepsilon}| 
\leq d \ln \paren*{1+\frac{4 KB_\dagger H}{\varepsilon}}+d^2 \ln \paren*{1+\frac{8 \sqrt{d} C_\dagger^2}{\varepsilon^2}}= \cO\paren*{d^2}\polylog\paren*{d, K, H, B_\dagger, C_\dagger,\varepsilon^{-1}}
\)
\end{itemize}
\end{lemma}
\begin{proof}
\looseness=-1
The statements in ($\mathrm{ii}$) immediately follow from the proof of \textbf{Lemma D.6} in \citet{jin2020provably}. 

\looseness=-1
We prove the first claim ($\mathrm{i}$).
For $\oqf{\pi, r}_{(k), h}$, we have
\begin{align*}
\oqf{\pi, r}_{(k), h}[\kappa]
=
&r_h + \clip\brace*{\co\beta^{(k)} + \hP^{(k)}\ovf{\pi, r}_{(k), h+1}[\kappa],\; 0,\; (H-h)(1+\kappa \ln A)} \\
= &
r_h + \clip\brace*{\co \sqrt{\bphi(s, a)^{\top} \paren*{\bLambda^{(k)}_h}^{-1} \bphi(s, a)}
+ \bphi(s, a)^\top 
\paren*{\bmu^{(k)}_h}^\top \ovf{\pi, r}_{(k), h+1}[\kappa],\; 0, \; (H-h)(1 + \kappa \ln A)}\;.
\end{align*}
According to the definition of $Q^{(\bw, \bLambda), r}_h$ (\Cref{def:abs-Q-function}), the claim immediately holds by showing the L2 bound of \( \left(\bmu_h^{(k)}\right)^{\top} \ovf{\pi, r}_{(k), h+1}[\kappa]\).
For any $h \in \bbrack{1, H}$ and $k \in \bbrack{1, K}$, we have
\begin{align*}
\norm*{\paren*{\bmu^{(k)}_h}^\top \ovf{\pi, r}_{(k), h+1}[\kappa]}_2
= & \norm*{
\sum^{k-1}_{i=1} 
\ovf{\pi, r}_{(k), h+1}[\kappa] \paren*{s^{(i)}_{h+1}} \bphi\paren*{s^{(i)}_h, a_h^{(i)}}^\top \paren*{\bLambda^{(k)}_h}^{-1}
}_2\\
\numeq{\leq}{a} & H_\kappa\norm*{
\paren*{\bLambda^{(k)}_h}^{-1}
\sum^{k-1}_{i=1} 
\bphi\paren*{s^{(i)}_h, a_h^{(i)}}
}_2
\leq KH_\kappa\;.
\end{align*}
where (a) uses $\norm*{\bphi}_2 \leq 1$ with $\rho=1$ and $0 \leq \ovf{\pi, r}_{(k), h+1}[\kappa] \leq H_\kappa$.

\looseness=-1
The remaining claims for \(\pqf{\pi, u}_{(k), h}(s, a) \in \cQu_h \) and \(\oqf{\pi, \dagger}_{(k), h}(s, a) \in \cQc_h \) can be similarly proven.
\end{proof}

\begin{definition}[Composite $Q$ function class]\label{def:composite-Q-function}
For each $h$, let $\Qcomp_h$ denote a function class such that
\begin{align*}
\Qcomp_h\df \brace*{
Q^{\dagger} + Q^{r} + \lambda Q^{u}\given 
Q^{\dagger} \in \cQc_h,\;  Q^{r} \in \cQr_h, \; Q^{u} \in \cQu_h,\; \text{ and } \; \lambda \in [0, C_\lambda]}\;.
\end{align*}
where $C_\lambda > 0$.
We let ${\cN}^{\Qcomp_h}_{\varepsilon}$ be the $\varepsilon$-cover of $\Qcomp_h$ with the distance metric $\dist_\infty$.
\end{definition}

\begin{lemma}[Composite Q cover]\label{lemma:composite-Q-cover}
When \Cref{algo:zero-vio-linear MDP} is run with $\rho =1$, the following statements hold:
\begin{itemize}
    \item[$(\mathrm{i})$] For all $(k, h)$, for any $\pi \in \Pi$, and for any $\lambda \in [0, C_\lambda]$,  
\(
\oqf{\pi, \dagger}_{(k), h} + \oqf{\pi, r}_{(k), h}[\kappa] + \lambda \pqf{\pi, u}_{(k), h} \in \Qcomp_h
\)
    \item[$(\mathrm{ii})$]
    \(
\ln \abs*{\cN^{\Qcomp_h}_{\varepsilon}} = 
\cO\paren*{d^2} 
\polylog\paren*{d, K, H_\kappa, \co, \cp, B_\dagger, C_\dagger,C_\lambda,\varepsilon^{-1}}
\)
\end{itemize}
\end{lemma}
\begin{proof}
\looseness=-1
The claim $(\mathrm{i})$ clearly holds by \Cref{lemma:Q eps covers} and \Cref{def:composite-Q-function}.

\looseness=-1
We prove the second claim $(\mathrm{ii})$.
Let $\cN^\lambda_\varepsilon$ be the $\varepsilon$-cover of a set $\brace*{\lambda \given \lambda \in [0, C_\lambda]}$ with the distance metric $\norm*{\cdot}_2$.
Let $\varepsilon_\dagger, \varepsilon_{r}, \varepsilon_{u}, \varepsilon_{\lambda} > 0$ be positive scalars.
Consider 
$\widetilde{Q}^\dagger\in \cN^{\cQc_h}_{\varepsilon_\dagger}$,
$\widetilde{Q}^r\in \cN^{\cQr_h}_{\varepsilon_r}$,
$\widetilde{Q}^u \in \cN^{\cQu_h}_{\varepsilon_u}$, 
and $\widetilde{\lambda} \in \cN^\lambda_{\varepsilon_\lambda}$.
For any $Q^\dagger \in \cQc_h$, $Q^r \in \cQr_h$, $Q^u \in \cQu_h$, and $\lambda \in [0, C_\lambda]$, we have
\begin{align*}
\begin{aligned}
& \dist_\infty\paren*{
Q^\dagger + Q^r+\lambda Q^u, 
\widetilde{Q}^\dagger + \widetilde{Q}^r +\widetilde{\lambda} \widetilde{Q}^u}\\
&\leq
\underbrace{\sup _{s, a} \abs*{Q^\dagger(s, a) - \widetilde{Q}^\dagger(s, a)}}_{\leq \varepsilon_\dagger}
+ \underbrace{\sup _{s, a} \abs*{Q^r(s, a) - \widetilde{Q}^r(s, a)}}_{\leq \varepsilon_r}\\
&\quad + \underbrace{\lambda\sup _{s, a}\abs*{\paren*{Q^u(s, a) - \widetilde{Q}^u(s, a)}}}_{\leq C_\lambda \varepsilon_u} 
+ \underbrace{\sup _{s, a}
\abs*{
\paren*{\lambda - \widetilde{\lambda}}Q^u(s, a)}}_{\varepsilon_\lambda H}\\
&\numeq{\leq}{a}
\varepsilon_\dagger +  \varepsilon_r + C_\lambda \varepsilon_u + \varepsilon_\lambda H\;,
\end{aligned}    
\end{align*}
where 
(a) appropriately chooses $\widetilde{Q}^\dagger, \widetilde{Q}^r, \widetilde{Q}^u, \widetilde{\lambda}$.
By replacing 
$\varepsilon_\dagger$ with $\varepsilon / 4$, 
$\varepsilon_r$ with $\varepsilon / 4$, 
$\varepsilon_u$ with $1 / 4 C_\lambda$, and 
$\varepsilon_\lambda$ with $\varepsilon / 4 H$, the above inequality is upper bounded by $\varepsilon$.
Thus, 
\begin{align*}
\ln \abs*{\cN^{\Qcomp_h}_{\varepsilon}} 
\leq &
\ln \abs*{\cN^{\lambda}_{\varepsilon / 4 H}} 
+
\ln \abs*{\cN^{\cQu_h}_{\varepsilon / 4 C_\lambda}} 
+ \ln \abs*{\cN^{\cQr_h}_{\varepsilon / 4}}
+ \ln \abs*{\cN^{\cQc_h}_{\varepsilon / 4}} \\
\leq &
\cO\paren*{d^2} 
\polylog\paren*{d, K, H_\kappa, \co, \cp, B_\dagger, C_\dagger,C_\lambda,\varepsilon^{-1}}
\;.
\end{align*}
where the second inequality uses \Cref{lemma:ball-covering-bound} and \Cref{lemma:Q eps covers}.
\end{proof}

\begin{definition}[Policy class]\label{def:policy-class}
$\tiPi \df \widetilde{\Pi}_1 \times \dots \times \widetilde{\Pi}_H$ denotes a softmax policy class such that 
\begin{align*}
\tiPi_h \df
\brace*{
\pi_Q \in \Pi
\given 
Q \in \Qcomp_h
}
\;\text{ where }\;
\pi_Q\paren{\cdot \given s}=\softmax\paren*{\frac{1}{\kappa}
Q(s, \cdot)
} \; \forall s \in \S\;
\;,
\end{align*}
where $\kappa > 0$.
We let ${\cN}^{\tiPi_h}_{\varepsilon}$ be the $\varepsilon$-cover of $\tiPi_h$ with the distance metric $\dist_1$.
\end{definition}

\begin{lemma}[$\pi^{(k), \lambda}$ cover]\label{lemma:policy-cover}
When \Cref{algo:zero-vio-linear MDP} is run with $\rho =1$ and $\kappa > 0$, for all $h$, the following statements hold:
\begin{itemize}
\item[$(\mathrm{i})$] For all $(k, h)$ and $\lambda \in [0, C_\lambda]$ in \Cref{algo:zero-vio-linear MDP}, 
\(\pi^{(k), \lambda}_h \in \tiPi_h\)
\item[$(\mathrm{ii})$]
$
\ln \abs*{\cN^{\tiPi_h}_\varepsilon}=
\cO\paren*{d^2} 
\polylog\paren*{d, K, H_\kappa, \co, \cp, B_\dagger, C_\dagger,C_\lambda,\varepsilon^{-1}, \kappa^{-1}}
$
\end{itemize}
\end{lemma}
\begin{proof}
The claim $(\mathrm{i})$ immediately follows from \Cref{lemma:composite-Q-cover} and \Cref{def:composite-softmax-policy}.

\looseness=-1
We prove the second claim.
For a $Q: \SA \to \R$, let $\pi_Q$ be a softmax policy such that $\pi_Q\paren{\cdot \given s} = \softmax\paren*{\frac
{Q(s, \cdot)}{\kappa}}$. 
Consider $\widetilde{Q}$ from $\cN^{\Qcomp_h}_{\varepsilon}$.
Then, for any $Q \in \Qcomp_h$, we have
\begin{align*}
\dist_1\paren*{\pi_{Q}, \pi_{\widetilde{Q}}}
\numeq{\leq}{a} 
\frac{8}{\kappa}
\dist_\infty\paren*{Q, \widetilde{Q}}
\numeq{\leq}{b} \frac{8\varepsilon}{\kappa}\;,
\end{align*}
where (a) uses \Cref{lemma:softmax-policy-bound} and (b) appropriately chooses $\widetilde{Q}$ from $\cN^{\Qcomp_h}_\varepsilon$.
Therefore, 
\begin{align*}
\ln \abs*{\cN^{\tiPi_h}_\varepsilon}
\leq 
\ln \abs*{\cN^{\Qcomp_h}_{\kappa\varepsilon / 8}}
\leq
\cO\paren*{d^2} 
\polylog\paren*{d, K, H_\kappa, \co, \cp, B_\dagger, C_\dagger,C_\lambda,\varepsilon^{-1}, \kappa^{-1}}
\end{align*}
where the second inequality uses \Cref{lemma:composite-Q-cover}.
\end{proof}

\begin{definition}[$V$ function class]\label{def:V-cover}
Let $\cVr_h$, $\cVu_h$, and $\cVc_h$ denote value function classes such that
\begin{align*}
&\cVr_h\df \brace*{V^\pi_Q[\kappa]: \S \to \R \given 
\pi \in \tiPi_h \cup \brace{\pisafe_h}\;
\text{ and }\; Q \in \cQr_h
}\;,\\
&\cVu_h\df \brace*{V^\pi_Q[0]: \S \to \R \given 
\pi \in \tiPi_h \cup \brace{\pisafe_h}\;
\text{ and }\; Q \in \cQu_h
}\;,\\
\text{ and }\;
&\cVc_h\df \brace*{V^\pi_Q[0]: \S \to \R \given 
\pi \in \tiPi_h \cup \brace{\pisafe_h}\;
\text{ and }\; Q \in \cQc_h
}\;,\\
\text{ where }\;
&V^\pi_Q[\kappa](s) \df \sum_{a \in \A} \pi\paren{a \given s}\paren*{Q(s, a) - \kappa \ln \pi\paren{a \given s}}\; \forall s \in \S \;.
\end{align*}
We let $\cN^{\cVr_h}_{\varepsilon}$, $\cN^{\cVu_h}_{\varepsilon}$, and $\cN^{\cVc_h}_{\varepsilon}$ be the $\varepsilon$-covers of $\cVr_h$, $\cVu_h$, and $\cVc_h$ with the distance metric $\dist_\infty$.
\end{definition}

\begin{lemma}[$V$ covers]\label{lemma:V-cover}
\looseness=-1
When \Cref{algo:zero-vio-linear MDP} is run with $\rho =1$ and $\kappa > 0$, for all $h$, the following statements hold:
\begin{itemize}
\item[$(\mathrm{i})$] 
For all $(k, h)$, for any $\lambda \in [0, C_\lambda]$, and for both $\pi=\pi^{(k),\lambda}$ and $\pi=\pisafe$, we have:\\
\(\ovf{\pi, r}_{(k), h}[\kappa] \in \cVr_h\),
\(\pvf{\pi, u}_{(k), h} \in \cVu_h\),
and \(\ovf{\pi, \dagger}_{(k), h} \in \cVc_h\)  
\item[$(\mathrm{ii})$]
\(\ln \abs*{\cN^{\cVr_h}_{\varepsilon}} =
\cO\paren*{d^2} \polylog\paren*{d, K, H_\kappa, \co, \cp, B_\dagger, C_\dagger,C_\lambda,\varepsilon^{-1}, \kappa^{-1}}
\),\\
\(\ln \abs*{\cN^{\cVu_h}_{\varepsilon}} =
\cO\paren*{d^2} \polylog\paren*{d, K, H_\kappa, \co, \cp, B_\dagger, C_\dagger,C_\lambda,\varepsilon^{-1}, \kappa^{-1}}
\),\\
and \(\ln \abs*{\cN^{\cVc_h}_{\varepsilon}} =
\cO\paren*{d^2} \polylog\paren*{d, K, H_\kappa, \co, \cp, B_\dagger, C_\dagger,C_\lambda,\varepsilon^{-1}, \kappa^{-1}}
\)
\end{itemize}
\end{lemma}
\begin{proof}
The condition $(\mathrm{i})$ immediately follow from \Cref{lemma:Q eps covers} and \Cref{lemma:policy-cover} with \Cref{def:V-cover} and \Cref{def:composite-softmax-policy}.

\looseness=-1
We prove the second claim $(\mathrm{ii})$.
Let $Q \in \cQr_h$ and $\widetilde{Q} \in \cN^{\cQr_h}_{\varepsilon^r}$ where $\varepsilon^r > 0$.
For any two $\pi, \tpi : \S \to \sP(\A)$, for any $s$, we have
\begin{align*}
&\abs*{
\sum_{a \in \A}\pi\paren{a \given s} \paren*{Q(s, a) -\kappa \ln \pi\paren{a \given s}}
- \sum_{a \in \A}\widetilde{\pi}\paren{a \given s} \paren*{\widetilde{Q}(s, a) - \kappa \ln \pi \paren{a \given s}}
}\\
\leq &
\abs*{
\sum_{a \in \A}\pi\paren{a \given s} Q(s, a) 
- \sum_{a \in \A}\pi\paren{a \given s} \widetilde{Q}(s, a)    
+ \sum_{a \in \A}\pi\paren{a \given s} \widetilde{Q}(s, a)    
- \sum_{a \in \A}\widetilde{\pi}\paren{a \given s} \widetilde{Q}(s, a)    
}\\
&+ \kappa \underbrace{\abs*{\sum_{a \in \A}\pi\paren{a \given s}\ln \pi\paren{a \given s} - \tpi\paren{a \given s}\ln \tpi\paren{a \given s}}}_{\fd \sH(\pi) - \sH(\tpi)}
\\
\leq &
\sum_{a \in \A}\pi\paren{a \given s} \underbrace{\abs*{Q(s, a) - \widetilde{Q}(s, a)}}_{\leq \varepsilon^r}
+ 
\norm*{\pi\paren{\cdot \given s} - \widetilde{\pi}\paren{\cdot \given s}}_1 \underbrace{\norm*{\widetilde{Q}(\cdot, s)}_\infty}_{\leq H_\kappa} + \kappa \paren*{\sH(\pi) - \sH(\tpi)}
\\
\leq & \varepsilon^r + 
H_\kappa \norm*{\pi\paren{\cdot \given s} - \widetilde{\pi}\paren{\cdot \given s}}_1  
+ \kappa \paren*{\sH(\pi) - \sH(\tpi)}
\end{align*}
where the second inequality chooses appropriate $\widetilde{Q}$.
We defined entropies of $\pi$ and $\tpi$ as $\sH(\pi) \df \sum_{a \in \A}\pi\paren{a \given s}\ln \pi\paren{a \given s}$ and $\sH(\tpi) \df \sum_{a \in \A}\tpi\paren{a \given s}\ln \tpi\paren{a \given s}$, respectively.

\looseness=-1
The remaining task is to bound $H_\kappa \norm*{\pi\paren{\cdot \given s} - \widetilde{\pi}\paren{\cdot \given s}}_1 + \kappa \paren*{\sH(\pi) - \sH(\tpi)}$.
When $\pi = \pisafe$, choosing $\tpi=\pisafe$ trivially bounds this term by $0$.
Thus, we only consider the case when $\pi \in \tiPi_h$, i.e., $\pi\paren{\cdot \given s} = \softmax\paren*{\frac{1}{\kappa}Q^\circ(s, \cdot)}$ with $Q^\circ \in \Qcomp_h$.
We also consider $\tpi\paren{\cdot \given s} = \softmax\paren*{\frac{1}{\kappa}\widetilde{Q}^\circ(s, \cdot)}$ with $\widetilde{Q}^\circ \in \cN^{\Qcomp_h}_{\varepsilon^\circ}$, where $\varepsilon^\circ > 0$.
For the entropy gap, we have
\begin{align*}
&\sH(\pi) - \sH(\tpi)\\
=&\abs*{\sum_{a \in \A}\pi\paren{a \given s}\ln \pi\paren{a \given s} - \tpi\paren{a \given s}\ln \tpi\paren{a \given s}} \\
=&\abs*{\sum_{a \in \A}\paren*{\pi\paren{a \given s} - \tpi\paren{a \given s}}\ln \pi\paren{a \given s} + \sum_{a \in \A}\tpi\paren{a \given s}\paren*{\ln \pi\paren{a \given s} - \ln \tpi\paren{a \given s}}} \\
\leq&
\norm*{\pi\paren{\cdot \given s} - \tpi\paren{\cdot \given s}}_1 \max_a \ln \pi\paren{a \given s}
+ 
\max_a \abs*{\ln \pi\paren{a \given s} - \ln \tpi\paren{a \given s}}\\
\numeq{\leq}{a}&
\underbrace{\norm*{\pi\paren{\cdot \given s} - \tpi\paren{\cdot \given s}}_1}_{\leq \frac{8}{\kappa} \max_{a} \abs*{Q^\circ(s, a) - \widetilde{Q}^\circ(s, a)} \;\text{ by \Cref{lemma:softmax-policy-bound}}} \max_a \ln \pi\paren{a \given s}
+ 
\frac{2}{\kappa}\underbrace{\max_a \abs*{Q^\circ\paren{s, a} - \widetilde{Q}^\circ\paren{s, a}}}_{\leq \varepsilon^\circ}\\
\numeq{\leq}{b}& 
\frac{\varepsilon^\circ}{\kappa}
\paren*{
8\max_a \ln \pi\paren{a \given s}
+ 2 
}\;,
\end{align*}
where (a) utilizes a decomposition similar to the proof of \Cref{lemma:softmax-policy-bound}, and (b) chooses an appropriate $\widetilde{Q}^\circ$.
Finally, $\ln \pi\paren{a \given s}$ can be bounded as
\begin{align*}
\max_a \ln \pi\paren{a \given s}   
= 
\max_a \frac{1}{\kappa}Q^\circ(s, a) - \ln \sum_{a'} \exp\paren*{\frac{1}{\kappa}Q^\circ(s, a')}
\leq \frac{B_\dagger H + H_\kappa + C_\lambda H}{\kappa}\;,
\end{align*}
where the last inequality is due to \Cref{def:composite-Q-function}.

\looseness=-1
Therefore, we have
\begin{align*}
&H_\kappa 
\norm*{\pi\paren{\cdot \given s} - \widetilde{\pi}\paren{\cdot \given s}}_1+ \kappa \paren*{\sH(\pi) - \sH(\tpi)} \leq  \varepsilon^\circ \underbrace{\paren*{
2 + \frac{8}{\kappa}\paren*{B_\dagger H + 2H_\kappa + C_\lambda H}}}_{\fd Z}\;.
\end{align*}

\looseness=-1
Finally, by setting $\varepsilon^r = \varepsilon / 2H_\kappa$ and $\varepsilon^\circ = \varepsilon / 2Z$, $\ln \abs*{\cN^{\cVr_h}_{\varepsilon}}$ is bounded as:
\begin{align*}
\ln \abs*{\cN^{\cVr_h}_{\varepsilon}}
\leq 
\ln \paren*{\abs*{\cN^{\Qcomp_h}_{\varepsilon / 2 Z}} + 1}
+ \ln |\cN^{\cQr_h}_{\varepsilon / 2 H_\kappa}|
=
\cO\paren*{d^2} 
\polylog\paren*{d, K, H_\kappa, \co, \cp, B_\dagger, C_\dagger,C_\lambda,\varepsilon^{-1}, \kappa^{-1}}
\;,
\end{align*}
where the second inequality is due to \Cref{lemma:Q eps covers} and \Cref{lemma:composite-Q-cover}.
The claims for $\ln \abs*{\cN^{\cVu_h}_\varepsilon}$ and $\ln \abs*{\cN^{\cVc_h}_\varepsilon}$ can be similarly proven. 
\end{proof}

\subsection{Good Events and Value Confidence Bounds for \Cref{lemma:opt-pes-MDP-main} Proof}

\begin{lemma}[Good event 1]\label{lemma:good-event1-MDP}
Define $\rmvEevent$ as the event where the following inequality holds: 
\begin{align*}
&\sum_{k=1}^K \sum_{h=1}^H
\E\brack*{\norm*{\bphi(s_h^{(k)}, a_h^{(k)})}^2_{\paren*{\bLambda_h^{(k)}}^{-1}}\given s^{(k)}_h, a^{(k)}_h \sim \pi^{(k)}_h}\\
\leq&
2 \sum_{k=1}^K \sum_{h=1}^H
\norm*{\bphi(s_h^{(k)}, a_h^{(k)})}^2_{\paren*{\bLambda_h^{(k)}}^{-1}}
+ 4H \ln \frac{2KH}{\delta} \;.
\end{align*}
If \Cref{algo:zero-vio-linear MDP} is run with $\rho = 1$, $\P(\rmvEevent) \geq 1 - \delta$.
\end{lemma}
\begin{proof}
\looseness=-1
The claim immediately follows from \Cref{lemma:exp-to-concrete} with $\norm*{\phi}_2\leq 1$ and $\rho=1$.
\end{proof}

\begin{lemma}[Good event 2]\label{lemma:good-event2-MDP}
Define $\confevent$ as the event where the following condition holds: \\
For all $k, h$ and for any $V^r \in \cVr_{h+1}$, $V^u \in \cVu_{h+1}$, and $V^\dagger \in \cVc_{h+1}$
\begin{align*}
&\abs*{\paren*{\paren*{\hP_h^{(k)}-P_h}
V^r}(s, a)} 
\leq 
\co\beta^{(k)}_h(s, a)\quad \forall (h, s, a) \in \HSA\\
&\abs*{\paren*{\paren*{\hP_h^{(k)}-P_h}
V^u}(s, a)} 
\leq 
\cp\beta^{(k)}_h(s, a)\quad \forall (h, s, a) \in \HSA\\
\text{and } &\abs*{\paren*{\paren*{\hP_h^{(k)}-P_h}
V^\dagger}(s, a)} 
\leq 
C_\dagger\beta^{(k)}_h(s, a)\quad \forall (h, s, a) \in \HSA\;.
\end{align*}
\looseness=-1
If \Cref{algo:zero-vio-linear MDP} is run with $\rho =1$, $\co = \tiO(dH_\kappa)$, $\cp=\tiO(dH)$, and $C_\dagger=\tiO(dHB_\dagger)$, we have $\P(\confevent) \geq 1 - 2\delta$.
\end{lemma}
\begin{proof}
Using \Cref{lemma:hP-P-V bound} with $\cN^{\cVr_{h+1}}_{1 / K}$, with probability at least $1-\delta$, for any $(k, h, s, a)$, 
\begin{align*}
&\abs*{\paren*{\paren*{\hP_h^{(k)}-P_h}
V^r}(s, a)} \\
\numeq{\leq}{a}
&\norm*{\bphi(s, a)}_{\paren*{\bLambda_h^{(k)}}^{-1} }
\paren*{
\sqrt{d}H_\kappa
+ 
2 H_\kappa\sqrt{\frac{d}{2} \ln \paren*{2K}}
+2H_\kappa \sqrt{\ln \frac{\abs*{\cN^{\cVr_{h+1}}_{1/K}}}{\delta}}
+4
}\\
\numeq{\leq}{b}
&\norm*{\bphi(s, a)}_{\paren*{\bLambda_h^{(k)}}^{-1} }
\tiO \paren*{dH_\kappa}\ln \co
\numeq{\leq}{c}
\norm*{\bphi(s, a)}_{\paren*{\bLambda_h^{(k)}}^{-1} }\co
\end{align*}
where (a) sets $\varepsilon=1/K$ to $\cN^{\cVu_h}_{\varepsilon}$ and uses \cref{lemma:hP-P-V bound}, (b) uses \Cref{lemma:Q eps covers}, and (c) set sufficiently large $\co=\tiO\paren*{dH_\kappa}$ and uses \cref{lemma:alogx-inequality}.
The claim for $\cVu_{h+1}$ and $\cVc_{h+1}$ can be similarly proven.
\end{proof}

\begin{lemma}[Remove clipping one-side]\label{lemma:no-clip}
Under $\confevent$, for any $(k, h, s, a)$, and for any $\lambda \in [0, C_\lambda]$, for both $\pi = \pi^{(k), \lambda}$ and $\pi=\pisafe$, we have 
\begin{align*}
&\co\beta^{(k)}_h(s, a) + \paren*{\hP^{(k)}_h \ovf{\pi, r}_{(k), h+1}[\kappa]}(s, a) \geq \paren*{P_h \ovf{\pi, r}_{(k), h+1}[\kappa]}(s, a) \geq 0\;,\\
&-\cp \beta^{(k)}_h(s, a) + \paren*{\hP^{(k)}_h \pvf{\pi, u}_{(k), h+1}}(s, a) \leq  \paren*{P_h \pvf{\pi, u}_{(k), h+1}}(s, a) \leq H-h\;,\\
\text{ and }\;
&C_\dagger\beta^{(k)}_h(s, a) + \paren*{\hP^{(k)}_h \ovf{\pi, \dagger}_{(k), h+1}}(s, a) \geq \paren*{P_h \ovf{\pi, \dagger}_{(k), h+1}}(s, a) \geq 0\\
\end{align*}
\end{lemma}
\begin{proof}
We have
\begin{align*}
&\co\beta^{(k)}_h(s, a) + \paren*{\hP^{(k)}_h \ovf{\pi, r}_{(k), h+1}[\kappa]}(s, a)\\
\numeq{\geq}{a} &\abs*{\paren*{P_h - \hP^{(k)}_h} \ovf{\pi, r}_{(k), h+1}[\kappa]}(s, a) + \paren*{\hP^{(k)}_h \ovf{\pi, r}_{(k), h+1}[\kappa]}(s, a)\\
\geq &\paren*{P_h - \hP^{(k)}_h} \ovf{\pi, r}_{(k), h+1}[\kappa](s, a) + \paren*{\hP^{(k)}_h \ovf{\pi, r}_{(k), h+1}[\kappa]}(s, a)\\
= &
P_h \ovf{\pi, r}_{(k), h+1}[\kappa](s, a) \numeq{\geq}{b} 0\;,
\end{align*}   
where (a) is due to $\confevent$ with \Cref{lemma:V-cover} and (b) is due to $r \geq 0$ and by the definition of $\ovf{\pi, r}_{(k), h+1}[\kappa]$.
The claim for $\ovf{\pi, \dagger}_{(k), h+1}$ can be similarly proven.

\looseness=-1
For $\pvf{\pi, u}_{(k), h+1}$, we have
\begin{align*}
&-\cp \beta^{(k)}_h(s, a) + \paren*{\hP^{(k)}_h \pvf{\pi, u}_{(k), h+1}}(s, a)\\
\numeq{\leq}{a} &-\abs*{\paren*{\hP^{(k)}_h - P_h} \pvf{\pi, u}_{(k), h+1}}(s, a) + \paren*{\hP^{(k)}_h \pvf{\pi, u}_{(k), h+1}}(s, a)\\
\leq &-\paren*{\hP^{(k)}_h - P_h} \pvf{\pi, u}_{(k), h+1}(s, a) + \paren*{\hP^{(k)}_h \pvf{\pi, u}_{(k), h+1}}(s, a)\\
= &
P_h \pvf{\pi, u}_{(k), h+1}(s, a) \numeq{\leq}{b} H-h\;,
\end{align*}   
where (a) is due to $\confevent$ with \Cref{lemma:V-cover} and (b) is due to $u \leq \bone$ and by the definition of $\pvf{\pi, u}_{(k), h+1}$.
\end{proof}

\begin{definition}[$Q$ estimation gap]\label{def:delta-Q}
For any $h, k$ and $\pi \in \Pi$, define 
$\delta^{\pi, r}_{(k), h},
\delta^{\pi, u}_{(k), h}, \delta^{\pi, \dagger}_{(k), h}:\SA \to \R$
be functions such that:
\begin{align*}
&\delta^{\pi, r}_{(k), h} =\clip\brace*{\co \beta^{(k)}_h + \paren*{\hP^{(k)}_h \ovf{\pi, r}_{(k), h+1}[\kappa]}, 0, H_\kappa-h_\kappa}
- \paren*{P_h \ovf{\pi, r}_{(k), h+1}[\kappa]}\;, \\
&\delta^{\pi, u}_{(k), h} =\paren*{P_h \pvf{\pi, u}_{(k), h+1}} - \clip\brace*{-\cp\beta^{(k)}_h + \paren*{\hP^{(k)}_h \pvf{\pi, u}_{(k), h+1}}, 0, H-h}\;,\\
\text{ and }\; 
&\delta^{\pi, \dagger}_{(k), h} =\clip\brace*{C_\dagger \beta^{(k)}_h + \paren*{\hP^{(k)}_h \ovf{\pi, \dagger}_{(k), h+1}}, 0, B_\dagger (H-h)}
- \paren*{P_h \ovf{\pi, \dagger}_{(k), h+1}}\;, \\
\end{align*}
It is clear that these functions satisfy, for any $(\pi, k, h)$,  %
\begin{equation}\label{eq:Q-Q-diff}
\begin{aligned}
\oqf{\pi, r}_{(k), h}[\kappa] = \qf{\pi, r + \delta^{\pi, r}_{(k)}}_{P, h}[\kappa],\quad
\pqf{\pi, u}_{(k), 1} = \qf{\pi, u - \delta^{\pi, u}_{(k)}}_{P, h},\;
\text{ and }\;
\oqf{\pi, \dagger}_{(k), h} = \qf{\pi, B_\dagger \beta^{(k)} + \delta^{\pi, \dagger}_{(k)}}_{P, h}\;.
\end{aligned}
\end{equation}
Additionally, let $\Deltar$, $\Deltau$, and $\Deltac$ be function classes such that:
\begin{align*}
&\Deltar \df  
\brace*{\delta: \HSA\to \R \given 
\bzero \leq \delta_h \leq \min\brace*{2\co \beta^{(k)}_h, H_\kappa-h_\kappa} \; \forall h \in \bbrack{1, H}}\\
&\Deltau \df 
\brace*{\delta: \HSA\to \R \given 
\bzero \leq \delta_h \leq \min \brace*{2\cp \beta^{(k)}_h, H-h} \; \forall h \in \bbrack{1, H}}\\
\text{ and }\;&\Deltac \df 
\brace*{\delta: \HSA\to \R \given 
\bzero \leq \delta_h \leq \min \brace*{2C_\dagger \beta^{(k)}_h, B_\dagger(H-h)} \; \forall h \in \bbrack{1, H}}\;.
\end{align*}
\end{definition}

\begin{lemma}\label{lemma:delta-bound}
Under $\confevent$, for any $k$ and for any $\lambda \in [0, C_\lambda]$, for both $\pi = \pi^{(k), \lambda}$ and $\pi=\pisafe$, it holds that
\(\delta^{\pi, r}_{(k), \cdot} \in \Deltar\),
\(\delta^{\pi, u}_{(k), \cdot} \in \Deltau\),
and  \(\delta^{\pi, \dagger}_{(k), \cdot} \in \Deltac\).
\end{lemma}
\begin{proof}
\looseness=-1
$\delta^{\pi, u}_{(k), h}(s, a) \leq H-h$ clearly holds.
Additionally, we have
\begin{align*}
\delta^{\pi, u}_{(k), h}(s, a) 
\numeq{=}{a}&\paren*{P_h \pvf{\pi, u}_{(k), h+1}}(s, a) - \max\brace*{-\cp \beta^{(k)}_h(s, a) + \paren*{\hP^{(k)}_h \pvf{\pi, u}_{(k), h+1}}(s, a), 0}\\
\leq &\paren*{P_h \pvf{\pi, u}_{(k), h+1}}(s, a) +\cp\beta^{(k)}_h(s, a) - \paren*{\hP^{(k)}_h \pvf{\pi, u}_{(k), h+1}}(s, a) \\
\leq &\cp \beta^{(k)}_h(s, a) + \abs*{\paren*{P_h - \hP^{(k)}_h} \pvf{\pi, u}_{(k), h+1}}(s, a) \numeq{\leq}{b} 2\cp\beta^{(k)}_h(s, a)\;,
\end{align*}
where (a) is due to \Cref{lemma:no-clip} and (b) is due to $\confevent$.
Finally, note that 
\begin{align*}
\delta^{\pi, u}_{(k), h}(s, a)
=&\paren*{P_h \pvf{\pi, u}_{(k), h+1}}(s, a) - \max\brace*{-\cp \beta^{(k)}_h(s, a) + \paren*{\hP^{(k)}_h \pvf{\pi, u}_{(k), h+1}}(s, a), 0}\\
\geq&
\underbrace{\cp \beta^{(k)}_h(s, a) + \paren*{P_h \pvf{\pi, u}_{(k), h+1}}(s, a) - \paren*{\hP^{(k)}_h \pvf{\pi, u}_{(k), h+1}}(s, a)}_{\geq 0 \;\text{ by }\; \confevent}
\geq 0\;.
\end{align*}
This concludes the proof for $\delta^{\pi, u}_{(k), h}$.
The claims for $\delta^{\pi, r}_{(k), h}$ and $\delta^{\pi, \dagger}_{(k), h}$ can be similarly proven.
\end{proof}

\begin{lemma}[Restatement of \Cref{lemma:opt-pes-MDP-main}]\label{lemma:opt-pes-MDP}
Suppose $\confevent$ holds. 
For any $k$ and for any $\lambda \in [0, C_\lambda]$, for both $\pi = \pi^{(k), \lambda}$ and $\pi=\pisafe$, we have 
\begin{align*}
&\vf{\pi,r}_{P, h} \leq \ovf{\pi,r}_{(k), h} \leq  \vf{\pi,r+ 2\co \beta^{(k)}}_{P, h},
&&\qf{\pi,r}_{P, h} \leq \oqf{\pi,r}_{(k), h} \leq  \qf{\pi,r+ 2\co \beta^{(k)}}_{P, h}  \\
&\vf{\pi,B_\dagger\beta^{(k)}}_{P, h} \leq \ovf{\pi,\dagger}_{(k), h} \leq  \vf{\pi,B_\dagger\beta^{(k)} + 2C_\dagger \beta^{(k)}}_{P, h},
&&\qf{\pi,B_\dagger\beta^{(k)}}_{P, h} \leq \oqf{\pi,\dagger}_{(k), h} \leq  \qf{\pi,B_\dagger\beta^{(k)} + 2C_\dagger \beta^{(k)}}_{P, h},\\
&\vf{\pi,u-2\cp \beta^{(k)}}_{P, h} \leq \pvf{\pi,u}_{(k), h} \leq \vf{\pi,u}_{P, h},
&&\qf{\pi,u-2\cp \beta^{(k)}}_{P, h} \leq \pqf{\pi,u}_{(k), h} \leq \qf{\pi,u}_{P, h}\;.
\end{align*}
\end{lemma}
\begin{proof}
The inequalities for $Q$ functions directly hold by \Cref{eq:Q-Q-diff} and \Cref{lemma:delta-bound}.

\looseness=-1
For the utility $V$ function,
\begin{align*}
\pvf{\pi^{(k)},u}_{(k), h}(s) - \vf{\pi^{(k)},u}_{P, h}(s)
&= 
\sum_{a \in \A} \pi_h\paren{a \given s}
\paren*{
\pqf{\pi,u}_{(k), h}(s, a) - \qf{\pi,u}_{P, h}(s, a)
}\\
&\numeq{=}{a} 
\sum_{a \in \A} \pi_h\paren{a \given s} \qf{\pi, -\delta^{\pi, u}_{(k)}}_{P, h}(s)\numeq{\leq}{b} 0\;,
\end{align*}
where (a) uses \Cref{eq:Q-Q-diff} and (b) uses \Cref{lemma:delta-bound}.
Similarly,
\begin{align*}
\pvf{\pi,u}_{(k), h}(s) - \vf{\pi,u - 2\cp \beta^{(k)}}_{P, h}(s)
&= 
\sum_{a \in \A} \pi_h\paren{a \given s}
\paren*{
\pqf{\pi,u}_{(k), h}(s, a) - \qf{\pi,u-2\cp \beta^{(k)}}_{P, h}(s, a)
}\\
&\numeq{=}{a} 
\sum_{a \in \A} \pi_h\paren{a \given s} \qf{\pi, -\delta^{\pi, u}_{(k)} + 2\cp \beta^{(k)}}_{P, h}(s)\numeq{\geq}{b} 0\;,
\end{align*}
where (a) uses \Cref{eq:Q-Q-diff} and (b) uses \Cref{lemma:delta-bound}.
The claims for $r$ and $\dagger$ can be similarly proven.
\end{proof}

\subsection{Proofs for Zero-Violation Guarantee (\Cref{subsec:MDP-zero-vio})}

\subsubsection{Proof of \Cref{lemma:trigger-condition-main} and \Cref{lemma:softmax-value-monotonicity-main}}\label{subsec:safe-softmax-policy-exists-proof}

\begin{lemma}[Restatement of \Cref{lemma:softmax-value-monotonicity-main}]\label{lemma:softmax-value-monotonicity}
Let $f, g: \HSA \to \R$ be functions and let $\kappa > 0$. 
Given $\lambda \geq 0$, let $\pi^\lambda$ be a softmax policy such that 
\begin{align*}
\pi^\lambda_h \paren*{\cdot \given s} = \softmax\paren*{\frac{1}{\kappa}\paren*{
    \qf{\pi, f}_{P, h}[\kappa](s, \cdot)
    + \lambda \qf{\pi, g}_{P, h}(s, \cdot)}}\;.
\end{align*}
Then, \(\vf{\pi^\lambda, g}_{P, 1}(s_1)\) is monotonically increasing in $\lambda$.
\end{lemma}
\begin{proof}
Let $\cW \df \brace*{\occ{\pi}_{P, \cdot}: \HSA \to [0, 1] \given \pi \in \Pi}$ be the set of all the occupancy measures.
Let $\sL: \R\times \cW \to \R$ be a function such that:
\begin{align*}
\sL(\lambda, w) = &\sum_{h, s, a \in \HSA} w_h(s, a)\paren*{f_h(s, a) + \lambda g_h(s, a)}
- \kappa w_h(s, a)\ln \frac{w_h(s, a)}{\sum_{a' \in \A}w_h(s, a')}\;.
\end{align*}

\looseness=-1
We first show that $\sL$ is strictly concave in $\cW$.
Let 
$$\sH: w \in \cW \mapsto \sum_{h, s, a \in \HSA} - w_h(s, a)\ln \frac{w_h(s, a)}{\sum_{a' \in \A}w_h(s, a')}$$ 
be the function representing the second term of $\sL$.
Then, \footnote{This proof is based of \textbf{Lemma 14} from \citet{ding2023last}}
\begin{align*}
& \sH\left(\alpha w^1+(1-\alpha) w^2\right) \\
& =-\sum_{h, s, a}\left(\alpha w_h^1(s, a)+(1-\alpha) w_h^2(s, a)\right) \log \frac{\alpha w_h^1(s, a)+(1-\alpha) w_h^2(s, a)}{\alpha \sum_{a^{\prime}} w_h^1\left(s, a^{\prime}\right)+(1-\alpha) \sum_{a^{\prime}} w_h^2\left(s, a^{\prime}\right)} \\
& \numeq{\geq}{a}-\sum_{h, s, a} \alpha w_h^1(s, a) \log \frac{\alpha w_h^1(s, a)}{\alpha \sum_{a^{\prime}} w_h^1\left(s, a^{\prime}\right)}-\sum_{h, s, a}(1-\alpha) w_h^2(s, a) \log \frac{(1-\alpha) w_h^2(s, a)}{(1-\alpha) \sum_{a^{\prime}} w_h^2\left(s, a^{\prime}\right)} \\
& =\alpha \sH\left(w_h^1\right)+(1-\alpha) \sH\left(w_h^2\right)\;,
\end{align*}
for any $w^1, w^2 \in \cW$ and $\alpha \in [0, 1]$, where (a) is due to the log sum inequality 
$\left(\sum_i \bx_i\right) \ln \frac{\sum_i \bx_i}{\sum_i \by_i} \leq \sum_i \bx_i \ln \frac{\bx_i}{\by_i}$
for non-negative $\bx_i$ and $\by_i$.
Since (a) takes equality if and only if $w^1 = w^2$, $\sH$ is strictly concave.
Consequently, 
\(
\sL(\lambda, w) = \sum_{h, s, a \in \HSA} w_h(s, a)\paren*{f_h(s, a) + \lambda g_h(s, a)}
- \kappa \sH(w)
\) is also strictly concave in $\cW$.

\looseness=-1
Let \(w^\lambda = \argmax_{w \in \cW} \sL(\lambda, w)\), which is a unique maximizer due to the strict concavity. 
Define $\sL(\lambda) \df \max_{w \in \cW} \sL(\lambda, w)$. 
Using Danskin's theorem (\Cref{lemma:danskin's theorem}), 
$\sL(\lambda)$ is convex and $\frac{\partial\sL (\lambda)}{\partial \lambda} = \sum_{h, s, a \in \HSA} w^\lambda_h(s, a)g_h(s, a)$.
Since $\sL(\lambda)$ is convex, its derivative is non-decreasing. Therefore,
\begin{align}\label{eq:L-lambda-increasing}
\frac{\partial^2\sL (\lambda)}{\partial \lambda^2} =   
\frac{\partial}{\partial \lambda}
\sum_{h, s, a \in \HSA} w^\lambda_h(s, a)g_h(s, a) \geq 0\;.
\end{align}

\looseness=-1
Since $\pi^\lambda$ is the softmax policy, combined with the one-to-one mapping between occupancy measure and policy \citep{puterman1990markov}, the well-known analytical solution of regularized MDP \citep{geist2019theory} indicates that 
$w^\lambda$ corresponds to the occupancy measure of $\pi^\lambda$. 
Thus, due to \Cref{eq:L-lambda-increasing}, it holds that
$$0 \leq 
\frac{\partial}{\partial \lambda}
\sum_{h, s, a \in \HSA} w^\lambda_h(s, a)g_h(s, a)
=
\frac{\partial}{\partial \lambda}\vf{\pi^\lambda, g}_{P, 1}(s_1)\;.$$
This concludes the proof.
\end{proof}

\begin{definition}[Softmax policy with fixed $\bdelta$]\label{def:fixed-delta-softmax}
For any $k \in \unconfMDP^\complement$, $\bdelta \df (\delta^r, \delta^u, \delta^\dagger) \in \Deltar \times \Deltau \times \Deltac$ and $\lambda \geq 0$, let $\pi^{\bdelta, \lambda}\in \Pi$ be a policy such that 
\begin{align*}
\pi^{\bdelta, \lambda}_h \paren*{\cdot \given s} = \softmax\paren*{\frac{1}{\kappa}\paren*{
    \qf{\pi^{\bdelta, \lambda}, B_\dagger \beta^{(k)} + \delta^\dagger}_{P, h}(s, \cdot)
    + \qf{\pi^{\bdelta, \lambda}, r + \delta^r}_{P, h}[\kappa](s, \cdot) + \lambda \qf{\pi^{\bdelta, \lambda}, u - \delta^u}_{P, h}(s, \cdot)}}\;.
\end{align*}
\end{definition}

\begin{lemma}[Existence of feasible $\lambda$]\label{lemma:feasible-lambda-exists}
Suppose $\kappa \leq \frac{\bslt^2}{32H_\kappa^2 (B_\dagger + 1)}$.
For any $k$ and for any $\bdelta \in \Deltac\times \Deltar \times \Deltau$, there exists a $\lambda^{\bdelta} \in \brack*{0, \frac{8H_\kappa^2(B_\dagger + 1)}{\bslt}}$ such that, 
\(
\vf{\pi^{\bdelta, \lambda}, u - \delta^u}_{P, 1}(s_1) \geq b
\) holds for any $\lambda \geq \lambda^{\bdelta}$.
\end{lemma}
\begin{proof}
\looseness=-1
Throughout the proof, we use a shorthand $r^\bdelta \df B_\dagger \beta^{(k)} + \delta^\dagger + r + \delta^r$.
Consider the following entropy-regularized max-min optimization problem:
\begin{align}
&\max_{\pi \in \Pi}\min_{\lambda \geq 0} \vf{\pi, r^\delta}_{P, 1}[\kappa](s_1)  + \lambda \paren*{\vf{\pi, u - \delta^u}_{P, 1}(s_1) - b - \frac{\bslt}{4}} + \frac{\kappa}{2} \lambda^2\nonumber\\
=&\min_{\lambda \geq 0}\max_{\pi \in \Pi} \vf{\pi, r^\delta}_{P, 1}[\kappa](s_1)  + \lambda \paren*{\vf{\pi, u - \delta^u}_{P, 1}(s_1) - b - \frac{\bslt}{4}} + \frac{\kappa}{2} \lambda^2 \;. \label{eq:min-max-duality}
\end{align}
where the equality holds by the strong duality of regularized CMDPs (see, e.g., \textbf{Appendix C.1} in  \citet{ding2023last}).
Let $(\widetilde{\pi}, \widetilde{\lambda})$ be a saddle point of the problem, which is ensured to be unique thanks to the regularization. 
We first show the analytical forms of $(\widetilde{\pi}, \widetilde{\lambda})$.

\paragraph{Analytical forms of $(\widetilde{\pi}, \widetilde{\lambda})$.}
Due to the strong duality, we have
\begin{align*}
\max_{\pi \in \Pi} \vf{\pi, r^\bdelta \widetilde{\lambda} (u - \delta^u)}_{P, 1}[\kappa](s_1)
= \vf{\widetilde{\pi}, r^\bdelta + \widetilde{\lambda}(u - \delta^u)}_{P, 1}[\kappa](s_1)\;.
\end{align*}
Since the left-hand side is an entropy-regularized optimization problem in an MDP, the well-known analytical solution of regularized MDP indicates that \citep{geist2019theory}:
\begin{align}\label{eq:pilambda-analytical}
\widetilde{\pi}_{h}\paren*{\cdot \given s} = \softmax\paren*{\frac{1}{\kappa}\paren*{
\qf{\widetilde{\pi}, r^\bdelta}_{P, h}[\kappa](s, \cdot) + \widetilde{\lambda} \qf{\widetilde{\pi}, u - \delta^u}_{P, h}(s, \cdot)}}
=\pi^{\bdelta, \widetilde{\lambda}}_h\;, 
\end{align}
where the last equality is due to the definition of $\pi^{\bdelta, \lambda}_h$.
Additionally, due to the strong duality,
\begin{align*}
\widetilde{\lambda} \in \argmin_{\lambda \geq 0}\vf{\widetilde{\pi}, r^\bdelta}_{P, 1}[\kappa](s_1)  + \lambda \paren*{\vf{\widetilde{\pi}, u - \delta^u}_{P, 1}(s_1) - b - \frac{\bslt}{4}} + \frac{\kappa}{2} \lambda^2 \;.
\end{align*}
Since the right-hand side is a quadratic equation on $\lambda$, we have
\begin{align}
\widetilde{\lambda} = \frac{1}{\kappa} \brack*{b + \frac{\bslt}{4} - \vf{\widetilde{\pi}, u - \delta^u}_{P, 1}(s_1)}_+\;. \label{eq:lambda-analytical}
\end{align}

\looseness=-1
\paragraph{$\widetilde{\lambda}$ upper bound.}
Next, we will show that $\widetilde{\lambda}$ is upper bounded by constant.
We have
\begin{align*}
2H^2_\kappa (B_\dagger + 1)
&\numeq{\geq}{a} 
\vf{\widetilde{\pi}, r^\bdelta}_{P, 1}[\kappa](s_1)  - \underbrace{\frac{1}{2\kappa}\brack*{b + \frac{\bslt}{4} - \vf{\widetilde{\pi}, u - \delta^u}_{P, 1}(s_1)}^2_+}_{\geq 0}\\
&\numeq{=}{b} \vf{\widetilde{\pi}, r^\bdelta}_{P, 1}[\kappa](s_1)  + \widetilde{\lambda} \paren*{\vf{\widetilde{\pi}, u - \delta^u}_{P, 1}(s_1) - b - \frac{\bslt}{4}} + \frac{\kappa}{2} \widetilde{\lambda}^2\\
&\numeq{\geq}{c} \vf{\pisafe, r^\bdelta}_{P, 1}[\kappa](s_1) 
+ \widetilde{\lambda} \paren*{\vf{\pisafe, u - \delta^u}_{P, 1}(s_1) - b - \frac{\bslt}{4}} + \frac{\kappa}{2} \widetilde{\lambda}^2\\
&\geq \widetilde{\lambda} \paren*{\underbrace{\vf{\pisafe, u}_{P, 1}(s_1) - b - \frac{\bslt}{4}}_{\geq 3\bslt / 4} 
-\underbrace{\vf{\pisafe,2C_\beta\beta^{(k)}}_{P, 1}(s_1)}_{\leq \bslt / 2 \;\text{ since } k \in \unconfMDP^\complement}
} \geq \widetilde{\lambda} \frac{\bslt}{4}\;,
\end{align*}
where (a) is since $\norm*{r^\bdelta}_\infty = \norm*{B_\dagger \beta^{(k)} + \delta^\dagger + r + \delta^r}_\infty 
\leq B_\dagger + B_\dagger H + 1 + H
= (H + 1)(B_\dagger + 1)$, (b) is due to \Cref{eq:lambda-analytical}, (c) uses \Cref{eq:min-max-duality}.
By reformulating the inequality, 
\begin{align}\label{eq:lambda-upper-bound}
\widetilde{\lambda}  \leq  \frac{8H_\kappa^2(B_\dagger + 1)}{\bslt}\;.  
\end{align}

\paragraph{Constraint violation of $\pi^{\bdelta, \lambda}$}
Finally, we will show that for any $\lambda \geq \widetilde{\lambda}$, $\pi^{\bdelta, \lambda}$ guarantees zero constraint violation.
Due to \Cref{eq:pilambda-analytical,,eq:lambda-analytical,,eq:lambda-upper-bound}, we have
\begin{align*}
\kappa \widetilde{\lambda} = \brack*{b + \frac{\bslt}{4}- \vf{\pi^{\bdelta, \widetilde{\lambda}}, u - \delta^u}_{P, 1}(s_1)}_+ \leq 
\frac{8\kappa H^2_\kappa (B_\dagger + 1)}{\bslt}\;,
\end{align*}
which ensures the small violation of $\pi^{\bdelta, \widetilde{\lambda}}$ when $\kappa \ll 1$. 
Since $\vf{\pi^{\bdelta, \lambda}, u - \delta^u}_{P, 1}(s_1)$ is monotonically increasing in $\lambda$ due to \Cref{lemma:softmax-value-monotonicity}, for any $\lambda \geq \widetilde{\lambda}$, $\vf{\pi^{\bdelta, \lambda}, u - \delta^u}_{P, 1}(s_1) \geq b + \frac{\bslt}{4}- \frac{8\kappa H^2_\kappa (B_\dagger + 1)}{\bslt}$. 
Therefore, by setting 
$
\kappa \leq \frac{\bslt^2}{32H_\kappa^2 (B_\dagger + 1)}
$, we have 
$\vf{\pi^{\bdelta, \lambda}, u - \delta^u}_{P, 1}(s_1) \geq b$.
\end{proof}

\begin{lemma}[Restatement of \Cref{lemma:trigger-condition-main}]
If \Cref{algo:zero-vio-linear MDP} is run with $\rho = 1$, $C_\lambda \geq \frac{8H_\kappa^2 (B_\dagger + 1)}{\bslt}$, and $\kappa \leq \frac{\bslt^2}{32H_\kappa^2 (B_\dagger + 1)}$,
under $\confevent$, 
it holds $\pvf{\pi^{(k), C_\lambda}, u}_{(k), 1}(s_1) \geq b$ for any $k \in \unconfMDP^\complement$.
\end{lemma}
\begin{proof}
Due to $\confevent$, it holds that
$$
\bdelta \df \paren*{\delta^{\pi^{(k), C_\lambda}, r}_{(k), \cdot},
\delta^{\pi^{(k), C_\lambda}, u}_{(k), \cdot}, 
\delta^{\pi^{(k), C_\lambda}, \dagger}_{(k), \cdot}} \in \Deltac \times \Deltar \times \Deltau\;.
$$
\looseness=-1
According to \Cref{eq:Q-Q-diff}, this $\bdelta$ satisfies
$\pi^{\bdelta, C_\lambda} = \pi^{(k), C_\lambda}$ where $\pi^{\bdelta, C_\lambda}$ is defined in \Cref{def:fixed-delta-softmax}.
Therefore, using \Cref{lemma:feasible-lambda-exists}, 
$\pvf{\pi^{(k), C_\lambda}, u}_{(k), 1}(s_1) \geq b$.
This concludes the proof.
\end{proof}

\subsubsection{Proof of \Cref{lemma:Ck-bound-MDP-main}}
\begin{lemma}[Bonus summation bound]\label{lemma:elliptical-potential-MDP}
\looseness=-1
If \Cref{algo:zero-vio-linear MDP} is run with $\rho = 1$, under $\rmvEevent$ and $\confevent$, it holds that
\begin{align*}
&\sum_{k=1}^K
\paren*{\vf{\pi^{(k)}, \beta^{(k)}}_{P, 1}(s_1)}^2   
\leq 2 H^2 d\ln\paren*{1 + \frac{K}{d}}+ 4 H^2 \ln \frac{2KH}{\delta}
= \tiO\paren*{H^2d}\\
\text{and }\;
&\sum_{k=1}^K
\paren*{\vf{\pi^{(k)}, \beta^{(k)}}_{P, 1}(s_1)} \leq 
H\sqrt{K}
\sqrt{
2 d\ln\paren*{1 + \frac{K}{d}}+ 4 \ln \frac{2KH}{\delta}
}
=\tiO\paren*{H\sqrt{d K}}\;.\\
\end{align*}
\end{lemma}
\begin{proof}
\looseness=-1
We have 
\begin{align*}
\sum_{k=1}^K
\paren*{\vf{\pi^{(k)}, \beta^{(k)}}_{P, 1}(s_1)}^2
&=
\sum_{k=1}^K
\paren*{
\sum_{h=1}^H
\E\brack*{\beta^{(k)}_h(s_h, a_h)\given s_h, a_h \sim \pi^{(k)}}}^2\\
&\numeq{\leq}{a}
H
\sum_{k=1}^K
\sum_{h=1}^H
\paren*{
\E\brack*{\beta^{(k)}_h(s_h, a_h)\given s_h, a_h \sim \pi^{(k)}}}^2\\
&\numeq{\leq}{a}
H \sum_{k=1}^K
\sum_{h=1}^H
\E\brack*{\norm*{\bphi(s_h, a_h)}^2_{\paren*{\bLambda_h^{(k)}}^{-1}}\given s_h, a_h \sim \pi^{(k)}}\\
&\numeq{\leq}{b}
2H\sum_{k=1}^K
\sum_{h=1}^H
\norm*{\bphi(s_h^{(k)}, a_h^{(k)})}^2_{\paren*{\bLambda_h^{(k)}}^{-1}}
+ 4 H^2\ln \frac{2KH}{\delta} \\
&\numeq{\leq}{c} 2 H^2 d\ln\paren*{1 + \frac{K}{d}}+ 4 H^2 \ln \frac{2KH}{\delta} \;,
\end{align*}
where (a) is due to Jensen's inequality, (b) is due to $\rmvEevent$, and (c) uses \Cref{lemma:elliptical potential}.   
The second claim follows by:
\begin{align*}
\sum_{k=1}^K  \vf{\pi^{(k)}, \beta^{(k)}}_{P, 1}(s_1)
\numeq{\leq}{a} \sqrt{K}\sqrt{\sum_{k=1}^K \paren*{\vf{\pi^{(k)}, \beta^{(k)}}_{P, 1}(s_1)}^2}
\numeq{\leq}{b} 
H\sqrt{K}
\sqrt{
2 d\ln\paren*{1 + \frac{K}{d}}+ 4 \ln \frac{2KH}{\delta}
}\;,
\end{align*}
where (a) uses Cauchy–Schwarz inequality and (b) uses the first claim.
\end{proof}

\begin{lemma}[Restatement of \Cref{lemma:Ck-bound-MDP-main}]\label{lemma:Ck-bound-MDP}
Suppose \Cref{algo:zero-vio-linear MDP} is run with $\rho = 1$ and $\rmvEevent$ and $\confevent$ hold.
Then, 
\begin{align*}
|\unconfMDP|
\leq
\frac{64\cp^2 H^2 d}{\bslt^2}\ln\paren*{\frac{2KH}{\delta}}
= \tiO\paren*{\bslt^{-2} H^4 d^3}\;,
\end{align*}
where the last equality sets $\cp = \tiO\paren*{dH}$.
\end{lemma}
\begin{proof}
\looseness=-1
Using \Cref{lemma:elliptical-potential-MDP} and \Cref{def:unconf-set MDP}, we have 
\begin{align*}
\abs{\unconfMDP} \paren*{\frac{\bslt}{2}}^2
\leq
\sum_{k\in \unconfMDP}
\paren*{\vf{\pisafe, 2\cp\beta^{(k)}}_{P, 1}(s_1)}^2
\leq 
8\cp^2 H^2 d\ln\paren*{1 + \frac{K}{d}}+ 16\cp^2 H^2 \ln \frac{2KH}{\delta} \;.
\end{align*}
Therefore, we have
\begin{align*}
\abs{\unconfMDP}
\leq
\frac{32\cp^2 H^2 d}{\bslt^2}\ln\paren*{1 + \frac{K}{d}}
+ \frac{64 \cp^2 H^2}{\bslt^2}\ln \frac{2KH}{\delta} 
\leq
\frac{64\cp^2 H^2 d}{\bslt^2}\ln\paren*{\frac{2KH}{\delta}}
\;.
\end{align*}
\end{proof}

\subsection{Proofs for Sublinear Regret Guarantee (\Cref{subsec:MDP-regret-analysis})}

\looseness=-1
Suppose the good events $\rmvEevent\cap \confevent$ hold.
We decompose the regret as follows:
\begin{align}
&\regret(K) \nonumber \\
=& \sum_{k=1}^K \paren*{\vf{\pi^\star, r}_{P, 1}(s_1) - \vf{\pi^{(k)}, r}_{P, 1}(s_1)} \nonumber \\
=& 
\sum_{k\in \unconfMDP} \paren*{\vf{\pi^\star, r}_{P, 1}(s_1) - \vf{\pi^{(k)}, r}_{P, 1}(s_1)}  
+ 
\sum_{k \in \unconfMDP^\complement} \paren*{\vf{\pi^\star, r}_{P, 1}(s_1) - \vf{\pi^{(k)}, r}_{P, 1}(s_1)} \nonumber\\
\leq &
|\unconfMDP| H
+ 
\sum_{k \in \unconfMDP^\complement} \paren*{\ovf{\pi^{(k)}, r}_{(k), 1}[\kappa](s_1) - \vf{\pi^{(k)}, r}_{P, 1}(s_1)}
+
\sum_{k\in \unconfMDP^\complement} \paren*{\vf{\pi^\star, r}_{P, 1}(s_1) - \ovf{\pi^{(k)}, r}_{(k), 1}[\kappa](s_1)}\nonumber
\\
\numeq{\leq}{a} &
\tiO\paren*{d^3H^4 \bslt^{-2}}
+ 
\sum_{k \in \unconfMDP^\complement} \paren*{\ovf{\pi^{(k)}, r}_{(k), 1}[\kappa](s_1) - \vf{\pi^{(k)}, r}_{P, 1}[\kappa](s_1)}
+
\sum_{k\in \unconfMDP^\complement} \paren*{\vf{\pi^\star, r}_{P, 1}(s_1) - \ovf{\pi^{(k)}, r}_{(k), 1}[\kappa](s_1)}
+ \kappa KH \ln A\nonumber\\
\numeq{\leq}{b} &
\tiO\paren*{d^3H^4 \bslt^{-2}}
+ 
\underbrace{
2\co \sum_{k \in \unconfMDP^\complement} \vf{\pi^{(k)}, \beta^{(k)}}_{P, 1}(s_1)
}_{\circled{1}}
+
\underbrace{\sum_{k\in \unconfMDP^\complement} \paren*{\vf{\pi^\star, r}_{P, 1}(s_1) - \ovf{\pi^{(k)}, r}_{(k), 1}[\kappa](s_1)}}_{\circled{2}}
+ \kappa KH \ln A
\;, \label{eq:regret-decomposition}
\end{align}   
where (a) uses \Cref{lemma:Ck-bound-MDP} and (b) is due to \Cref{lemma:opt-pes-MDP} with $\confevent$.
Under $\rmvEevent \cap \confevent$, $\circled{1}$ can be easily bounded by \Cref{lemma:elliptical-potential-MDP}
\begin{align}\label{eq:martingale-bound}
\circled{1} \leq 
\co \tiO(H\sqrt{dK})
\leq \tiO\paren*{H_\kappa^{2} d^{3/2}\sqrt{K}}\;,
\end{align}
where the last equality inserts $\co = \tiO\paren*{dH_\kappa}$.

\subsubsection{Mixture Policy Decomposition}

\looseness=-1
We upper bound $\circled{2}$ in \Cref{eq:regret-decomposition} by the mixture policy technique.

\begin{lemma}[Mixture policy's feasibility]\label{lemma:mixture-feasibility}
Let $\alpha^{(k)} \df \frac{\bslt}{\bslt + 2 \vf{\pi^\star, 2\cp \beta^{(k)}}_{P, 1}(s_1)}$.
For any $k \in \unconfMDP^\complement$ and $\alpha \in [0, \alpha^{(k)}]$, $\pi^\alpha$ defined in \Cref{def:mixture-policy} satisfies
\(
\vf{\pi^\alpha, u-2\cp \beta^{(k)}}_{P, 1}(s_1) \geq b
\).
\end{lemma}
\begin{proof}
We have    
\begin{align*}
&\vf{\pi^\alpha, u-2\cp \beta^{(k)}}_{P, 1}(s_1) - b\\
=& 
(1-\alpha)\paren*{\vf{\pisafe, u-2\cp \beta^{(k)}}_{P, 1}(s_1) - b}
+ \alpha\paren*{\vf{\pi^\star, u-2\cp \beta^{(k)}}_{P, 1}(s_1) - b}\\
\geq&
(1-\alpha)\frac{\bslt}{2}
+ \alpha\paren*{\vf{\pi^\star, -2\cp \beta^{(k)}}_{P, 1}(s_1)}\;,
\end{align*}
where the last inequality holds because $\vf{\pisafe, 2\cp\beta^{(k)}}_{P, 1}(s_1) \leq \frac{\bslt}{2}$ due to $k \in \unconfMDP^\complement$.
Thus, $\vf{\pi^\alpha, u-2\cp \beta^{(k)}}_{P, 1}(s_1) - b \geq 0$ holds when 
\begin{align*}
\alpha \leq \frac{\bslt}{\bslt + 2 \vf{\pi^\star, 2\cp \beta^{(k)}}_{P, 1}(s_1)}\;.
\end{align*}
\end{proof}

\begin{lemma}[Mixture policy's optimism]\label{lemma:mixture-optimism}
Let $B_\dagger \geq \frac{4\cp H}{\bslt}$.
For any $k \in \unconfMDP^\complement$, $\pi^{\alpha^{(k)}}$ with $\alpha^{(k)}$ from \Cref{lemma:mixture-feasibility} satisfies,
\begin{align*}
\vf{\pi^{\alpha^{(k)}}, r+B_\dagger \beta^{(k)}}_{P, 1}(s_1)  \geq 
\vf{\pi^\star, r}_{P, 1}(s_1) 
\;\text{ and }\;
\vf{\pi^{\alpha^{(k)}}, u-2\cp \beta^{(k)}}_{P, 1}(s_1) \geq b\;.
\end{align*}
\end{lemma}
\begin{proof}
The sufficient condition that
\(
\vf{\pi^\alpha, r+B_\dagger \beta^{(k)}}_{P, 1}(s_1) 
\geq
\vf{\pi^\star, r}_{P, 1}(s_1)
\) to hold is 
\begin{align*}
B_\dagger
&\geq 
\frac{\vf{\pi^\star, r}_{P, 1}(s_1) - \vf{\pi^\alpha, r}_{P, 1}(s_1)}{\vf{\pi^\alpha, \beta^{(k)}}_{P, 1}(s_1)}
= 
\frac{(1-\alpha)\paren*{\vf{\pi^\star, r}_{P, 1}(s_1) - \vf{\pisafe, r}_{P, 1}(s_1)}}
{(1-\alpha)\vf{\pisafe, \beta^{(k)}}_{P, 1}(s_1)
+ \alpha\vf{\pi^\star, \beta^{(k)}}_{P, 1}(s_1)}\\
&= 
\frac{\vf{\pi^\star, r}_{P, 1}(s_1) - \vf{\pisafe, r}_{P, 1}(s_1)}
{\vf{\pisafe, \beta^{(k)}}_{P, 1}(s_1)
+ \frac{\alpha}{1-\alpha}\vf{\pi^\star, \beta^{(k)}}_{P, 1}(s_1)}\;.
\end{align*} 
By inserting $\alpha^{(k)} = \frac{\bslt}{\bslt + 2 \vf{\pi^\star, 2\cp \beta^{(k)}}_{P, 1}(s_1)}$ into $\alpha$, i.e., 
$\frac{\alpha}{1-\alpha} = \frac{\bslt}{2 \vf{\pi^\star, 2\cp \beta^{(k)}}_{P, 1}(s_1)}$, 
\begin{align*}
B_\dagger
\geq
\frac{\vf{\pi^\star, r}_{P, 1}(s_1) - \vf{\pisafe, r}_{P, 1}(s_1)}
{\vf{\pisafe, \beta^{(k)}}_{P, 1}(s_1)
+ \frac{\bslt}{4\cp \vf{\pi^\star, \beta^{(k)}}_{P, 1}(s_1)}\vf{\pi^\star, \beta^{(k)}}_{P, 1}(s_1)}
= 
\frac{4\cp\paren*{\vf{\pi^\star, r}_{P, 1}(s_1) - \vf{\pisafe, r}_{P, 1}(s_1)}}
{2\vf{\pisafe, 2\cp \beta^{(k)}}_{P, 1}(s_1) + \bslt}\;.
\end{align*} 
Thus, when $B_\dagger \geq \frac{4\cp H}{\bslt}$, 
it holds that \(\vf{\pi^{\alpha^{(k)}}, r+B_\dagger \beta^{(k)}}_{P, 1}(s_1)  \geq \vf{\pi^\star, r}_{P, 1}(s_1)\).
The second claim follows from \Cref{lemma:mixture-feasibility}.
\end{proof}

\looseness=-1
We are now ready to decompose $\circled{2}$.
Using \Cref{lemma:mixture-feasibility,,lemma:mixture-optimism}, we have
\begin{equation}\label{eq:optimism-decomposition}
\begin{aligned}
\circled{2}
=& \sum_{k\in \unconfMDP^\complement} \paren*{\vf{\pi^\star, r}_{P, 1}(s_1) - \ovf{\pi^{(k)}, r}_{(k), 1}[\kappa](s_1)}\\
\leq& \sum_{k\in \unconfMDP^\complement} \paren*{
\vf{\pi^{\alpha^{(k)}}, B_\dagger \beta^{(k)}}_{P, 1}(s_1)  
+ \vf{\pi^{\alpha^{(k)}}, r}_{P, 1}[\kappa](s_1)  
- \ovf{\pi^{(k)}, r}_{(k), 1}[\kappa](s_1)}\\
=&
\sum_{k\in \unconfMDP^\complement} 
\Big(\vf{\pi^{\alpha^{(k)}}, B_\dagger \beta^{(k)}}_{P, 1}(s_1) 
+ \vf{\pi^{\alpha^{(k)}}, r}_{P, 1}[\kappa](s_1)
+ \widebar{\lambda}^{(k, T)} \vf{\pi^{\alpha^{(k)}}, u - 2\cp \beta^{(k)}}_{P, 1}(s_1)\\
&
\underbrace{\quad\quad
- \ovf{\pi^{(k)}, \dagger}_{(k), 1}(s_1)
- \ovf{\pi^{(k)}, r}_{(k), 1}[\kappa](s_1)
- \widebar{\lambda}^{(k, T)} \pvf{\pi^{(k)}, u}_{(k), 1}(s_1)
\Big)\quad \quad\quad\quad\quad}_{\circled{3}}\\
&+ 
\underbrace{\sum_{k\in \unconfMDP^\complement}  \ovf{\pi^{(k)}, \dagger}_{(k), 1}(s_1) }_{\circled{4}}
+ 
\underbrace{\sum_{k\in \unconfMDP^\complement} 
\widebar{\lambda}^{(k, T)}
\paren*{
\pvf{\pi^{(k)}, u}_{(k), 1}(s_1) - \vf{\pi^{\alpha^{(k)}}, u - 2\cp\beta^{(k)}}_{P, 1}(s_1)
}}_{\circled{5}} \;,
\end{aligned}
\end{equation}
where $\widebar{\lambda}^{(k, T)}$ is defined in \Cref{line:pik-deploy}.
Using \Cref{lemma:opt-pes-MDP}, the term $\circled{4}$ is bounded as
\begin{align*}
\circled{4} \leq \vf{\pi^{(k)}, (B_\dagger + 2C_\dagger)\beta^{(k)}}_{P, 1}(s_1)\;.
\end{align*}
Using \Cref{lemma:elliptical-potential-MDP}, it holds that
\begin{align}\label{eq:additional bonus bound}
\circled{4} \leq 
(B_\dagger + 2C_\dagger) \tiO\paren*{H\sqrt{dK}}
= \tiO\paren*{H^{4} d^{5/2} \bslt^{-1}\sqrt{K}}\;,
\end{align}
where the last equality inserts $B_\dagger = 4\bslt^{-1}\cp H$, $\cp = \tiO(dH)$, and $C_\dagger=\tiO(dHB_\dagger)$.
We will bound $\circled{3}$ and $\circled{5}$ separately.

\subsubsection{Optimistic Bounds}

\begin{lemma}[Optimism in composite value function]\label{lemma:optimism-composite}
Suppose $\confevent$ holds. Then,
\begin{align*}
\circled{3} = 
&\sum_{k\in \unconfMDP^\complement} 
\Big(\vf{\pi^{\alpha^{(k)}}, B_\dagger \beta^{(k)}}_{P, 1}(s_1) 
+ \vf{\pi^{\alpha^{(k)}}, r}_{P, 1}[\kappa](s_1)
+ \widebar{\lambda}^{(k, T)} \vf{\pi^{\alpha^{(k)}}, u - 2\cp \beta^{(k)}}_{P, 1}(s_1)\\
&
\quad\quad
- \ovf{\pi^{(k)}, B_\dagger \beta^{(k)}}_{(k), 1}(s_1)
- \ovf{\pi^{(k)}, r}_{(k), 1}[\kappa](s_1)
- \widebar{\lambda}^{(k, T)} \pvf{\pi^{(k)}, u}_{(k), 1}(s_1)
\Big) \leq 0\;.
\end{align*}
\end{lemma}
\begin{proof}
Using \Cref{lemma:regularized value difference}, for any $k \in \unconfMDP^\complement$, we have
\begin{align*}
& \ovf{\pi^{(k)}, B_\dagger \beta^{(k)}}_{(k), 1}(s_1)
+\ovf{\pi^{(k)}, r}_{(k), 1}[\kappa](s_1)
+\widebar{\lambda}^{(k, T)} \pvf{\pi^{(k)}, u}_{(k), 1}(s_1) \\
& -\vf{\pi^{\alpha^{(k)}}, B_\dagger \beta^{(k)}}_{P, 1}(s_1) 
- \vf{\pi^{\alpha^{(k)}}, r}_{P, 1}[\kappa](s_1)
- \widebar{\lambda}^{(k, T)} \vf{\pi^{\alpha^{(k)}}, u - 2\cp \beta^{(k)}}_{P, 1}(s_1)\\
=&
\vf{\pi^{\alpha^{(k)}}, f^1}_{P, 1}(s_1) 
+ \vf{\pi^{\alpha^{(k)}}, f^2}_{P, 1}(s_1) 
+ \widebar{\lambda}^{(k, T)}\vf{\pi^{\alpha^{(k)}}, 2\cp \beta^{(k)}}_{P, 1}(s_1) 
\end{align*}
where $f^1: \HSA\to \R$ and $f^2: \HSA\to \R$ are functions such that
\begin{align*}
f^1_h(s, a) &= 
\sum_{a \in \A}
\paren*{
\pi^{(k)}_h\paren{a \given s} 
\paren*{\oqf{\pi^{(k)}, r}_{(k), h}[\kappa](s, a) + \widebar{\lambda}^{(k, T)} \pqf{\pi^{(k)}, u}_{(k), h}(s, a)
- \kappa \ln \pi^{(k)}_h\paren{a \given s}}
}\\
&\quad - 
\sum_{a \in \A}
\paren*{
\pi^{\alpha^{(k)}}_h\paren{a \given s} 
\paren*{\oqf{\pi^{(k)}, r}_{(k), h}(s, a) + \widebar{\lambda}^{(k, T)} \pqf{\pi^{(k)}, u}_{(k), h}(s, a)
- \kappa \ln \pi^{\alpha^{(k)}}_h\paren{a \given s}}
}\\
f^2_h(s, a) &= \delta^{\pi^{(k)}, r}_{(k)} - \widebar{\lambda}^{(k, T)}\delta^{\pi^{(k)}, u}_{(k)}\;.
\end{align*}

It is well-known that the analytical maximizer of 
\(
\max_{\pi \in \sP(\A)}
\sum_{a \in \A}
\pi\paren{a} \paren*{\bx(a) - \kappa \ln \pi\paren{a}}
\) is $\softmax\paren*{\frac{1}{\kappa} \bx(\cdot)}$. Therefore, the function $f^1$ is non-negative and thus $\vf{\pi^{\alpha^{(k)}}, f^1}_{P, 1}(s_1) \geq 0$.

\looseness=-1
On the other hand, using \Cref{lemma:delta-bound}, we have 
\begin{align*}
f^2_h(s, a) 
= \delta^{\pi^{(k)}, r}_{(k), h} - \widebar{\lambda}^{(k, T)}\delta^{\pi^{(k)}, u}_{(k), h}
\numeq{\geq}{a}
- \widebar{\lambda}^{(k, T)} 2\cp \beta^{(k)}_h
\end{align*}
Therefore, it holds that 
\begin{align*}
\vf{\pi^{\alpha^{(k)}}, f^2}_{P, 1}(s_1) 
+ \widebar{\lambda}^{(k, T)}\vf{\pi^{\alpha^{(k)}}, 2\cp \beta^{(k)}}_{P, 1}(s_1) 
\geq 0\;.
\end{align*}

\looseness=-1
By combining all the results, we have \(\circled{3} \leq 0\).
\end{proof}

\subsubsection{Bounds for Bisection Search}

\looseness=-1
Using \Cref{lemma:mixture-feasibility}, $\circled{5}$ is further bounded by 
\begin{align*}
\circled{5} &= \sum_{k\in \unconfMDP^\complement}
\widebar{\lambda}^{(k, T)}
\paren*{
\pvf{\pi^{(k)}, u}_{(k), 1}(s_1) - \vf{\pi^{\alpha^{(k)}}, u - 2\cp\beta^{(k)}}_{P, 1}(s_1)
}\\
& \leq 
\sum_{k\in \unconfMDP^\complement}
\widebar{\lambda}^{(k, T)}
\paren*{
\pvf{\pi^{(k)}, u}_{(k), 1}(s_1) - b
}
\leq 
C_\lambda
\sum_{k\in \unconfMDP^\complement}
\paren*{
\pvf{\pi^{(k)}, u}_{(k), 1}(s_1) - b
}
\;.
\end{align*}
We bound the last term using the bisection search in \Cref{algo:zero-vio-linear MDP}.
Note that we focus only the case $\pvf{\pi^{(k), 0}, u}_{(k), 1}(s_1) < b$ and $\pvf{\pi^{(k), C_\lambda}, u}_{(k), 1}(s_1) \geq b$ due to \Cref{line:pizero-deploy} and \Cref{line:pisafe-deploy} in \Cref{algo:zero-vio-linear MDP}. 
Due to the definitions of $\widebar{\lambda}^{(k, t)}$ and $\underline{\lambda}^{(k, t)}$ in \Cref{algo:zero-vio-linear MDP}, 
$$
\pvf{\pi^{(k), \underline{\lambda}^{(k, t)}}, u}_{(k), 1}(s_1) < b
  \;\text{ and }\;
\pvf{\pi^{(k), \widebar{\lambda}^{(k, t)}}, u}_{(k), 1}(s_1) \geq b\;
$$
hold for any $t \in \bbrack{1, T}$. 
Therefore, 
\begin{align*}
\circled{5}
\leq &C_\lambda \sum_{k \in \unconfMDP^\complement} 
\paren*{
\pvf{\pi^{(k), \widebar{\lambda}^{(k, T)}}, u}_{(k), 1}(s_1) - 
\pvf{\pi^{(k), \underline{\lambda}^{(k, T)}}, u}_{(k), 1}(s_1)
}
\end{align*}
To bound the right-hand side, we derive the sensitivity of $\vf{\pi^{(k), \lambda}, u}_{(k), 1}(s_1)$ with respect to $\lambda$.

\begin{lemma}[Restatement of \Cref{lemma:lambda-sensitive-main}]\label{lemma:lambda-sensitive}
Let $X \df K \paren*{1 + \frac{8(1+C_\lambda)(H_\kappa + B_\dagger H + H)}{\kappa}}$ 
and 
$Y \df \frac{8(H_\kappa + B_\dagger H + H)}{\kappa}$.
For any $k$ and $\lambda \in [0, C_\lambda]$, it holds that
\begin{align*}
\abs*{\pvf{\pi^{(k), \lambda}, u}_{(k), 1}(s_1)
- \pvf{\pi^{(k), \lambda + \varepsilon}, u}_{(k), 1}(s_1)}
\leq X^H H^2 Y \varepsilon\;.
\end{align*}   
\end{lemma}
\begin{proof}
\looseness=-1
The proof is based on \textbf{Lemma 2} from \citet{ghosh2024towards}.
For notational simplicity, we denote $\pi \df \pi^{(k), \lambda}$ and $\pi' \df \pi^{(k), \lambda + \varepsilon}$.
Additionally, we use shorthand:
\begin{align*}
&v_h^r \df \norm*{\ovf{\pi, r}_{(k), h}[\kappa] - \ovf{\pi', r}_{(k), h}[\kappa]}_\infty\;,
&&q_h^r \df \norm*{\oqf{\pi, r}_{(k), h}[\kappa] - \oqf{\pi', r}_{(k), h}[\kappa]}_\infty\;,
\\
&v_h^\dagger \df \norm*{\ovf{\pi, \dagger}_{(k), h} - \ovf{\pi', \dagger}_{(k), h}}_\infty\;,
&&q_h^\dagger \df \norm*{\oqf{\pi, \dagger}_{(k), h} - \oqf{\pi', \dagger}_{(k), h}}_\infty\;,\\
&v_h^u \df \norm*{\pvf{\pi, u}_{(k), h} - \pvf{\pi', u}_{(k), h}}_\infty \;,
&&q_h^u \df \norm*{\pqf{\pi, u}_{(k), h} - \pqf{\pi', u}_{(k), h}}_\infty \;.
\end{align*}

\looseness=-1
For any $h$, we have
\begin{align*}
&v_h^r = \norm*{ \pi_h\oqf{\pi, r}_{(k), h}[\kappa] - \pi'_h \oqf{\pi', r}_{(k), h}[\kappa]}_\infty \leq H_\kappa\norm*{\pi_h - \pi'_h}_1 + q^r_h \\
&v_h^\dagger \leq B_\dagger H\norm*{\pi_h - \pi'_h}_1 + q_h^\dagger\\
&v_h^u \leq H\norm*{\pi_h - \pi'_h}_1 + q_h^u\;.
\end{align*}

\looseness=-1
Since $\pi_h$ and $\pi'_h$ are softmax policies, using \Cref{lemma:softmax-policy-bound},
\begin{align*}
\norm*{\pi_h - \pi'_h}_1
&\leq 
\frac{8}{\kappa}\norm*{
\oqf{\pi, \dagger}_{(k), h} 
+ \oqf{\pi, r}_{(k), h}[\kappa] 
+ \lambda \pqf{\pi, u}_{(k), h}
- \oqf{\pi', \dagger}_{(k), h}
-\oqf{\pi', r}_{(k), h}[\kappa]
- (\lambda+\varepsilon) \pqf{\pi', u}_{(k), h}}_\infty\\
&\leq 
\frac{8}{\kappa}
\paren*{
q_h^\dagger
+ q_h^r
+C_\lambda q_h^u
+ \varepsilon H
}
\end{align*}

\looseness=-1
Additionally, 
\begin{align*}
&q_h^r \leq \norm*{ \hP^{(k)}_h \paren*{\ovf{\pi, r}_{(k), h+1}[\kappa]  - \ovf{\pi', r}_{(k), h+1}[\kappa]} }_\infty \leq K v_{h+1}^r\\
&q_h^\dagger \leq \norm*{ \hP^{(k)}_h \paren*{\ovf{\pi, \dagger}_{(k), h+1}  - \ovf{\pi', \dagger}_{(k), h+1}} }_\infty \leq K v_{h+1}^\dagger\\
&q_h^u \leq \norm*{ \hP^{(k)}_h \paren*{\pvf{\pi, u}_{(k), h+1}  - \pvf{\pi', u}_{(k), h+1}} }_\infty \leq K v_{h+1}^u\;,
\end{align*}
where we used the fact that, for any $V: \S \to \R$,
\begin{align*}
\abs*{\hP^{(k)}_h V}(s, a)
&= 
\abs*{{\bphi(s, a)^{\top}\paren{\bLambda^{(k)}_h}^{-1} \sum_{i=1}^{k-1}\bphi\paren{s_h^{(i)}, a_h^{(i)}} V\paren{s_{h+1}^{(i)}}}}\\
&\leq \norm*{{\paren{\bLambda^{(k)}_h}^{-1} \sum_{i=1}^{k-1}\bphi\paren{s_h^{(i)}, a_h^{(i)}} }}_2 \norm*{V}_\infty
\leq K \norm*{V}_\infty\;.
\end{align*}

\looseness=-1
By combining all the results, 
\begin{align*}
&v_h^r \leq 
K\paren*{\frac{8H_\kappa}{\kappa} + 1} v_{h+1}^r
&&+K\frac{8H_\kappa}{\kappa} v_{h+1}^\dagger
&&&+ K\frac{8H_\kappa C_\lambda}{\kappa} v_{h+1}^u
&&&&+ \frac{8H_\kappa}{\kappa} \varepsilon H\\
&v_h^\dagger \leq 
K\frac{8B_\dagger H}{\kappa} v_{h+1}^r
&&+K\paren*{\frac{8B_\dagger H}{\kappa} + 1}  v_{h+1}^\dagger
&&&+ K\frac{8B_\dagger H C_\lambda}{\kappa} v_{h+1}^u
&&&&+ \frac{8B_\dagger H}{\kappa} \varepsilon H\\
&v_h^u \leq 
K\frac{8 H}{\kappa} v_{h+1}^r
&&+K\frac{8 H}{\kappa}  v_{h+1}^\dagger
&&&+ K\paren*{\frac{8H}{\kappa} + 1}C_\lambda v_{h+1}^u
&&&&+ \frac{8H}{\kappa} \varepsilon H\;.
\end{align*}
Let 
$X \df K \paren*{1 + \frac{8(1+C_\lambda)(H_\kappa + B_\dagger H + H)}{\kappa}}$ 
and 
$Y \df \frac{8(H_\kappa + B_\dagger H + H)}{\kappa}$.
Then,
\begin{align*}
v_h^r + v_h^\dagger + v_h^u 
&\leq  X(v_{h+1}^r + v_{h+1}^\dagger + v_{h+1}^u) + Y H \varepsilon\\
&\leq  X^2(v_{h+2}^r + v_{h+2}^\dagger + v_{h+2}^u) + XY H\varepsilon + Y H \varepsilon \\
&\leq \dots\\
&\leq \paren*{X^H+ \dots + X+ 1}Y H \varepsilon \;. 
\end{align*}
\end{proof}

\looseness=-1
We are now ready to bound $\circled{5}$
Applying \Cref{lemma:lambda-sensitive} to $\circled{5}$, we obtain the following lemma.
\begin{lemma}\label{lemma:binary-search-bound}
\looseness=-1
When $T = \tiO(H)$, it holds that 
$$\circled{5} \leq 
C_\lambda \sum_{k \in \unconfMDP^\complement} 
\paren*{
\pvf{\pi^{(k), \widebar{\lambda}^{(k, T)}}, u}_{(k), 1}(s_1) - 
\pvf{\pi^{(k), \underline{\lambda}^{(k, T)}}, u}_{(k), 1}(s_1)
}
\leq \tiO(1)\;.$$
\end{lemma}
\begin{proof}
\looseness=-1
Due to the bisection search update rule, $\widebar{\lambda}^{(k, T)} - \underline{\lambda}^{(k, T)} = 2^{-T}$.
Thus, 
\begin{align*}
\circled{5} 
&\leq 
C_\lambda \sum_{k \in \unconfMDP^\complement} 
\paren*{
\pvf{\pi^{(k), \widebar{\lambda}^{(k, T)}}, u}_{(k), 1}(s_1) - 
\pvf{\pi^{(k), \underline{\lambda}^{(k, T)}}, u}_{(k), 1}(s_1)
}
\leq
X^H C_\lambda K H^2 Y 2^{-T}
\end{align*}    
where the inequality uses \Cref{lemma:lambda-sensitive} with $X$ and $Y$ defined in \Cref{lemma:lambda-sensitive}.
Thus, $\circled{5}\leq \tiO\paren*{1}$ holds by setting $T = H \polylog(X, H, Y)$. 
This concludes the proof.
\end{proof}

\looseness=-1
We are now ready to prove \Cref{theorem:MDP-regret-main}.
The proof is under the parameters of: 
$\rho = 1$,
$\co=\tiO(dH)$, 
$\cp=\tiO(dH)$,
$C_\dagger=\tiO(d^2H^3\bslt^{-1})$,
$B_\dagger = \tiO\paren*{dH^2\bslt^{-1}}$,
$\kappa = \widetilde{\Omega}\paren*{\bslt^3 H^{-4}d^{-1}K^{-0.5}}$, 
$T=\tiO(H)$, and
$C_\lambda = \tiO\paren*{dH^4\bslt^{-2}}$.

\subsubsection{Proof of \Cref{theorem:MDP-regret-main}}

\looseness=-1
We condition the proof with the good events $\rmvEevent \cap \confevent$, which holds with probability at least $1-3\delta$ by \Cref{lemma:good-event1-MDP,,lemma:good-event2-MDP}.

\looseness=-1
In \Cref{algo:zero-vio-linear MDP}, the deployed policy switches between $\pisafe \in \Pisafe$ and the softmax policies.
Since \Cref{algo:zero-vio-linear MDP} deploys the softmax policies only when $\pvf{\pi^{(k), 0}, u}_{(k), 1}(s_1) \geq b$, due to \Cref{lemma:delta-bound} and the good events, all the deployed policies satisfy $\pi^{(k)} \in \Pisafe$ for all $k \in \bbrack{1, K}$.
This concludes the proof of the zero-violation guarantee.

\looseness=-1
Next, we derive the regret bound.
Recall from \Cref{eq:regret-decomposition} that
\begin{align*}
\regret(K) \leq     
\tiO\paren*{d^3H^4 \bslt^{-2}}
+ 
\circled{1}
+
\circled{2}
+ \kappa KH \ln A
\leq     
\tiO\paren*{d^3H^4 \bslt^{-2}}
+ 
\circled{1}
+
\circled{2}
+ \tiO(\sqrt{K})\;,
\end{align*}
where the second inequality is due to the value of $\kappa$.

\looseness=-1
Using \Cref{eq:martingale-bound}, 
$$
\circled{1} \leq 
\tiO\paren*{H^{2} d^{3/2}\sqrt{K}}.
$$

\looseness=-1
Using \Cref{eq:optimism-decomposition}, $\circled{2}$ can be decomposed as:
\begin{align*}
\circled{2} &\leq \circled{3} + \circled{4} + \circled{5}\;.
\end{align*}
Each term can be bounded as:
\begin{itemize}
\item $\circled{3} \leq 0$ by \Cref{lemma:optimism-composite}
\item $\circled{4} \leq \tiO\paren*{H^{4} d^{5/2} \bslt^{-1}\sqrt{K}}$ by \Cref{eq:additional bonus bound},
\item $\circled{5} \leq \tiO(1)$ by \Cref{lemma:binary-search-bound}
\end{itemize}

\looseness=-1
Finally, by combining all the results, we have 
\begin{align*}
\regret(K) &\leq 
\tiO\paren*{d^3H^4 \bslt^{-2}}
+ 
\tiO\paren*{H^{2} d^{3/2}\sqrt{K}}
+ 
\tiO\paren*{H^{4} d^{5/2} \bslt^{-1}\sqrt{K}}\;.
\end{align*}
This concludes the proof of the sublinear regret guarantee.

\newpage
\section{Numerical Experiments}\label{sec: experiment}

\looseness=-1
This section presents empirical results supporting \cref{theorem:MDP-regret-main}, which guarantees $\sqrt{K}$ regret and episode-wise safety of \MDPalgo.
We also evaluate how often \MDPalgo deploys the safe policy $\pisafe$, a key technique for achieving sublinear regret (\cref{lemma:Ck-bound-MDP-main}).
All experiments were conducted within 30 minutes using eight Intel Core i7 CPUs and 32 GiB of RAM.

\looseness=-1
The source code for the experiment is available at \url{https://github.com/matsuolab/Episode-Wise-Safe-Linear-CMDP}.

\looseness=-1
We compare \MDPalgo against the previous state-of-the-art linear CMDP algorithm by \citet{ghosh2024towards} and the tabular CMDP algorithm called DOPE \citep{bura2022dope}. 
\citet{ghosh2024towards} achieves $\tiO(\sqrt{K})$ bounds for both regret and violation regret, and DOPE achieves $\tiO(\sqrt{K})$ regret with zero episode-wise violation.

\looseness=-1
For a sequence of policies $\brace{\pi^{(k)}}_{k \in [1, K]}$, the violation regret is defined as:
\begin{align}\label{eq:vio-regret}
{\textstyle
\operatorname{Vio}(K) 
\df 
\sum^K_{k=1} 
\max\brace*{
b - 
\vf{\pi^{(k)}, u}_{P, 1}(s_1)
,\; 0
}
}\;.
\end{align}
Clearly, if all the policies satisfy $\pi^{(k)} \in \Pisafe$, the violation regret is zero.

\looseness=-1
Additionally, we also report the performance of a uniform policy defined by $\pi_h\paren{\cdot \given s} = 1 / \aA$ for all $h, s$, to highlight the sublinear regret of our algorithm.

\paragraph{Implementations of \citet{ghosh2024towards} and DOPE.}
\looseness=-1
\citet{ghosh2024towards}'s algorithm can be implemented similarly to ours, with a few modifications: remove the $\pisafe$ deployment trigger, eliminate the pessimism compensation bonuses by setting $C_\dagger = B_\dagger = 0$, and apply an optimistic constraint bonus instead of our pessimistic one (i.e., use a negative sign for $C_u$).
We use $C_r$ and $C_u$ to denote the bonus scaling parameters for \citet{ghosh2024towards}.
See Algorithm 1 of \citet{ghosh2024towards} for further implementation details.

\looseness=-1
The DOPE algorithm can be implemented in tabular environments with a moderately small state space.
It computes the policy $\pi^{(k)}$ by solving the following optimistic–pessimistic problem:
\begin{align}\label{eq:DOPE-optpes}
\pi^{(k)}\in 
\max_{\pi \in \Pi} \max_{P' \in \mathcal{P}^{(k)}} \vf{\pi, r + C_r \beta^{(k)}}_{P', 1}(s_1) 
\; \text{ such that }\; \vf{\pi, u - C_u \beta^{(k)}}_{P', 1}(s_1) \geq b\;,
\end{align}
where $\beta^{(k)}_h(s, a)$ denotes the bonus at step $h$ for the state-action pair $(s, a)$ and $\cP^{(k)}$ denotes the confidence set for the transition kernel.
Specifically, using the visitation count\footnote{$\done[\mathrm{E}]$ equals $1$ if the event $E$ is true, and $0$ otherwise. For two scalars $a$ and $b$, we use shorthand $a \vee b \df \max\{a, b\}$.} $n^{(k)}_h(s, a, s') \df \sum_{k'=1}^{k} \done \brack{s_h^{(k)}=s, a_h^{(k)}=a, s_{h+1}^{(k)}=s'}$, the bonus and the confidence set are defined as
\begin{align*}
\beta^{(k)}_h(s, a) = \sum_{s' \in \S}\gamma^{(k)}_h(s, a, s')
\;\text{ where }\;
&\gamma^{(k)}_h(s, a, s') \propto \sqrt{\frac{\hP^{(k)}_h\paren{s' \given s, a}(1 - \hP^{(k)}_h\paren{s' \given s, a})}{n_h^{(k)}(s, a) \vee 1}}\;,\\
n^{(k)}_h(s, a) \df \sum_{s' \in \S} n^{(k)}_h(s, a, s')\;,
\;\text{ and }\; & \hP^{(k)}_h\paren{s' \given s, a} \df \frac{n^{(k)}_h(s, a, s')}{n^{(k)}_h(s, a) \vee 1}\;.
\end{align*}

\looseness=-1
For simplicity, we omit absolute constants and logarithmic factors, and use this simplified form in all experiments.
Further implementation details can be found in \citet{bura2022dope}.

\looseness=-1
For each environment, we select the hyperparameters of each algorithm using heuristic adjustments to balance exploration and exploitation.
To ensure numerical stability, we assign relatively small values to these parameters. The detailed values are provided below.

\looseness=-1
\paragraph{Synthetic tabular environments.}
To evaluate the exact regret values, we conduct experiments on tabular CMDPs with a small state space size.
Tabular CMDP is the special case of linear CMDP with $d=\abs{\S}$ and allows us to compute the optimal policy $\pi^\star$ by linear programming.

\looseness=-1
We instantiated CMDPs with $\abs{\S}=5$, $\abs{\A}=3$, $H=4$, employing a construction strategy akin to that of \citet{dann2017unifying}.
For all $s, a, h$, the transition probabilities $P_h\paren*{\cdot \given s, a}$ were independently sampled from $\operatorname{Dirichlet}(0.1, \dots, 0.1)$.
This transition probability kernel is concentrated yet encompasses non-deterministic transition probabilities.

\looseness=-1
The reward values for the objective $r_h(s, a)$ are set to $0$ with probability $0.1$ and to $1$ otherwise.
The utility values for the constraint $u_h(s, a)$ are assigned in the same way.
The initial state $s_1$ is randomly chosen from $\S$ and fixed during the training.
The constraint threshold is set as $b = 0.6\max_{\pi \in \Pi}\vf{\pi, u}_{P, 1}(s_1)$.

\looseness=-1
We choose the hyperparameters of the algorithms as follows:
\begin{itemize}
    \item OPSE-LCMDP: $C_r=C_u=C_d=1.0$, $C_\lambda=300$, $B_\dagger = 1.0$, and $\kappa = 0.1$. 
    \item \citet{ghosh2024towards}: $C_r=C_u=1.0$ and $\kappa = 0.1$.
    \item DOPE \citep{bura2022dope}: $C_r = C_u=1.0$.
\end{itemize}

\looseness=-1
\paragraph{Media Streaming CMDP Environments.}
As a realistic environment, we also evaluate algorithms on the media streaming environment from \citet{bura2022dope}.
In the environment, a wireless base station (agent) transmits media to a device using either a fast or slow service option, each incurring different costs.
The slow and fast services correspond to actions $a=1$ and $a=2$, respectively.

\looseness=-1
The fast service succeeds with probability $\mu_1$, and the slow one with $\mu_2 = 1 - \mu_1$, where both follow independent Bernoulli distributions.
At each environment construction, we randomly sample $\mu_1$ from $[0.5, 0.9]$.
Packets received at the device are stored in a media buffer and played out according to a Bernoulli process with parameter $\rho$. 
We sample $\rho$ uniformly from $[0.1, 0.4]$.

\looseness=-1
Let $A_h, B_h \in \brace{0, 1}$ denote the number of arriving and departing packets, respectively.
The media buffer length represents the state, and transitions as $s_{h+1} = \min\brace*{\max\brace*{0, s_h + A_h - B_h}, L}$ where $L$ denotes the maximum buffer length. We set $L = 5$, $\abs{\S}= L + 1$, and $H=4$.
The initial state is set to $s_1 = 0$.

\looseness=-1
The objective is to deliver enough packets to the buffer while limiting the use of the fast service.
Accordingly, the agent receives a reward $r_h(s, \cdot) = \done\brace{s \geq 0.3 L}$ and incurs a constraint utility $u_h(\cdot, a) = \done\brace{a = 1}$.
The constraint threshold is set as $b = 0.6 \max_{\pi \in \Pi}\vf{\pi, u}_{P, 1}(s_1)$.

\looseness=-1
We choose the hyperparameters of the algorithms as follows:
\begin{itemize}
    \item OPSE-LCMDP: $C_r=C_u=C_d=2.0$, $C_\lambda=300$, $B_\dagger = 1.0$, and $\kappa = 0.1$. 
    \item \citet{ghosh2024towards}: $C_r=C_u=2.0$ and $\kappa = 0.1$.
    \item DOPE \citep{bura2022dope}: $C_r = C_u=1.0$.
\end{itemize}

\begin{figure}[tb!]
    \hspace{-1.5cm}
    \includegraphics[width=1.2\linewidth]{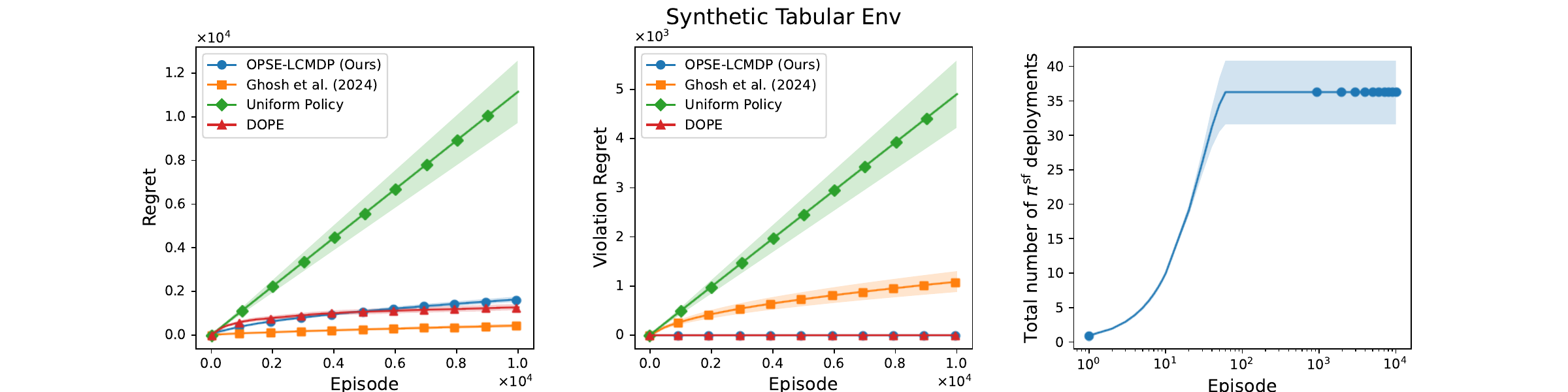}

    \vspace{0.3cm}
    \hspace{-1.5cm}
    \includegraphics[width=1.2\linewidth]{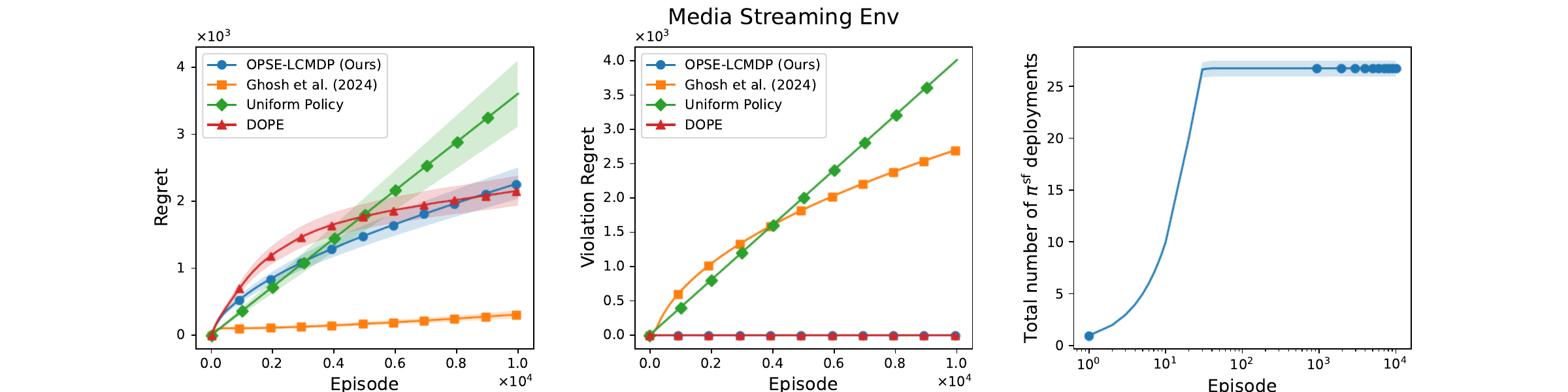}

    \vspace{0.3cm}
    \hspace{-1.5cm}
    \includegraphics[width=1.2\linewidth]{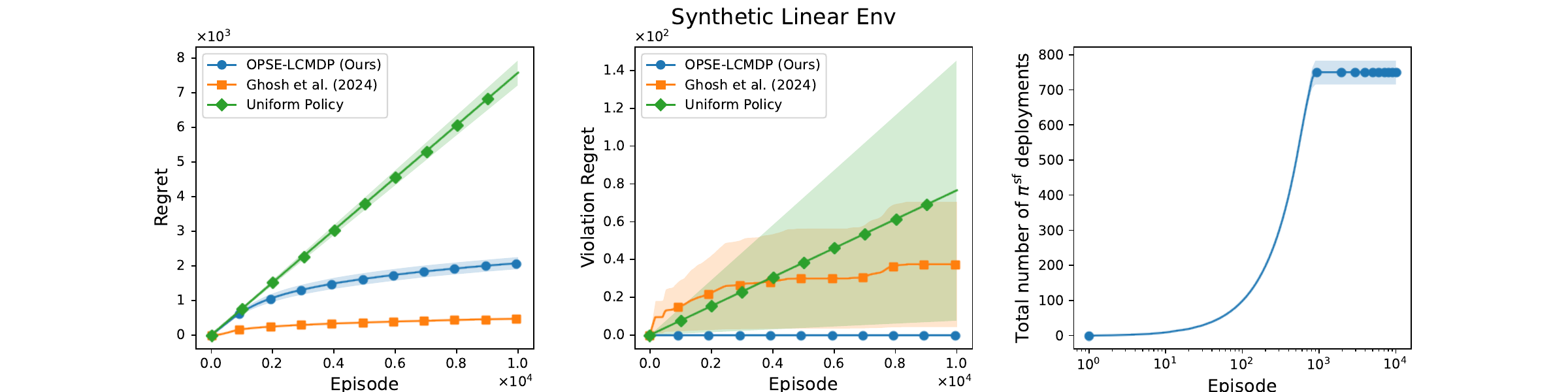}
    \caption{Numerical comparison of the algorithms in the synthetic tabular environment (\textbf{Top}), the media streaming environment (\textbf{Middle}), and the synthetic linear environment (\textbf{Bottom}). We do not run DOPE in the linear CMDP environment due to its computational intractability (see \cref{remark:DOPE not in linear}). \textbf{Left}: regret (\cref{eq:CMDP-goal}), \textbf{Middle}: violation regret (\cref{eq:vio-regret}), and \textbf{Right}: total number of $\pisafe$ deployments in \MDPalgo.}
    \label{fig:regret-experiment}
\end{figure}

\looseness=-1
\paragraph{Synthetic linear environments.}
Building on the experiment by \citet{amani2021safe}, we randomly construct linear CMDPs in which the number of states is larger than the feature map dimension. 
We test the algorithms on environments with $\aS=100$, $A=3$, $d=5$, and $H=4$.
This setup has a relatively large state space while still allowing us to analytically compute the optimal policy and exact regret. 

\looseness=-1
For each $(s, a) \in \SA$, the feature vector $\bphi(s, a) \in \R^{d}$ is sampled from $\operatorname{Dirichlet}(0.1, \dots, 0.1)$. 
Recall from \cref{assumption:linear mdp} the definition of $\bmu_h=\paren*{\bmu_h^1, \dots, \bmu_h^d} \in \R^{\aS\times d}$.
For each $(h, i) \in \bbrack{1, H}\times \bbrack{1, d}$, we sample $\bmu^i_h$ from $\operatorname{Dirichlet}(0.1, \dots, 0.1)$.
With these $\bmu$ and $\bphi$, we set $P_h\paren{s'\given s, a}=\bmu_h(s')^\top \bphi(s, a)$. This construction ensures that $P_h\paren{\cdot \given s, a}=1$ becomes a valid probability distribution for any $(h, s, a)$.

\looseness=-1
For the reward and utility functions, we sample both $\btheta_h^r$ and $\btheta_h^u$ from a uniform distribution over $[0, 1]^d$.
The reward and utility functions are then constructed such that $r_h(s, a)=(\btheta^r_h)^\top \bphi(s, a)$ and $u_h(s, a) = (\btheta^u_h)^\top \bphi(s, a)$.

\looseness=-1
The initial state $s_1$ is randomly chosen from $\S$ and fixed during the training.
The constraint threshold is set as $b = 0.68 \max_{\pi \in \Pi}\vf{\pi, u}_{P, 1}(s_1)$.

\begin{remark}\label{remark:DOPE not in linear}
We do not run the DOPE algorithm in this linear environment due to its heavy computational cost. 
Using the extended LP technique introduced by \citet{efroni2020exploration}, the optimization problem in \eqref{eq:DOPE-optpes} can be reformulated as the following standard LP problem:
$$
\begin{array}{ll}
\min _{\bx} & \bc^{\top} \bx \;\text { such that }\; A \bx=\bb \;\text{ and }\; G \bx \geq \bh\;,
\end{array}
$$
where parameters are defined appropriately.
This LP involves $H \aS^2 \aA$ decision variables and more than $H \aS^2 \aA$ number of constraints (see \citep{efroni2020exploration} for more details).
Therefore, in our synthetic linear CMDP experiment, the matrices $A$ and $G$ require at least $10^{10}$ entries, which is computationally intractable in practice.
\end{remark}

\looseness=-1
\paragraph{Results.}
\looseness=-1
\cref{fig:regret-experiment} shows the performance of the algorithms, averaged over 10 random seeds, with regret plotted on the left, violation regret in the middle, and the total number of $\pisafe$ deployments on the right.

\looseness=-1
Across all settings, both \MDPalgo and the algorithm by \citet{ghosh2024towards} exhibit sublinear regret.
However, while \MDPalgo maintains zero constraint violation throughout, \citet{ghosh2024towards} consistently violates the constraint, leading to increasing violation regret.
These results empirically validate \cref{theorem:MDP-regret-main}, confirming the $\tiO(\sqrt{K})$ regret and episode-wise safety guarantees of our algorithm.

\looseness=-1
While DOPE achieves sublinear regret with zero-violation, it is limited to the tabular settings where $\aS$ is small, as described in \cref{remark:DOPE not in linear}.
This highlights the computational tractability of our \MDPalgo in large $\aS$, which supports \cref{remark:computational cost}.

\looseness=-1
Finally, the right plot shows that \MDPalgo explores the environment using $\pisafe$ primarily during the early stages of training, and stops deploying it after approximately $10^4$ episodes. This behavior supports \cref{lemma:Ck-bound-MDP-main}.

\end{document}